\documentclass{article}

\usepackage{microtype}
\usepackage{graphicx}
\usepackage{subfigure}
\usepackage{booktabs} 
\usepackage{multirow}
\usepackage{hyperref}

\usepackage[accepted]{icml2025}

\usepackage{amsmath}
\usepackage{amssymb}
\usepackage{mathtools}
\usepackage{amsthm}
\usepackage[capitalize,noabbrev]{cleveref}

\theoremstyle{plain}
\newtheorem{theorem}{Theorem}[section]
\newtheorem{proposition}[theorem]{Proposition}

\theoremstyle{definition}
\newtheorem{definition}[theorem]{Definition}
\newtheorem{assumption}[theorem]{Assumption}
\theoremstyle{remark}
\newtheorem{remark}[theorem]{Remark}

\usepackage[textsize=tiny]{todonotes}

\icmltitlerunning{Directly Forecasting Belief for Reinforcement Learning with Delays}

\begin{document}

\twocolumn[
\icmltitle{Directly Forecasting Belief for Reinforcement Learning with Delays}




\icmlsetsymbol{equal}{*}

\begin{icmlauthorlist}
\icmlauthor{Qingyuan Wu}{equal,southampton}
\icmlauthor{Yuhui Wang}{equal,kaust}
\icmlauthor{Simon Sinong Zhan}{equal,northwest}\\
\icmlauthor{Yixuan Wang}{northwest}
\icmlauthor{Chung-Wei Lin}{ntu_tw}
\icmlauthor{Chen Lv}{ntu_sig}
\icmlauthor{Qi Zhu}{northwest}
\icmlauthor{Jürgen Schmidhuber}{kaust,idisa}
\icmlauthor{Chao Huang}{southampton}
\end{icmlauthorlist}

\icmlaffiliation{northwest}{Northwestern University}
\icmlaffiliation{kaust}{GenAI, King Abdullah University of Science and Technology}
\icmlaffiliation{idisa}{The Swiss AI Lab IDSIA/USI/SUPSI}
\icmlaffiliation{ntu_tw}{National Taiwan University}
\icmlaffiliation{ntu_sig}{Nanyang Technological University}
\icmlaffiliation{southampton}{University of Southampton}

\icmlcorrespondingauthor{Chao Huang}{chao.huang@soton.ac.uk}

\icmlkeywords{Machine Learning, ICML}

\vskip 0.3in
]



\printAffiliationsAndNotice{\icmlEqualContribution} 

\begin{abstract}
Reinforcement learning (RL) with delays is challenging as sensory perceptions lag behind the actual events: the RL agent needs to estimate the real state of its environment based on past observations. State-of-the-art (SOTA) methods typically employ recursive, step-by-step forecasting of states. This can cause the accumulation of compounding errors. To tackle this problem, our novel belief estimation method, named Directly Forecasting Belief Transformer (DFBT), directly forecasts states from observations without incrementally estimating intermediate states step-by-step. We theoretically demonstrate that DFBT greatly reduces compounding errors of existing recursively forecasting methods, yielding stronger performance guarantees. In experiments with D4RL offline datasets, DFBT reduces compounding errors with remarkable prediction accuracy. DFBT's capability to forecast state sequences also facilitates multi-step bootstrapping, thus greatly improving learning efficiency. On the MuJoCo benchmark, our DFBT-based method substantially outperforms SOTA baselines. Code is available at \href{https://github.com/QingyuanWuNothing/DFBT}{https://github.com/QingyuanWuNothing/DFBT}.
\end{abstract}

\section{Introduction}
Reinforcement learning (RL) has achieved impressive success in various scenarios, including board games~\citep{silver2016mastering, schrittwieser2020mastering}, video games~\citep{mnih2013playing, berner2019dota} and cyber-physical systems~\citep{wei2017deep,wang2023joint,wang2023enforcing,zhan2024state}.
The timing factor is critical to enable RL in real-world applications, particularly the delays that occur in the interaction between the agent and the environment due to physical distance, computational processing, and signal transmission.
Otherwise, delays fundamentally affect the system's safety~\citep{arm_delay}, performance~\citep{cao2020using} and efficiency~\citep{quadrotor_delay}. 
Unlike rewards-delays-arising credit assignment issues that have been well studied~\citep{arjona2019rudder, wang2024highway}, observation-delays and action-delays disrupt the Markovian property of the Markov Decision Process (MDP), posing more challenges in RL.
Most of the literature~\citep{kim2023belief, wang2023addressing, wu2024boosting} mainly focuses on observation-delays, which has been proved as a superset of action-delays~\citep{delay_mdp, revisiting_augment}.
Therefore, this work mainly focuses on the RL with delays $\Delta$ in observation: at time step $t$, the agent can not observe the real environment state $s_{t}$ but only the history state $s_{t-\Delta}$ $\Delta$ steps ago.

To enable RL in environments with delayed observations, it is essential to restore the Markovian property~\citep{altman_delay, delay_mdp}.
Augmentation-based methods~\citep{dcac, kim2023belief} are based on an observation that the information state $x_t = \{s_{t-\Delta}, a_{t-\Delta:t-1}\}$, augmented from the last observable state and the sequential actions, carries the equivalent information with $s_t$~\citep{bertsekas2012dynamic}. Thus, the policy learning over the original state space $\mathcal{S}$ can be transferred to the policy learning over the augmented state space $\mathcal{X}$, which is memoryless and can be addressed by nominal RL.
However, as delays $\Delta$ increase, the dimensionality of the augmented state space $\mathcal{X}$ expands significantly,
leading to a drastic deterioration in learning efficiency due to the exponentially increased sample complexity~\citep{wu2024boosting}.

To improve learning efficiency in augmentation-based methods, belief-based methods~\citep{learning_planning_delayed_feedback, chen2021delay} propose to conduct RL in the original state space $\mathcal{S}$, by estimating the instant unobservable state $s_t$ from the information state $x_t$, of which the mapping is referred to as belief. The belief function is commonly forecasted recursively~\citep{chen2021delay, karamzade2024reinforcement}: for $i=1,\ldots,\Delta$, $s_{t-\Delta+i}$ is estimated by applying approximate dynamic function with the previous state $s_{t-\Delta+i-1}$ and the previous action $a_{t-\Delta+i-1}$.
Obtaining the estimation of the instant state $s_t$, the delayed RL problem is effectively reduced to a delay-free RL problem without enlarging the state space, mitigating the curse of dimensionality.
However, this recursive process is evidently affected by the error accumulation of the approximate dynamic function across $\Delta$ steps: the compounding errors grow exponentially with the delays $\Delta$. 
This fundamental limitation of such recursive methodology for belief forecasting leads to significant performance degradation, especially in environments with long-delayed signals.

This work aims to squarely address the compounding errors in existing belief-based techniques with recursive strategy, using a direct strategy in sequence modeling. 
Specifically, we present the Directly Forecasting Belief Transformer (DFBT), a novel directly forecasting belief method by reformulating belief forecasting as a sequence modeling problem.
DFBT first represents the information state $x_t$ and reward signals $r_{t-\Delta:t-1}$ to $\Delta$ tokens in the form $\{s_{t-\Delta}, a_{t-\Delta+i}, r_{t-\Delta+i}\}_{i=0}^{\Delta-1}$.
\textbf{
Unlike recursively forecasting belief methods that invoke the forward prediction process $\Delta$ times, DFBT simultaneously forecasts the states $s_{t-\Delta+1:t}$ without introducing accumulated errors, effectively addressing the issue of compounding errors.
}
We then integrate DFBT with the Soft Actor-Critic (SAC) method~\citep{soft_actor_critic_application} in online learning, incorporating multi-step bootstrapping on accurate state predictions generated from DFBT, which further improves learning efficiency.
Theoretically, we demonstrate how the compounding errors in recursively forecasting belief affect the RL performance and how it is alleviated by our direct strategy.
Empirically, using the D4RL benchmark~\citep{fu2020d4rl}, we show that DFBT achieves significantly higher prediction accuracy compared to other belief methods.
On the MuJoCo benchmark~\citep{mujoco}, across various delays settings, we demonstrate that our DFBT-SAC consistently outperforms SOTA augmentation-based and belief-based methods in both learning efficiency and overall performance.

In \Cref{sec:preliminaries}, we introduce the delayed RL problem and the concept of belief representation. \Cref{sec:approach} presents our proposed DFBT, which directly forecasts states from delayed observations, avoiding the need for recursive single-step predictions. Additionally, we develop DFBT-SAC by leveraging multi-step bootstrapping on the states predicted by DFBT. Through theoretical analysis in \Cref{sec:theoretical_analysis}, we show that DFBT effectively mitigates the compounding errors associated with recursively forecasting belief methods, ensuring superior performance. Finally, in \Cref{sec:experiments}, we empirically demonstrate the superior prediction accuracy of DFBT on the D4RL datasets and the remarkable efficacy of DFBT-SAC compared to state-of-the-art approaches on MuJoCo across various delays settings.
Overall, our key contributions are summarized as follows:
\begin{itemize}
    \item We present Directly Forecasting Belief Transformer (DFBT), a novel directly forecasting belief method that effectively addresses compounding errors in recursively generated belief. 
    \item We propose DFBT-SAC, a novel delayed RL method that further improves the learning efficiency via multi-step bootstrapping on the DFBT.
    \item We theoretically demonstrate that our directly forecasting belief significantly reduces compounding errors compared to existing recursively forecasting belief approaches, offering a more robust performance guarantee.
    \item We empirically demonstrate that our DFBT method effectively forecasts state sequences with significantly higher prediction accuracy compared to baseline approaches.
    \item We empirically show that our DFBT-SAC outperforms SOTAs in terms of sample efficiency and performance on the MuJoCo benchmark.
\end{itemize}

\section{Related Works}
In classical control theory, delays are modelled with delay–differential equations (DDEs) \citep{cooke1963differential}. A rich toolbox is available for analysing their reachability \citep{fridman2003reachable,xue2020over}, stability \citep{feng2019taming}, and safety \citep{xue2021reach}. These techniques, however, assume fully known dynamics and become impractical for high-dimensional systems.
Delayed reinforcement learning (RL) is far closer to real deployments than the traditional delay approach and the delay-free RL setting that dominates the literature. Latencies are intrinsic to robotics \citep{arm_delay,quadrotor_delay}, high-frequency trading \citep{hasbrouck2013low}, intelligent transportation \citep{cao2020using}, and many other domains. 
Depending on the methodology of retrieving the Markovian property, existing delayed RL approaches can be mainly categorized as augmentation-based and belief-based approaches.

\paragraph{Augmentation-based Approaches.}
Augmentation-based approaches retrieve the Markovian property by augmenting the original state with the sequence of actions within the delays window, optimizing the policy in the resulting augmented MDP to mitigate performance degradation caused by errors in belief representation approximation.
Specifically, DIDA~\citep{dida} learns the delayed policy in the augmented MDP by imitating the behaviours from the delay-free expert policy in the original MDP;
DC/AC~\citep{dcac} suggest using the delays correction operator to improve the learning efficiency in the augmented MDP;
ADRL~\citep{wu2024boosting} introduces the auxiliary delayed task with the smaller augmented MDP for bootstrapping the larger augmented MDP;
VDPO~\citep{wu2024variational} proposes to learn the delayed policy through the lens of variational inference, improving the learning efficiency via extensive optimization tools.
However, augmentation-based approaches still operate on an increasingly large augmented state space as delays grow, inevitably suffering from the curse of dimensionality and resulting in significant learning inefficiencies.

\paragraph{Belief-based Approaches.}
The belief-based approach retrieves the Markovian property by forecasting the unobservable state through belief representation, then optimizing policy within the original state space.
Thus, correspondingly, the prediction accuracy of approximated belief representation is highly related to the overall performance.
DATS~\citep{chen2021delay} is the first to propose approximating the belief representation using a Gaussian distribution. Inspired by the world models~\citep{schmidhuber1990making, schmidhuber1990line, ha2018world}, 
D-Dreamer~\citep{karamzade2024reinforcement} employ a recurrent neural network to capture temporal dependencies in the belief through forward prediction in the latent space, enabling adaptation to high-dimensional observations.
D-SAC~\citep{learning_belief} utilizes the attention mechanism in the causal transformer to integrate the information across the delays.
Unfortunately, the approximation errors of recursively forecasting belief accumulated at each step finally lead to compounding errors, which hinder prediction accuracy and performance.

\paragraph{Time Series Forecasting in RL.}
Time series forecasting plays a critical role in predicting future states in RL, particularly within the context of world model learning~\citep{ha2018world}, which simulates the dynamics of the environment, predicting future states and planning effectively without directly interacting with the real environment. 
World model learning often leverages time series forecasting techniques to capture temporal dependencies, enhancing the agent’s understanding of the environment. Applications of world models include robotics, autonomous driving, and gaming scenarios~\citep{hafner2019dream}.
Additionally, the RL problem can be treated as a sequence modeling problem, which can be effectively solved by the transformer~\citep{transformer}. Directly predicting the action from the desired outcomes~\citep{schmidhuber2019reinforcement}, decision transformer~\citep{chen2021decision} predict future actions based on historical trajectories and desired outcomes. Similarly, trajectory transformer~\citep{janner2021offline} models trajectories as sequences, enabling policy optimization by leveraging autoregressive modeling.
Therefore, in this paper, we investigate how to leverage the advanced time series forecasting approaches in addressing the issue of compounding errors existing in the recursively forecasting belief.
We also present a comprehensive theoretical analysis of how compounding errors seriously degenerate performance as delays are increased.

\section{Preliminaries}
\label{sec:preliminaries}
\subsection{From Delay-free RL to Delayed RL}
A conventional delay-free RL problem is usually formalized as a Markov Decision Process (MDP) represented as a tuple $\langle \mathcal{S}, \mathcal{A}, \mathcal{P}, \mathcal{R}, \rho, \gamma \rangle$, where $\mathcal{S}$ is the state space, $\mathcal{A}$ is the action space, $\mathcal{P}: \mathcal{S} \times \mathcal{A} \times \mathcal{S} \rightarrow \left[0, 1\right]$ is the dynamic function, $\mathcal{R}: \mathcal{S} \times \mathcal{A} \rightarrow \mathbb{R}$ is the reward function, $\rho$ is the initial state distribution, and $\gamma \in (0, 1)$ is the discount factor.
The agent selects an action $a_t\sim \pi(\cdot|s_t)$ according to the policy $\pi: \mathcal{S} \times \mathcal{A} \rightarrow \left[0, 1\right]$ based on the current state $s_t$ at time step $t$.
The agent will then observe the next state $s_{t+1} \sim \mathcal{P}(\cdot|s_t, a_t)$ and the reward $r_t = \mathcal{R}(s_t, a_t)$ from the environment.
The objective of the agent is to find the optimal policy $\pi^*$ that maximizes the expected discounted return $G:=\mathbb{E}_{\tau \sim p_\pi(\tau)}\left[\sum_{t=0}^\infty \gamma^t \mathcal{R}(s_t, a_t)\right]$ where $p_\pi(\tau)$ represents the distribution of trajectories induced by policy $\pi$.

A delayed RL problem can be formalized as an augmented MDP by applying the augmentation technique~\citep{altman_delay, delay_mdp} to retrieve the Markovian property.
For a delayed RL problem with constant delays $\Delta$, the newly formed MDP is represented as $\langle \mathcal{X}, \mathcal{A}, \mathcal{P}_\Delta, \mathcal{R}_\Delta, \rho_\Delta, \gamma \rangle$, where $\mathcal{X}:= \mathcal{S} \times \mathcal{A}^\Delta$ is the augmented state space, $\mathcal{A}$ is the original action space, the delayed dynamic function is
$\mathcal{P}_\Delta(x_{t+1}|x_t, a_t) 
:=
\mathcal{P}(s_{t-\Delta+1}|s_{t-\Delta}, a_{t-\Delta})\delta_{a_t}(a'_t)\prod_{i=1}^{\Delta-1}\delta_{a_{t-i}}(a'_{t-i})
$ where $\delta$ is the Dirac distribution, the delayed reward function is defined as $\mathcal{R}_\Delta(x_t, a_t):= \mathop{\mathbb{E}}_{s_t\sim b(\cdot|x_t)}\left[\mathcal{R}(s_t, a_t)\right]$ where $b:\mathcal{X} \times \mathcal{S} \rightarrow \left[0, 1\right]$ is the belief representation mapping from the augmented state space $\mathcal{X}$ to the original state space $\mathcal{S}$, $\rho_\Delta =\rho\prod_{i=1}^{\Delta}\delta_{a_{-i}}$ is the initial augmented state distribution.

\subsection{Belief Representation in Delayed RL} 
Delayed RL can be viewed as a special form of a partially observable RL problem where the observation is contaminated by noise, instead of being delayed.
Therefore, similar to partially observable RL, delayed RL also have the belief representation defined as follows:
\begin{equation}
    \begin{aligned}
        &b(s_t|x_t):=\\
        &\int_{\mathcal{S}^\Delta}\prod_{i=0}^{\Delta-1}\mathcal{P}(s_{t-\Delta+i+1}|s_{t-\Delta+i}, a_{t-\Delta+i})\mathrm{d}{s_{t-\Delta+i+1}}.\\
    \end{aligned}
\end{equation}
The belief representation can retrieve the Markovian property via mapping the augmented state space $\mathcal{X}$ to $\mathcal{S}$, recasting the delayed RL problem in the original MDP without augmenting the state space.
The belief representation can be viewed as the recursive forward prediction of the dynamics $\mathcal{P}$.
With the belief representation, the agent can directly learn in the original state space $\mathcal{S}$.
In this work, we use $s$ and $\hat{s}$ to represent the true state of the environment and the predicted state of the approximate belief, respectively.

\begin{figure*}[t]
    \centering
    \includegraphics[width=1.\linewidth]{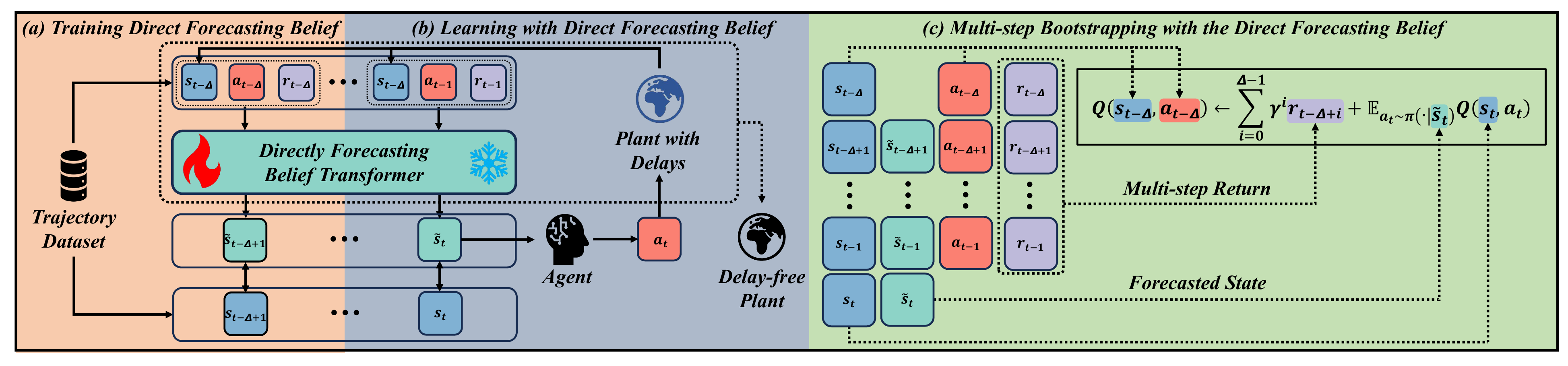}
    \caption{Pipeline of DFBT-SAC. (a) Training DFBT on the trajectory dataset. (b) The agent can interact and learn with the delay-free environment recovered by the DFBT as highlighted in the dashline box. (c) Multi-step bootstrapping on the forecasted states from DFBT.\label{fig:dbt}}
\end{figure*}

\section{Our Approach}
\label{sec:approach}

In this section, we present the Directly Forecasting Belief Transformer (DFBT), a new directly forecasting belief method. 
By framing belief forecasting as a sequence modeling problem, DFBT directly predicts the unobservable states.
Instead of recursively predicting the next state step-by-step, DFBT can effectively capture the dependency across the delays via the attention mechanism of the transformer model.
Therefore, our DFBT can effectively reduce the compounding errors of the recursively forecasting belief, especially in the long delays.
Specifically, as summarized in \Cref{fig:dbt}, we train the DFBT using a pre-collected trajectory dataset, latter deploying the trained DFBT in the environment with delays to reconstruct the delay-free environment for training.
Furthermore, we integrate multi-step bootstrapping with the predicted states of DFBT to improve learning efficiency.

\subsection{Directly Forecasting Belief}
The belief representation learning problem can be viewed as a sequence modeling problem.
Given a pre-collected delay-free trajectory $\{s_{i}, a_{i}, r_{i}\}_{i=0}^{T}$, we sample the sub-trajectory with $\Delta$ timesteps $\{s_{t-\Delta+i}, a_{t-\Delta+i}, r_{t-\Delta+i}\}_{i=0}^{\Delta}$.
We model the past reward signals $r_{t-\Delta:t-1}$ into the augmented state $x_t = \{s_{t-\Delta}, a_{t-\Delta:t-1}\}$ for considering more dynamic information in belief. 
Then, we have reformed the representation of the augmented state to $\Delta$ tokens for sequence modeling: $x^\text{tokens}_t = \{s_{t-\Delta}, a_{t-\Delta+i}, r_{t-\Delta+i}\}_{i=0}^{\Delta - 1}$.
The Directly Forecasting Belief Transformer (DFBT) $b_\theta$ leverages the transformer architecture~\citep{transformer}, utilizing the attention mechanism to effectively capture dependencies across long delays.
Inputting with $x^\text{tokens}_t$, DFBT simultaneously predicts the unobserved $\Delta$ states $\{s_{t-\Delta+i}\}_{i=1}^{\Delta}$ via autoregressive modeling with loss:
\begin{equation}
\label{eq:dbt_loss}
\bigtriangledown_{\theta} 
\left[
\mathop{\sum}_{i=1}^\Delta
\left[
- \log b^{(i)}_\theta(s_{t-\Delta+i}|x^\text{tokens}_t)
\right]
\right],
\end{equation}
where $b^{(i)}_\theta(\cdot|{x_t})$ represents the $i$-th prediction.
In a deterministic environment with a deterministic belief function, we typically replace \Cref{eq:dbt_loss} with the Mean-Square-Error (MSE) loss.

\subsection{Multi-step Bootstrapping with DFBT}
\label{sec:multi_step_bootstrap}
Next, we directly deploy the trained DFBT $b_{\theta}$ in the online environment with delayed signals to reconstruct the delay-free environment where the agent can directly learn with the original state space $\mathcal{S}$ instead of the augmented state space $\mathcal{X}$, thus maintaining the superior sample complexity in online learning.

Then, we present our practical delayed RL method, named DFBT-SAC by incorporating the trained DFBT $b_\theta$ with Soft Actor-Critic~\citep{soft_actor_critic_application}.
To improve the learning efficiency with the DFBT, the critic of DFBT-SAC is bootstrapped on the states predicted by the DFBT with the multi-step learning~\citep{rlai, hessel2018rainbow} and delay-free training techniques~\citep{wu2024boosting, kim2023belief}.
Specifically, given multi-step data $(x^\text{tokens}_t, s_{t-\Delta+1:t})$, the critic $Q_{\psi}$ parameterized by $\psi$ is updated via: 
\begin{equation}
\label{eq:dbt_critic_loss}
    \bigtriangledown_{\psi}
    \left[
        \frac{1}{2}
        \left(
            Q_{\psi}(s_{t-\Delta}, a_{t-\Delta})
            -
            \mathbb{Y}
        \right)^2
    \right],
\end{equation}
where $N$-step temporal difference target $\mathbb{Y}$ is defined as:
$$
\begin{aligned}
&\mathbb{Y} := \mathop{\sum}_{i=0}^{N-1}\gamma^i r_{t-\Delta+i}\\
&+\gamma^N 
\mathop{\mathbb{E}}_{a \sim \pi(\cdot|\hat{s}_{t-\Delta + N})\atop
\hat{s}_{t-\Delta+N}\sim b^{(N)}_\theta(\cdot|x^\text{tokens}_t)
}\left[
Q(s_{t-\Delta+N}, a) + \log\pi(a|\hat{s}_{t-\Delta + N})
\right],\\
\end{aligned}
$$
and $N(\leq \Delta)$ is the bootstrapping steps on the DBFT.
In this work, we set $N=8$ as default, and we also conduct the ablation study (\Cref{sec:experiments}) to investigate bootstrapping steps settings.
The actor $\pi_\phi$ parameterized by $\phi$ is updated by:
\begin{equation}
\label{eq:dbt_actor_loss}
    \bigtriangledown_{\phi}
    \mathop{\mathbb{E}}_{
    a \sim \pi(\cdot|\hat{s}_{t-\Delta + N})\atop
    \hat{s}_{t-\Delta+N}\sim b^{(N)}_\theta(\cdot|x^\text{tokens}_t)
    }
    \left[
    \log\pi_\phi(a|\hat{s}_{t-\Delta+N}) - Q_\psi({s}_{t-\Delta+N}, a)
    \right].
\end{equation}
The pseudo-code of DFBT-SAC is provided in \Cref{alg::dbt_sac}.

\begin{algorithm}[t]
   \caption{DFBT-SAC}
   \label{alg::dbt_sac}
   \begin{algorithmic}
        \STATE {\bfseries Input:} offline dataset $\mathcal{D}$, DFBT $b_\theta$, critic $Q_\psi$, actor $\pi_\phi$;
        \STATE \textcolor{black!30}{$\#$ training DFBT on the offline dataset}
        \FOR{Epoch = $1, \ldots$}
            \STATE Update $b_\theta$ on the $\mathcal{D}$ via \Cref{eq:dbt_loss}
        \ENDFOR
        \STATE \textcolor{black!30}{$\#$ learning with DFBT on the online environment}
        \FOR{Epoch = $1, \ldots$}
            \STATE Update $Q_\psi$ on via \Cref{eq:dbt_critic_loss}
            \STATE Update $\pi_\phi$ on via \Cref{eq:dbt_actor_loss}
        \ENDFOR
        \STATE {\bfseries Output:} $b_\theta$ and $\pi_\phi$
    \end{algorithmic}
\end{algorithm}

\section{Theoretical Analysis}
\label{sec:theoretical_analysis}
In this section, we theoretically illustrate that performance degeneration is rooted in compounding errors of recursively forecasting belief (\Cref{sec:single_step_belief}), which can be effectively addressed by our DFBT, directly forecasting belief, thus achieving a better performance guarantee (\Cref{sec:sequential_belief}).

Before starting the theoretical analysis, we introduce the definition of the performance degeneration of the ground-truth belief $b$ as follows.
\begin{definition}[Performance Degeneration of Ground-truth Belief~\citep{dida}]
    \label{lemma:performance_diff_true_belief}
    For policies $\pi$ and $\pi_\Delta$ with delays $\Delta$. Given any $x_t\in\mathcal{X}$, the performance difference of the ground-truth belief $b$ is denoted as $I^\text{true}_{\Delta}(x_t)$
    $$
    \begin{aligned}
        I^\text{true}_{\Delta}(x_t)
        &= \frac{1}{1-\gamma} \mathop{\mathbb{E}}_{\substack{
        {\hat{s} \sim b(\cdot|\hat{x})}\\
        {\hat{a} \sim\pi_\Delta(\cdot|\hat{x})}\\
        {\hat{x} \sim d_{x_t}^{\pi}}
        }}
        \left[{V^{\pi}(\hat{s})} - {Q^{\pi}}({\hat{s}}, \hat{a})\right],\\
    \end{aligned}
    $$
    where $d_{x_t}^{\pi}$ is the augmented state distribution induced by policy $\pi$ with the initial state $x_t$.
\end{definition}

Furthermore, we also introduce the Lipschitz Continuity of MDP (\Cref{lemma:lc_mdp}) and value function (\Cref{lemma:lc_v}), which are common and mild assumptions in the literature.
\begin{definition}[Lipschitz Continuity of MDP~\citep{rl_lipschitz_continuous}]
    \label{lemma:lc_mdp}
    An MDP is $L_\mathcal{P}$-LC, if $\forall (s_1, a_1), (s_2, a_2) \in \mathcal{S}\times\mathcal{A}$, dynamic function $\mathcal{P}$ satisfies:
    $$
    \mathcal{W}\left(\mathcal{P}(\cdot|s_1, a_1)||\mathcal{P}(\cdot|s_2, a_2)\right) \leq L_\mathcal{P}(d_\mathcal{S}(s_1, s_2) + d_\mathcal{A}(a_1, a_2)),
    $$
    where $\mathcal{W}$ is the L1-Wasserstein distance, $d_\mathcal{S}$ and $d_\mathcal{A}$ are distrance measure of the $\mathcal{S}$ and $\mathcal{A}$, repectively.
\end{definition}
\begin{definition}[Lipschitz Continuity of Value Function~\citep{rl_lipschitz_continuous}]
    \label{lemma:lc_v}
    Consider a $L_Q$-LC Q-function $Q^\pi$ of the $L_\pi$-LC policy $\pi$, value function $V^\pi$ satisfies that
    $$
        \left|
        \mathop{\mathbb{E}}_{
            s_1 \sim \mu \atop
            s_2 \sim \upsilon
        }\left[
        V^\pi(s_1) - V^\pi(s_2)
        \right]
        \right|
        \leq
        L_V \mathcal{W}\left(\mu||\upsilon\right),
    $$
    where $L_V = L_Q(1+L_\pi)$ and $\mu, \upsilon$ are two arbitrary distributions over $\mathcal{S}$.
\end{definition}

\subsection{Recursively Forecasting Belief: Compounding Errors Analysis}
\label{sec:single_step_belief}
In this section, we show that the performance degeneration of recursively forecasting belief is influenced by compounding errors, which consist of the recursive error and are exponentially increased with delays.
We assume that the recursively forecasting belief $\mathcal{P}_\theta$ has the approximated error bound as shown in \Cref{assumption:performance_diff_true_belief}.
\begin{assumption}[Approximated Dynamic Difference Bound]
    \label{assumption:performance_diff_true_belief}
    The distance between the approximated dynamic function $\mathcal{P}_\theta$ parameterized by $\theta$ and the ground-truth dynamic function $\mathcal{P}$ is bounded, it satisfies that $\forall (s, a) \in \mathcal{S}\times \mathcal{A}$, 
    $$
    \mathcal{W}(\mathcal{P}_\theta(\cdot|{s, a}) || \mathcal{P}(\cdot|{s, a})) \leq \epsilon_\mathcal{P}.
    $$
\end{assumption}

Then, we demonstrate that the performance difference of the recursively forecasting belief is determined by the compounding errors (\Cref{theo_delayed_performance_diff_single}).

\begin{theorem}[Performance Difference of Recursively Forecasting Belief, Proof in \Cref{appendix:theo_delayed_performance_diff_single}]
    \label{theo_delayed_performance_diff_single}
    For the delay-free policy $\pi$ and the delayed policy $\pi_\Delta$. Given any $x_t \in \mathcal{X}$, the performance difference $I^\text{recursive}(x_t)$ of the recursively forecasting belief $b_\theta$ can be bounded as follows, respectively.

    For deterministic delays $\Delta$, we have
    $$
    \left|I^\text{recursive}(x_t)\right| 
    \leq 
    \left|I^\text{true}_{\Delta}(x_t)\right| 
    + L_V \underbrace{\frac{1 - {L_\mathcal{P}}^{\Delta}}{1 - L_\mathcal{P}} \epsilon_\mathcal{P}}_{\text{compounding errors}}.
    $$
    And for stochastic delays $\delta \sim d_\Delta(\cdot)$, we have
    $$
    \begin{aligned}
        \left|I^\text{recursive}(x_t)\right| 
        &\leq \mathop{\mathbb{E}}_{\delta \sim d_\Delta(\cdot)}\left[
        \left|I^\text{true}_{\delta}(x_t)\right| + L_V \underbrace{\frac{1 - {L_\mathcal{P}}^{\delta}}{1 - L_\mathcal{P}} \epsilon_\mathcal{P}}_{\text{compounding errors}}
        \right]. \\
    \end{aligned}
    $$
\end{theorem}
\Cref{theo_delayed_performance_diff_single} tells that the compounding errors are exponentially increased with the delays, seriously degenerating the performance.

\begin{figure*}[t]
    \centering
    \centerline{
        \subfigure[HalfCheetah-v2]{\includegraphics[width=0.33\linewidth]{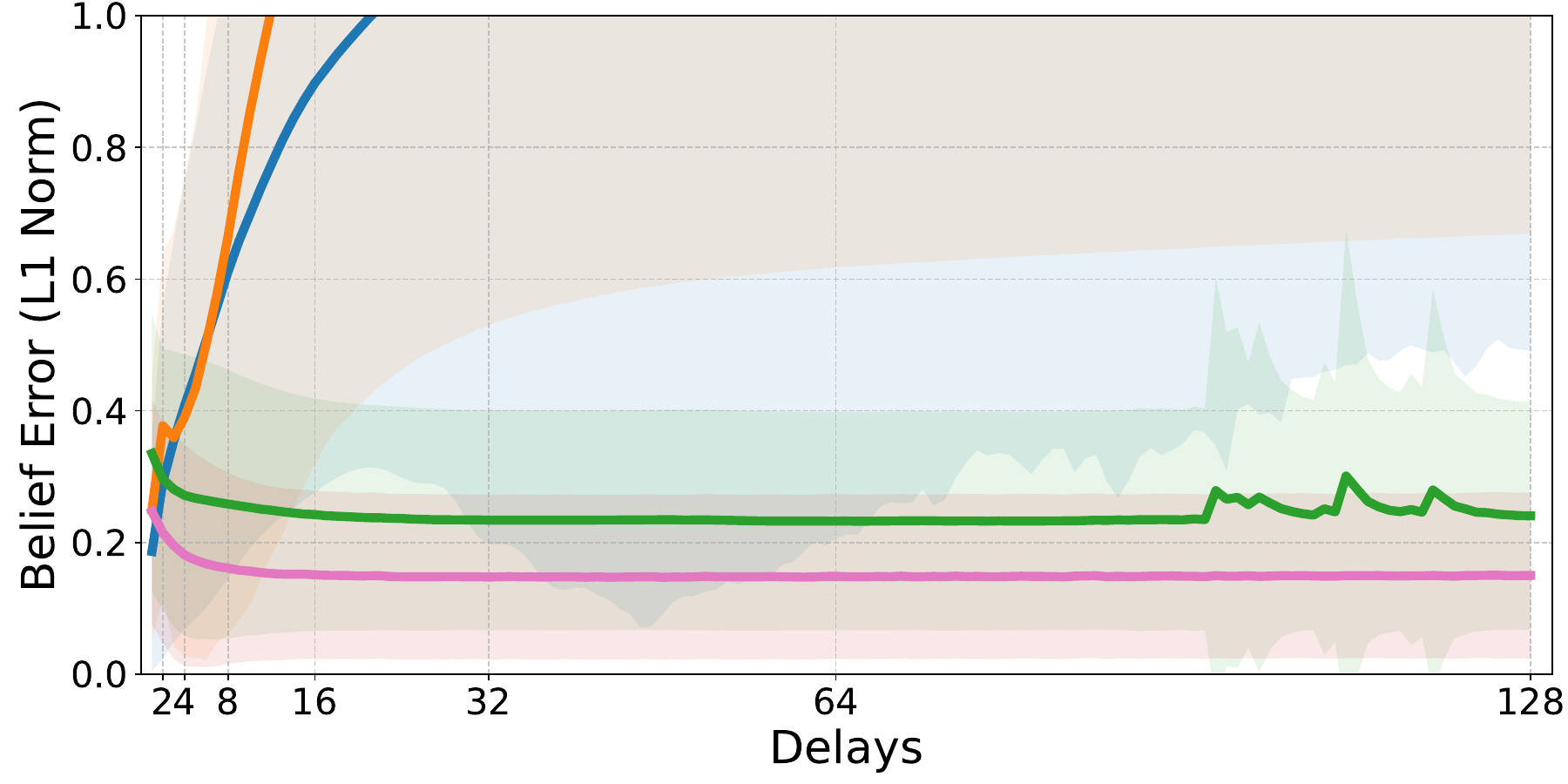}}
        \subfigure[Hopper-v2]{\includegraphics[width=0.33\linewidth]{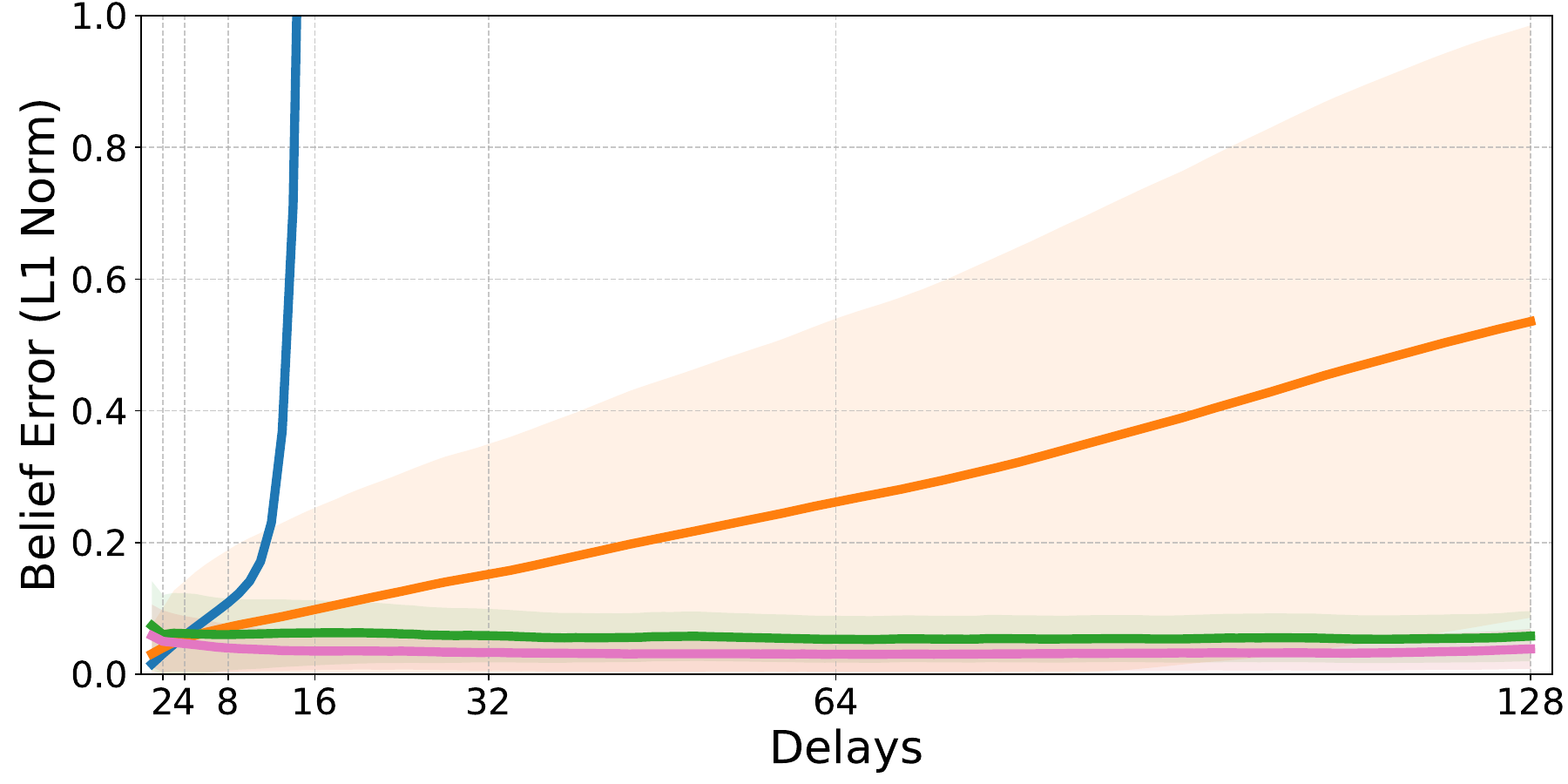}}
        \subfigure[Walker2d-v2]{\includegraphics[width=0.33\linewidth]{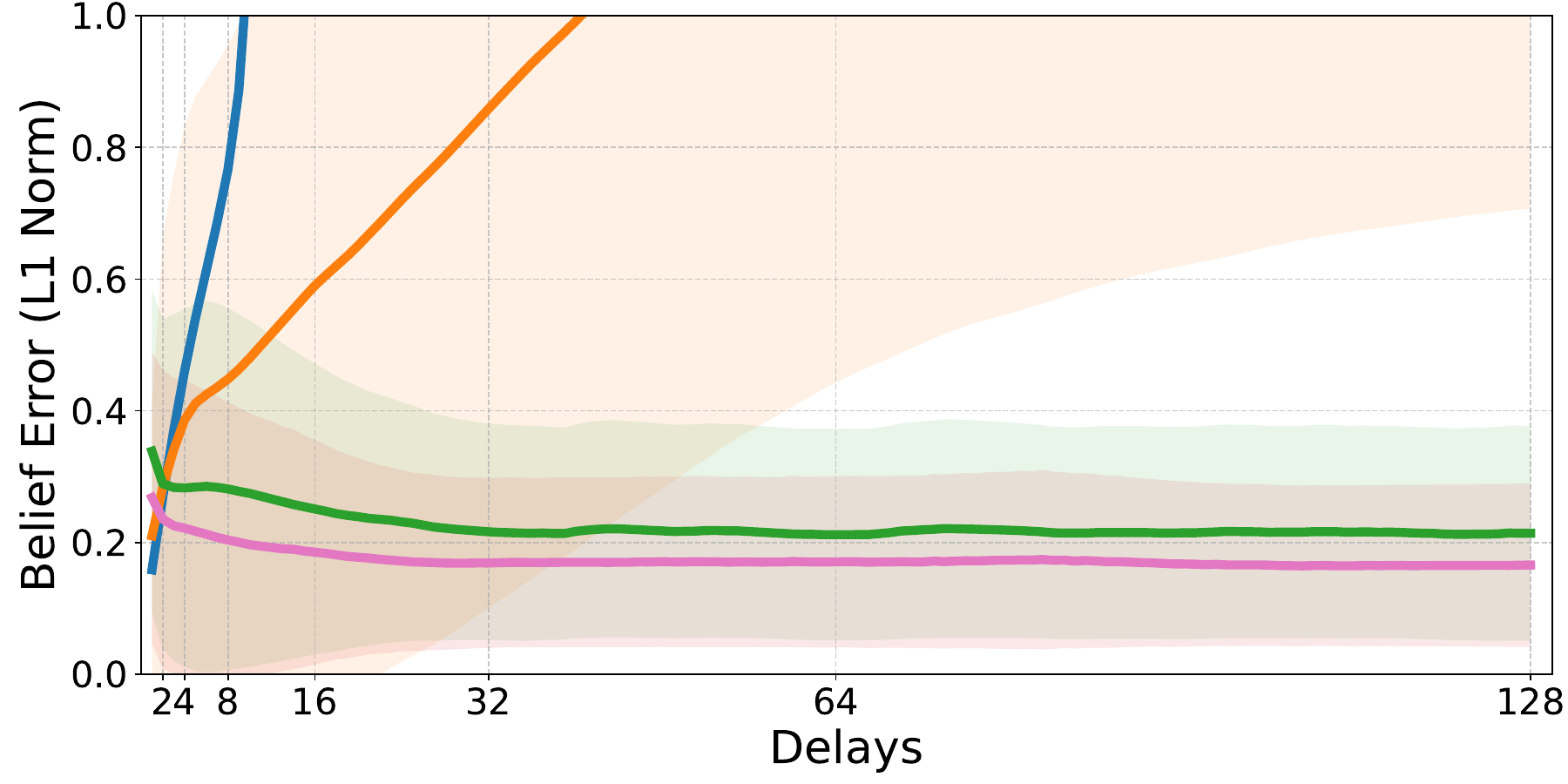}}
    }
    \includegraphics[width=1.\linewidth]{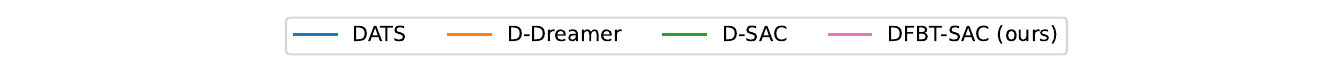}
    \vskip -0.1in
    \caption{Belief errors comparison on (a) HalfCheetah-v2, (b) Hopper-v2, and (c) Walker2d-v2.\label{fig:delayed_belief_accuracy}}
\end{figure*}

\subsection{Directly Forecasting Belief: Errors Analysis}
\label{sec:sequential_belief}
Next, we theoretically show that our directly forecasting belief can effectively address the compounding errors. The sequence-modeling method estimates the states $s_{t-\Delta + 1:t}$ based on the augmented state $x_{t}$ in parallel, instead of invoking $\Delta$ times iteratively, which effectively alleviates the source of compounding errors. Especially, the attention mechanism in the transformer can selectively capture the long-range relationships within the augmented state with long delays.
We assume that the belief error between the directly forecasting belief and the ground-truth belief is bounded (\Cref{definition:sequence_modeling_belief_error}).

\begin{assumption}[Directly Forecasting Belief Difference Bound]
    \label{definition:sequence_modeling_belief_error}
    The distance between the directly forecasting belief belief $b_\theta$ parameterized by $\theta$ and the ground-truth belief $b$ is bounded, it satisfies that $\forall x_t \in \mathcal{X}$,
    we have
    $ \mathcal{W}(b(\cdot|{x_t}) || b_{\theta}(\cdot|{x_t})) \leq \epsilon_{direct}$, where $
    \epsilon_{direct}:=\max_{i=1,\ldots,\Delta} \mathcal{W}(b^{(i)}(\cdot|{x_t}) || b^{(i)}_{\theta}(\cdot|{x_t}))
    $.
\end{assumption}

From \Cref{theo_delayed_performance_diff_single}, we can derive the performance degeneration bound of directly forecasting belief as follows.
\begin{proposition}[Performance Degeneration Bound of Directly Forecasting Belief, Proof in \Cref{appendix:theo_delayed_performance_diff_seq}]
For the delay-free policy $\pi$ and the delayed policy $\pi_\Delta$. Given any $x_t \in \mathcal{X}$, the performance degeneration $I^\text{direct}$ of the directly forecasting belief $b_\theta$ can bounded as follows respectively.

For deterministic delays $\Delta$, we have
$$
    \begin{aligned}
    \left|I^\text{direct}(x_t)\right| 
    &\leq \left|I^\text{true}_{\Delta}(x_t)\right| + L_V \epsilon_{direct}. \\
    \end{aligned}
$$
For stochastic delays $\delta \sim d_\Delta(\cdot)$, we have
$$
    \begin{aligned}
    \left|I^\text{direct}(x_t)\right| 
    &\leq \mathop{\mathbb{E}}_{\delta \sim d_\Delta(\cdot)}\left[
    \left|I^\text{true}_{\delta}(x_t)\right|\right] + L_V \epsilon_{direct}. \\
    \end{aligned}
$$    
\end{proposition}

Then, we have the performance degeneration comparison between directly forecasting belief and recursively forecasting belief.
\begin{proposition}[Performance Degeneration Comparison, Proof in \Cref{appendix:proposition_performance_degeneration_comparison}]
\label{proposition_performance_degeneration_comparison}
    Directly forecasting belief could achieve a better performance guarantee 
    $\left|I^\text{direct}(x_t)\right| \leq \left|I^\text{recursive}(x_t)\right|$, 
    if we have
    $$
    \epsilon_{direct} \leq \frac{1 - {L_\mathcal{P}}^{\Delta}}{1 - L_\mathcal{P}} \epsilon_\mathcal{P}
    $$
    for deterministic delays $\Delta$, and
    $$
    \epsilon_{direct}
    \leq 
    \mathop{\mathbb{E}}_{\delta \sim d_\Delta(\cdot)} \left[\frac{1 - {L_\mathcal{P}}^{\delta}}{1 - L_\mathcal{P}}\right] \epsilon_\mathcal{P}
    $$
    for stochastic delays $\delta \sim d_\Delta(\cdot)$.
\end{proposition}

\begin{remark}[Empirical Validation of \Cref{proposition_performance_degeneration_comparison}]
    It is obvious that the belief errors of recursively forecasting grows much faster than the directly forecasting, which is not strictly related with the delay length. We also empirically show that \Cref{proposition_performance_degeneration_comparison} is always held in \Cref{exp:delayed_belief_accuracy}. 
    In the future, we will theoretically investigate the sample complexity of recursive and direct forecasting beliefs.
\end{remark}

\begin{remark}[General Error Distribution Case]
In the context of time forecasting, our theoretical results of performance degeneration comparison have the potential to extend to the variance analysis commonly discussed in related literature~\cite{taieb2015bias, clements1996multi, chevillon2007direct}.
\end{remark}

\begin{table*}[t]
\centering
\caption{Performance on MuJoCo with Deterministic Delays. The best performance is underlined, the best belief-based method is in \textcolor{red}{red}.\label{table:deterministic_delays}}
\scalebox{0.8}{
\begin{tabular}{c|c|ccc|ccccc}
\hline
\multirow{2}{*}{Task} & \multirow{2}{*}{Delays} & \multicolumn{3}{c|}{Augmentation-based}                   & \multicolumn{5}{c}{Belief-based}                                                                   \\
                      &                         & A-SAC             & BPQL              & ADRL              & DATS              & D-Dreamer          & D-SAC             & DFBT-SAC(1) (ours) & DFBT-SAC (ours)   \\ \hline
                      & $8$                     & $0.10_{\pm 0.01}$ & $0.40_{\pm 0.04}$   & \underline{$0.44_{\pm 0.03}$} & $0.08_{\pm 0.01}$ & $0.08_{\pm 0.01}$ & $0.12_{\pm 0.06}$ & \textcolor{red}{$0.38_{\pm 0.03}$}  & $0.35_{\pm 0.12}$ \\
HalfCheetah-v2        & $32$                    & $0.02_{\pm 0.02}$ & $0.40_{\pm 0.03}$ & $0.26_{\pm 0.04}$ & $0.11_{\pm 0.04}$ & $0.08_{\pm 0.00}$ & $0.08_{\pm 0.02}$ & $0.40_{\pm 0.07}$  & \underline{\textcolor{red}{$0.42_{\pm 0.03}$}} \\
                      & $128$                   & $0.04_{\pm 0.06}$ & $0.08_{\pm 0.13}$ & $0.14_{\pm 0.02}$ & $0.10_{\pm 0.08}$ & $0.15_{\pm 0.05}$ & $0.09_{\pm 0.04}$ & $0.40_{\pm 0.06}$  & \underline{\textcolor{red}{$0.41_{\pm 0.03}$}} \\ \hline
                      & $8$                     & $0.61_{\pm 0.31}$ & $0.87_{\pm 0.09}$ & \underline{$0.95_{\pm 0.16}$} & $0.41_{\pm 0.31}$ & $0.11_{\pm 0.01}$ & $0.16_{\pm 0.05}$ & \textcolor{red}{$0.92_{\pm 0.28}$}  & $0.77_{\pm 0.18}$ \\
Hopper-v2             & $32$                    & $0.11_{\pm 0.02}$ & \underline{$0.89_{\pm 0.14}$} & $0.73_{\pm 0.20}$ & $0.07_{\pm 0.04}$ & $0.11_{\pm 0.05}$ & $0.11_{\pm 0.01}$ & $0.60_{\pm 0.23}$  & \textcolor{red}{$0.68_{\pm 0.20}$} \\
                      & $128$                   & $0.04_{\pm 0.01}$ & $0.08_{\pm 0.02}$ & $0.07_{\pm 0.01}$ & $0.08_{\pm 0.01}$ & $0.09_{\pm 0.03}$ & $0.06_{\pm 0.01}$ & $0.16_{\pm 0.02}$  & \underline{\textcolor{red}{$0.20_{\pm 0.03}$}} \\ \hline
                      & $8$                     & $0.44_{\pm 0.26}$ & \underline{$1.07_{\pm 0.02}$} & $0.97_{\pm 0.10}$ & $0.13_{\pm 0.05}$ & $0.11_{\pm 0.06}$ & $0.09_{\pm 0.05}$ & $0.95_{\pm 0.14}$  & \textcolor{red}{$0.99_{\pm 0.03}$} \\
Walker2d-v2           & $32$                    & $0.10_{\pm 0.02}$ & $0.37_{\pm 0.25}$ & $0.16_{\pm 0.08}$ & $0.02_{\pm 0.03}$ & $0.08_{\pm 0.05}$ & $0.08_{\pm 0.02}$ & $0.57_{\pm 0.21}$  & \underline{\textcolor{red}{$0.64_{\pm 0.10}$}} \\
                      & $128$                   & $0.06_{\pm 0.00}$ & $0.07_{\pm 0.08}$ & $0.08_{\pm 0.01}$ & $0.02_{\pm 0.02}$ & $0.08_{\pm 0.05}$ & $0.11_{\pm 0.06}$ & $0.38_{\pm 0.11}$  & \underline{\textcolor{red}{$0.40_{\pm 0.08}$}} \\ \hline
\end{tabular}
}
\end{table*}

\begin{table*}[h]
\centering
\caption{Performance on MuJoCo with Stochastic Delays. The best performance is underlined, and the best belief-based method is in \textcolor{red}{red}.\label{table:stochastic_delays}}
\scalebox{0.85}{
\begin{tabular}{c|c|ccc|cccc}
\hline
                                 &                                                                             & \multicolumn{3}{c|}{Augmentation-based}                                                                    & \multicolumn{4}{c}{Belief-based}                                                                           \\
\multirow{-2}{*}{Task}           & \multirow{-2}{*}{\begin{tabular}[c]{@{}c@{}}Delays\end{tabular}} & A-SAC             & BPQL                                     & ADRL                                     & DATS              & D-Dreamer              & D-SAC             & DFBT-SAC (ours)                                    \\ \hline
                                 & $U(1,8)$                                                                       & $0.09_{\pm 0.01}$ & $0.21_{\pm 0.07}$                        & \multicolumn{1}{c|}{$0.17_{\pm 0.07}$} & $0.09_{\pm 0.03}$ & $0.02_{\pm 0.01}$ & $0.03_{\pm 0.01}$ & \underline{ \textcolor{red}{$0.37_{\pm 0.12}$}} \\
HalfCheetah-v2                   & $U(1,32)$                                                                      & $0.01_{\pm 0.00}$ & \underline{ $0.33_{\pm 0.07}$} & \multicolumn{1}{c|}{$0.23_{\pm 0.02}$} & $0.11_{\pm 0.04}$ & $0.02_{\pm 0.00}$ & $0.01_{\pm 0.01}$ & \textcolor{red}{$0.31_{\pm 0.16}$}                        \\
                                 & $U(1,128)$                                                                     & $0.01_{\pm 0.01}$ & $0.03_{\pm 0.03}$                        & \multicolumn{1}{c|}{$0.15_{\pm 0.02}$} & $0.16_{\pm 0.03}$ & $0.16_{\pm 0.00}$ & $0.02_{\pm 0.00}$ & \underline{ \textcolor{red}{$0.39_{\pm 0.04}$}} \\ \hline
                                 & $U(1,8)$                                                                       & $0.17_{\pm 0.05}$ & $0.20_{\pm 0.04}$                        & \multicolumn{1}{c|}{$0.18_{\pm 0.04}$} & $0.04_{\pm 0.01}$ & $0.07_{\pm 0.05}$ & $0.14_{\pm 0.04}$ & \underline{ \textcolor{red}{$0.86_{\pm 0.18}$}} \\
Hopper-v2                        & $U(1,32)$                                                                      & $0.05_{\pm 0.01}$ & $0.07_{\pm 0.09}$                        & \multicolumn{1}{c|}{$0.05_{\pm 0.01}$} & $0.05_{\pm 0.01}$ & $0.04_{\pm 0.01}$ & $0.03_{\pm 0.01}$ & \underline{ \textcolor{red}{$0.43_{\pm 0.21}$}} \\
                                 & $U(1,128)$                                                                     & $0.03_{\pm 0.01}$ & $0.04_{\pm 0.01}$                        & \multicolumn{1}{c|}{$0.04_{\pm 0.02}$} & $0.05_{\pm 0.00}$ & $0.03_{\pm 0.01}$ & $0.03_{\pm 0.00}$ & \underline{ \textcolor{red}{$0.14_{\pm 0.01}$}} \\ \hline
                                 & $U(1,8)$                                                                       & $0.36_{\pm 0.24}$ & $0.40_{\pm 0.32}$                        & \multicolumn{1}{c|}{$0.41_{\pm 0.15}$} & $0.07_{\pm 0.01}$ & $0.07_{\pm 0.05}$ & $0.12_{\pm 0.04}$ & \underline{ \textcolor{red}{$1.11_{\pm 0.10}$}} \\
Walker2d-v2                      & $U(1,32)$                                                                      & $0.12_{\pm 0.03}$ & $0.16_{\pm 0.04}$                        & \multicolumn{1}{c|}{$0.11_{\pm 0.05}$} & $0.09_{\pm 0.04}$ & $0.12_{\pm 0.04}$ & $0.05_{\pm 0.02}$ & \underline{ \textcolor{red}{$0.67_{\pm 0.15}$}} \\
                                 & $U(1,128)$                                                                     & $0.06_{\pm 0.01}$ & $0.06_{\pm 0.06}$                        & $0.04_{\pm 0.02}$                      & $0.10_{\pm 0.04}$ & $0.15_{\pm 0.07}$ & $0.03_{\pm 0.04}$ & \underline{ \textcolor{red}{$0.30_{\pm 0.13}$}} \\ \hline
\end{tabular}
}
\end{table*}

\section{Experiments}
\label{sec:experiments}

\begin{figure*}[h]
    \vskip -0.1in
    \centering
    \centerline{
        \subfigure[HalfCheetah-v2 (128 Delays)]{\includegraphics[width=0.33\linewidth]{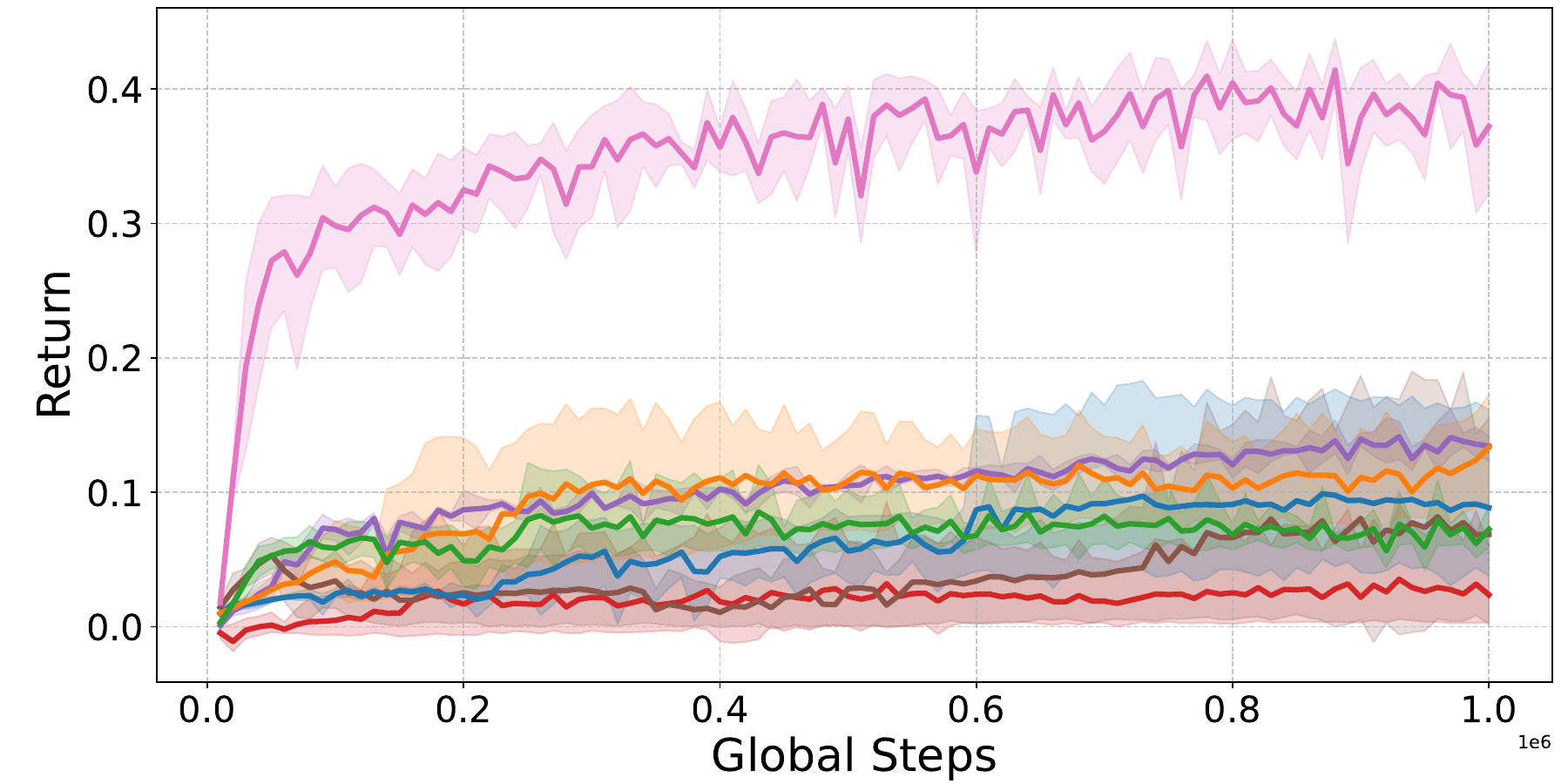}}
        \subfigure[Hopper-v2 (128 Delays)]{\includegraphics[width=0.33\linewidth]{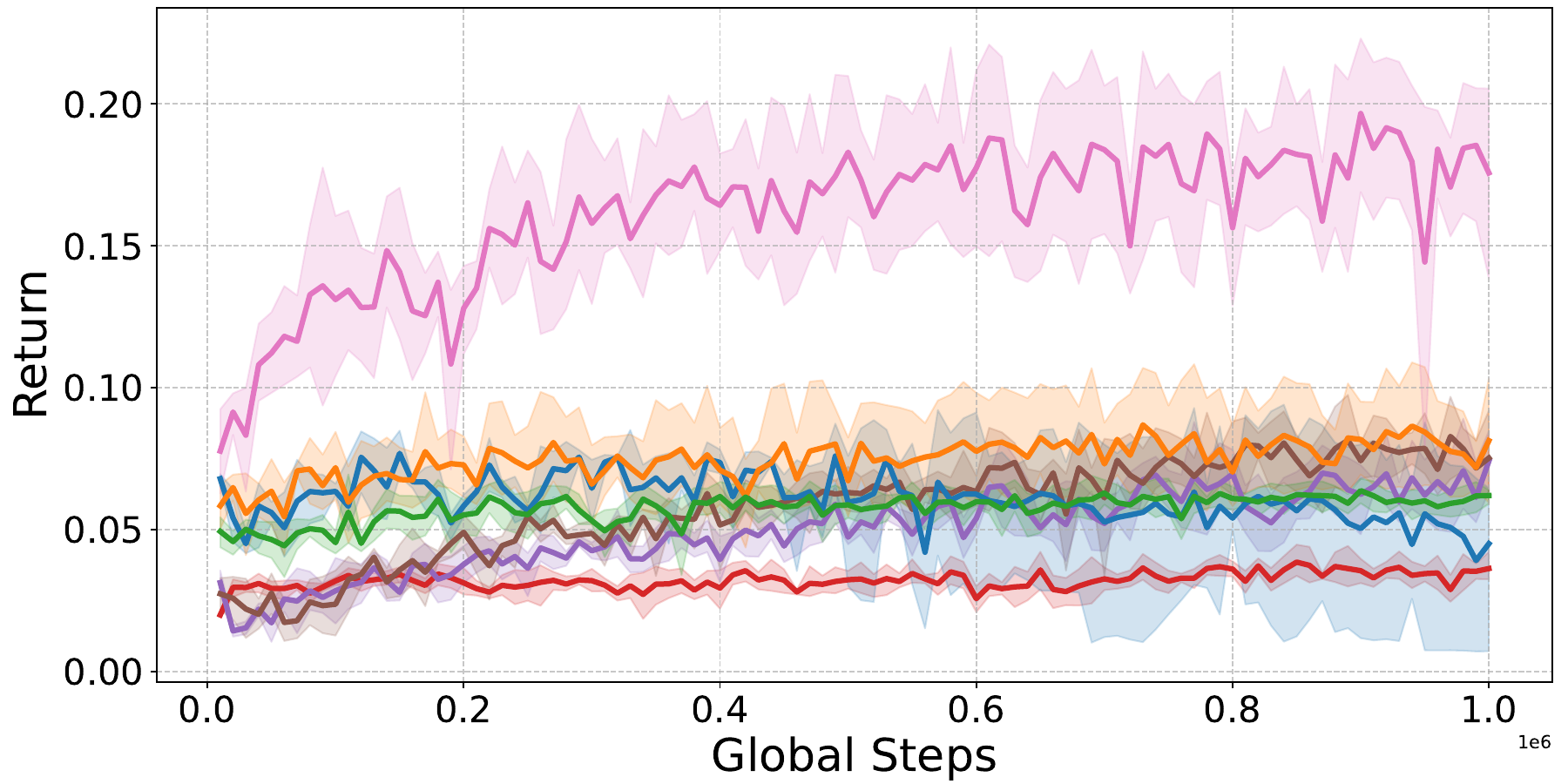}}
        \subfigure[Walker2d-v2 (128 Delays)]{\includegraphics[width=0.33\linewidth]{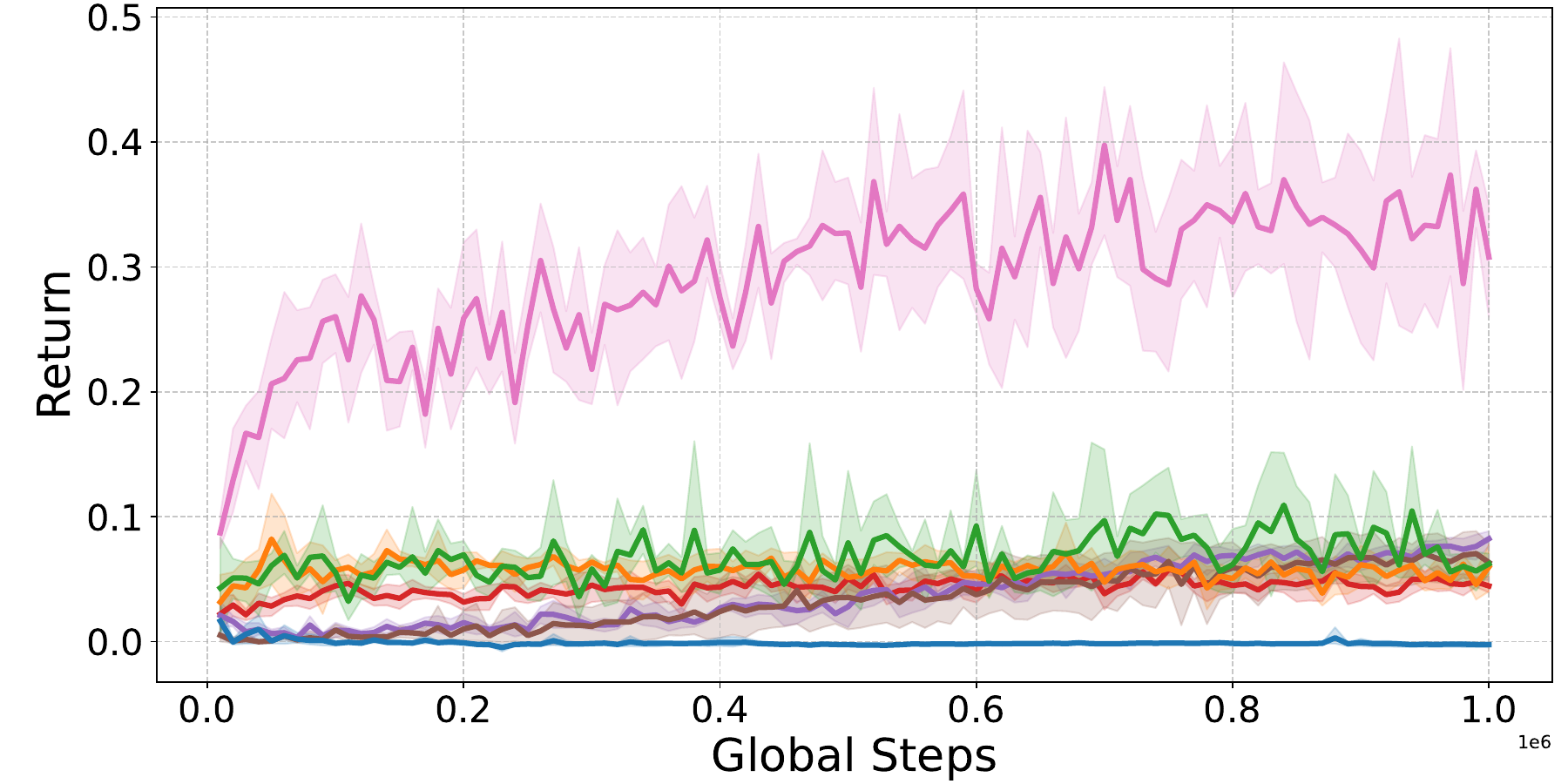}}
    }
    \vskip -0.1in
    \centerline{
        \subfigure[HalfCheetah-v2 ($U(1, 128)$ Delays)]{\includegraphics[width=0.33\linewidth]{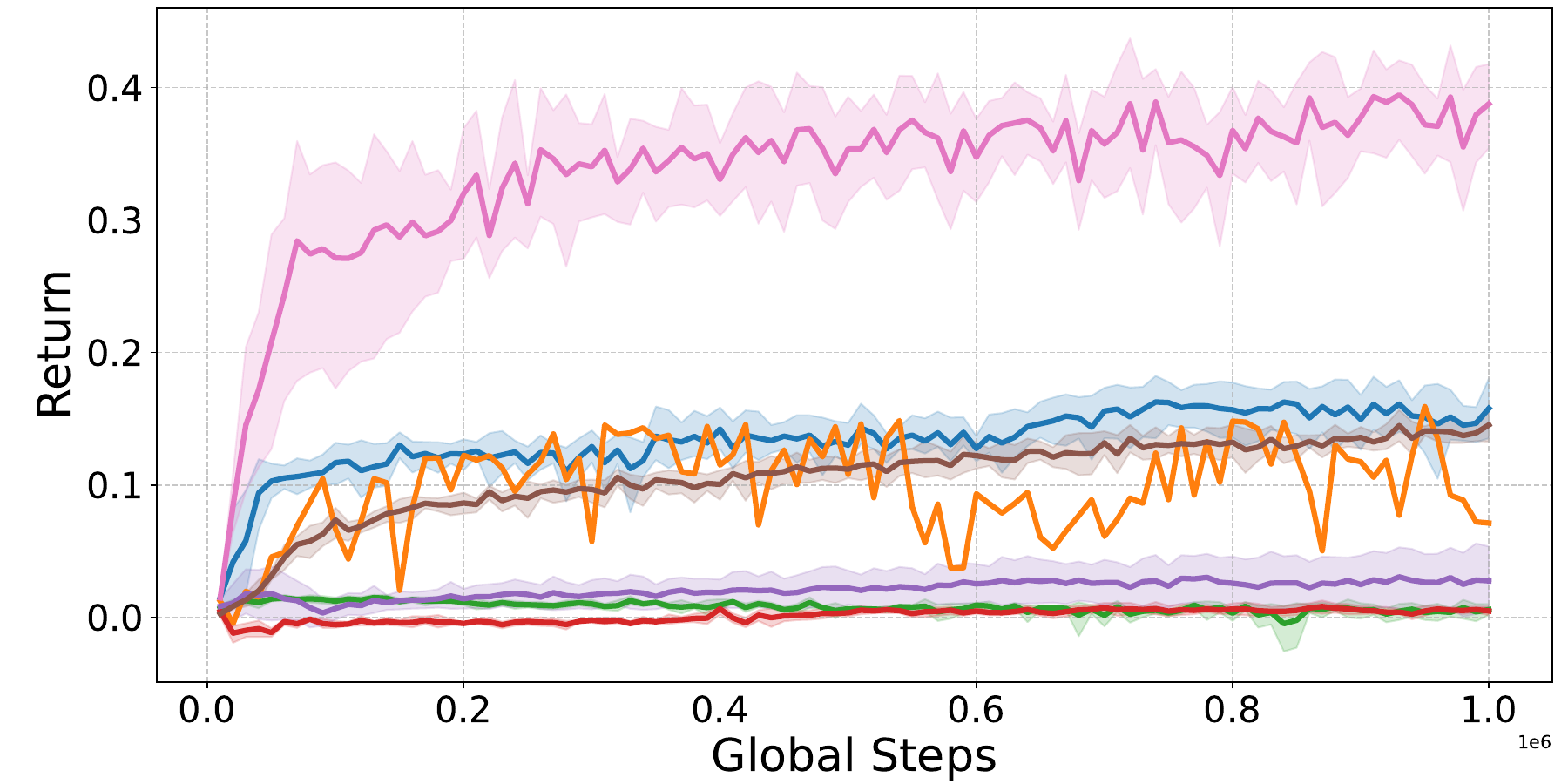}}
        \subfigure[Hopper-v2 ($U(1, 128)$ Delays)]{\includegraphics[width=0.33\linewidth]{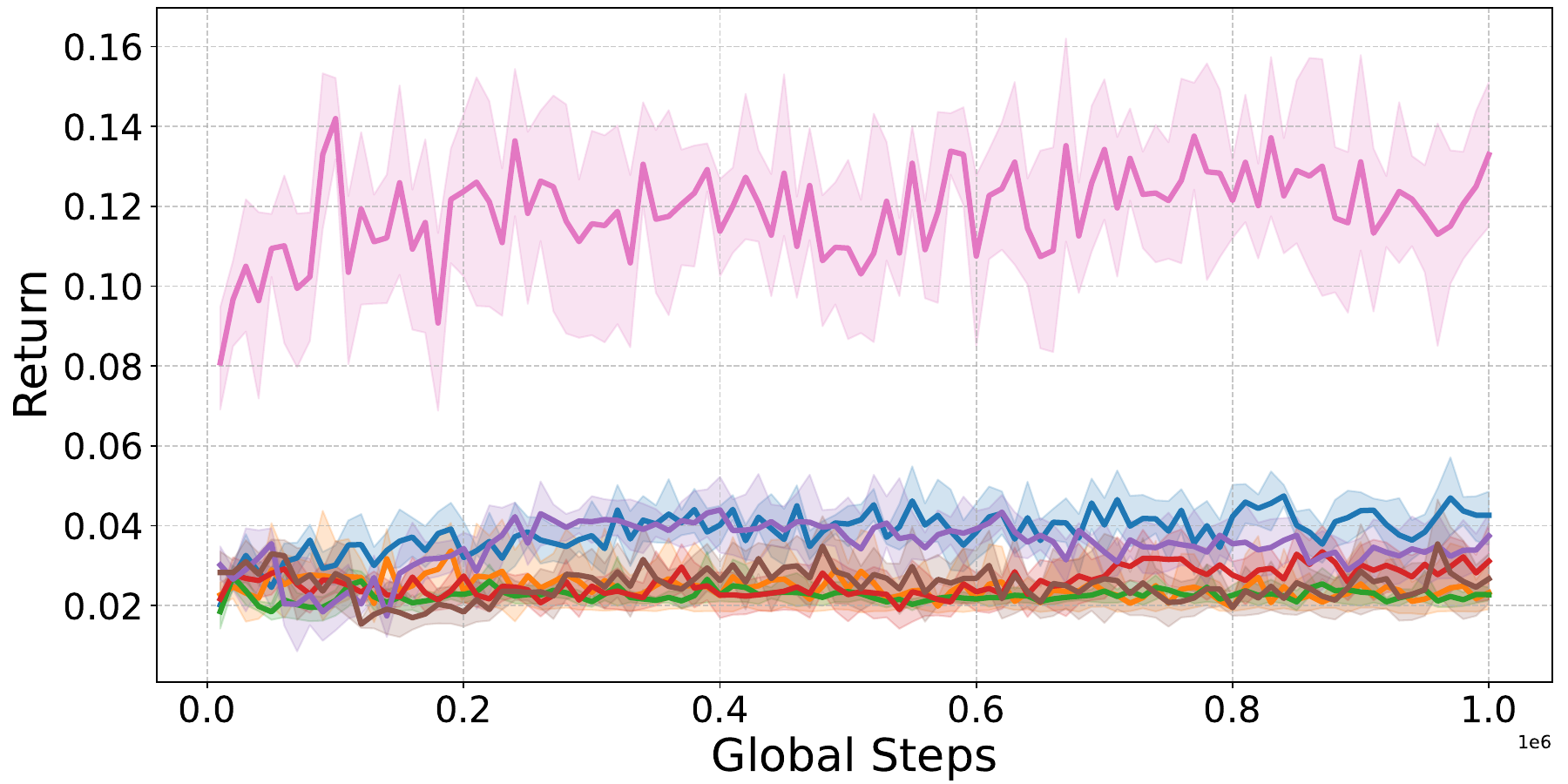}}
        \subfigure[Walker2d-v2 ($U(1, 128)$ Delays)]{\includegraphics[width=0.33\linewidth]{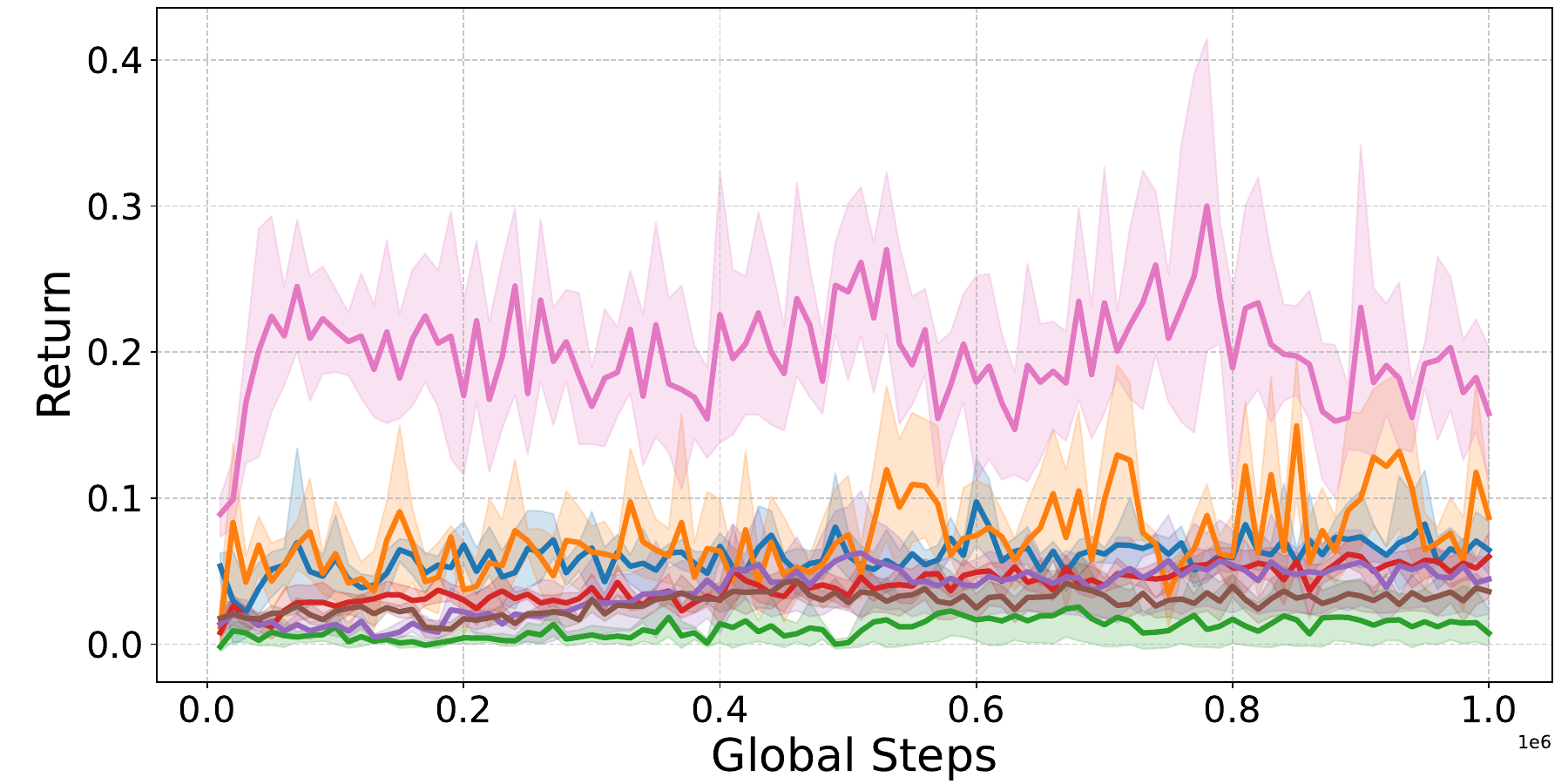}}
    }
    \vskip -0.1in
    \includegraphics[width=1.\linewidth]{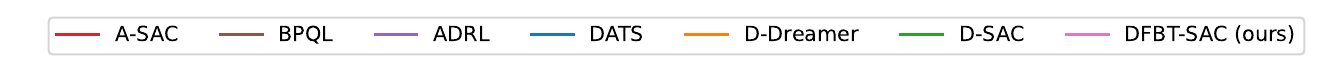}
    \vskip -0.1in
    \caption{Learning Curves on MuJoCo with 128 Delays and $U(1, 128)$ Delays.\label{exp:learning_curves}}
    \vskip -0.1in
\end{figure*}

\subsection{Experimental Setting}
We adopt D4RL~\citep{fu2020d4rl} and MuJoCo~\citep{mujoco} as the offline dataset and the benchmark respectively to evaluate our DFBT-SAC.
For the baselines, we choose the SOTA augmentation-based methods (A-SAC~\citep{soft_actor_critic_application}, BPQL~\citep{kim2023belief}, and ADRL~\citep{wu2024boosting}) and belief-based methods (DATS~\citep{chen2021delay}, D-Dreamer~\citep{karamzade2024reinforcement}, and D-SAC~\citep{learning_belief}).
All of the belief-based baselines and our DFBT-SAC are trained on the same D4RL dataset.
We first investigate the prediction accuracy of beliefs (\Cref{exp:delayed_belief_accuracy}), followed by the performance comparison with deterministic and stochastic delays (\Cref{exp:performance_comparison}).
Additionally, we conduct ablation studies on the multi-step bootstrapping of our DFBT-SAC  (\Cref{exp:ablation_studies}).

\subsection{Experimental Results}
\subsubsection{Belief Prediction Accuracy}
\label{exp:delayed_belief_accuracy}
We first evaluate the state prediction accuracy of our DFBT in D4RL offline datasets. Using the L1 norm as the belief error, We report the error curves increasing with delays of the belief of DATS, D-Dreamer, D-SAC and our DFBT-SAC in \Cref{fig:delayed_belief_accuracy}.
All methods are trained on the D4RL mixed dataset including random, medium and expert policy demonstrations.
The implementation details are provided in \Cref{appendix:sec:implementation_details}.
From the results, we can tell that our DFBT can address the compounding errors effectively via directly forecasting delayed observations, thus maintaining the best prediction accuracy with increased delays, which is consistent with our theoretical results.
We further provide the belief qualitative comparison in \Cref{appendix:sec:belief_qualitative_comparison}. 

\subsubsection{Performance Comparison}
\label{exp:performance_comparison}
We report the normalized return $R_\text{norm}:= \frac{R_\text{alg} - R_\text{random}}{R_\text{sac} - R_\text{random}}$, where $R_\text{alg}, R_\text{sac}, R_\text{random}$ are the return of the algorithm, delay-free SAC, and random returns, respectively.
Each method is evaluated across 5 random seeds, and the implementation details are provided in the \Cref{appendix:sec:implementation_details}.
Additional experiment results and learning curves for different delays are provided in the \Cref{appendix:sec:additional_results} and \Cref{appendix:sec:learning_curves}, respectively.

\paragraph{Deterministic Delays.}
The performance of DFBT-SAC and baselines are evaluated on MuJoCo with deterministic delays (8, 32, and 128) reported in \Cref{table:deterministic_delays}, showing that our DFBT-SAC overall outperforms other belief-based methods significantly.
Specifically, DFBT-SAC shows comparable performance with the SOTA augmentation-based methods (BPQL and ADRL) on tasks with relatively short delays (8). However, our DFBT-SAC yields a leading performance on challenging tasks when the delays increase to long delays (32 and 128).
Additionally, we also compare DFBT-SAC(1), the single-step bootstrapping version of DFBT-SAC.
The results imply that the multi-step bootstrapping technique can further improve the performance effectively, validating our statement in \Cref{sec:multi_step_bootstrap}.
We will investigate the bootstrapping steps in the later ablation study.

\paragraph{Stochastic Delays.}
We also evaluate DFBT-SAC on MuJoCo with stochastic delays which follow the uniform distribution $U$.
As shown in \Cref{table:stochastic_delays}, our DFBT-SAC remarkably outperforms all belief-based baselines on all tasks with all delays settings ($U(1, 8)$, $U(1, 32)$, and $U(1, 128)$).
Especially for $U(1, 128)$ delays, DFBT-SAC also performs approximately $144\%$, $180\%$, and $100\%$ better than the second best baselines on HalfCheetah-v2, Hopper-v2, and Walker2d-v2, respectively.

\paragraph{Learning Curves on MuJoCo with Challenging Delays.}
As summarized in \Cref{exp:learning_curves}, we report the learning curves on MuJoCo with $128$ and $U(1, 128)$ delays.
Our DFBT-SAC exhibits a leading learning efficiency across all these challenging delays.
Especially in the MuJoCo with 128 delays, DFBT-SAC can learn remarkably faster than all baselines, resulting in $173.3\%$ (HalfCheetah-v2), $122.2\%$ (Hopper-v2), and $263.6\%$ (Walker2d-v2) final performance improvement better than the second-best baselines.

\paragraph{Results Analysis and Discussion.}
Based on the above experimental results, we can observe a general trend in both deterministic and stochastic delays that DFBT-SAC shows comparable performance with SOTA augmentation-based methods under short delays scenarios, in which augmentation-based methods generally have better performance than other belief-based approaches. Since when delays are short, the dimensionality of the augmented state space is proper for learning, and without the approximation errors from belief representation.
However, as delays increase in both scenarios, DFBT-SAC unleashes its advantages. 
This empirical finding validates our theoretical analysis in \Cref{sec:theoretical_analysis}. 
For augmentation-based methods, as the delays increase, the augmented state in augmentation-based methods grows too rapidly, making efficient policy optimization infeasible. 
For recursively forecasting belief methods, the compounding errors in belief function approximation grow exponentially with increasing delays, as demonstrated in \Cref{fig:delayed_belief_accuracy} and \Cref{exp:delayed_belief_accuracy}. 
This results in inaccurate predictions of delayed observations and undermines subsequent policy training. 
In contrast, DFBT's prediction error remains largely unaffected by the length of delays, enabling DFBT-SAC to maintain strong performance even in scenarios with long delays.

\subsubsection{Ablation Study on Bootstrapping Steps}
\label{exp:ablation_studies}

As mentioned in \Cref{sec:multi_step_bootstrap}, we conduct the ablation study on the bootstrapping steps in DFBT-SAC. 
The result presented in \Cref{table:ablation_bootstrap} tells us that DFBT-SAC with $8$ bootstrapping steps achieves the averaged best performance compared to other choices.
It is also confirmed that bootstrapping with the states predicted by the trained DFBT can effectively improve performance.

\begin{table}[h]
\centering
\vskip -0.1in
\caption{Final Performance on Walker2d-v2 of DFBT-SAC with different bootstrapping steps. The best performance is underlined.\label{table:ablation_bootstrap}}
\scalebox{0.9}{
\begin{tabular}{c|cccc}
\hline
\multirow{2}{*}{Delays} & \multicolumn{4}{c}{Bootstrapping Steps $N$}                                        \\
                        & 1                 & 2                 & 4                 & 8                 \\ \hline
8                       & $0.86_{\pm 0.07}$ & $0.96_{\pm 0.07}$ & $1.00_{\pm 0.14}$ & \underline{$1.11_{\pm 0.10}$} \\
16                      & $0.84_{\pm 0.24}$ & $0.76_{\pm 0.21}$ & \underline{$1.02_{\pm 0.12}$} & $0.99_{\pm 0.06}$ \\
32                      & $0.63_{\pm 0.22}$ & $0.53_{\pm 0.15}$ & $0.67_{\pm 0.20}$ & \underline{$0.67_{\pm 0.15}$} \\
64                      & $0.34_{\pm 0.24}$ & $0.28_{\pm 0.22}$ & $0.29_{\pm 0.15}$ & \underline{$0.41_{\pm 0.10}$} \\
128                     & $0.24_{\pm 0.03}$ & $0.29_{\pm 0.07}$ & $0.27_{\pm 0.14}$ & \underline{$0.30_{\pm 0.13}$} \\ \hline
\end{tabular}
}
\vskip -0.1in
\end{table}

\begin{table}[h]
\centering
\caption{Performance comparison of DFBT-SAC with different belief training methods on MuJoCo tasks with 32 delays. The best performance is underlined.}
\scalebox{0.8}{
\begin{tabular}{c|ccc}
\hline
Task           & Online         & Offline                 & Offline + Fine-tuning   \\ \hline
HalfCheetah-v2 & $0.11\pm 0.39$ & \underline{$0.42\pm 0.12$} & $0.39\pm 0.04$          \\
Hopper-v2      & $0.10\pm 0.58$ & $0.68\pm 0.20$          & \underline{$0.84\pm 0.04$} \\
Walker2d-v2    & $0.09\pm 0.27$ & $0.64\pm 0.10$          & \underline{$0.96\pm 0.32$} \\ \hline
\end{tabular}
}
\vskip -0.1in
\label{table:finetune_belief}
\end{table}

\subsubsection{Evaluation on Additional MuJoCo tasks.}
we conducted experiments on the Pusher-v2, Reacher-v2, and Swimmer-v2 with deterministic 32 delays. The offline datasets (500k samples) are collected from a SAC policy, with other settings unchanged. The results are shown in \Cref{table:additional_mujoco}, showing that DFBT-SAC achieves superior performance on these tasks.

\begin{table*}[h]
\centering
\vskip -0.1in
\caption{Performance comparison on additional MuJoCo tasks with 32 deterministic delays. The best performance is underlined.}
\scalebox{0.9}{
\begin{tabular}{c|ccccccc}
\hline
Task       & A-SAC          & BPQL           & ADRL           & DATS           & D-Dreamer      & D-SAC          & DFBT-SAC (ours)                \\ \hline
Pusher-v2     & $0.05\pm 0.00$ & $0.93\pm 0.58$ & $0.95\pm 0.24$ & $0.92\pm 0.10$ & $0.87\pm 0.18$     & $0.94\pm 0.04$ & \underline{$1.04\pm 0.24$} \\
Reacher-v2    & $0.89\pm 0.08$ & $0.83\pm 0.06$ & $0.85\pm 0.01$ & $0.82\pm 0.13$ & $0.84\pm 0.02$     & $0.88\pm 0.07$ & \underline{$0.93\pm 0.06$} \\
Swimmer-v2    & $0.27\pm 0.05$ & $0.80\pm 0.14$ & $0.60\pm 0.06$ & $0.25\pm 0.05$ & $0.21\pm 0.07$     & $0.30\pm 0.07$ & \underline{$1.01\pm 0.27$} \\ \hline
\end{tabular}
}
\label{table:additional_mujoco}
\vskip -0.1in
\end{table*}

\begin{table*}[t]
\vskip -0.1in
\caption{Performance on stochastic MuJoCo with deterministic 128 delays. The best performance is underlined.}
\scalebox{0.9}{
\begin{tabular}{c|ccccccc}
\hline
Task           & A-SAC          & BPQL           & ADRL           & DATS           & D-Dreamer      & D-SAC          & DFBT-SAC (ours)               \\ \hline
HalfCheetah-v2 & $0.00\pm 0.03$ & $0.01\pm 0.05$ & $0.13\pm 0.04$ & $0.13\pm 0.03$ & $0.07\pm 0.01$ & $0.00\pm 0.04$ & \underline{$0.35\pm 0.04$} \\
Hopper-v2      & $0.03\pm 0.04$ & $0.08\pm 0.05$ & $0.06\pm 0.04$ & $0.05\pm 0.06$ & $0.04\pm 0.05$ & $0.04\pm 0.05$ & \underline{$0.13\pm 0.22$} \\
Walker2d-v2    & $0.06\pm 0.02$ & $0.04\pm 0.01$ & $0.08\pm 0.01$ & $0.06\pm 0.03$ & $0.10\pm 0.03$ & $0.05\pm 0.02$ & \underline{$0.30\pm 0.07$} \\ \hline
\end{tabular}
}
\label{table:stochastic_mujoco}
\vskip -0.1in
\end{table*}

\subsubsection{Ablation Results of Belief Training.}
Some task-specific information may be missing if the belief is frozen in the RL process, leading to limited performance improvement. This issue can be mitigated by fine-tuning the belief within the RL process. The results, shown in \Cref{table:finetune_belief}, demonstrate that fine-tuning helps the DFBT capture the task-specific information with better performance. Note that there are many potential methods for capturing task-specific information, not limited to fine-tuning DFBT. Belief learning from scratch in the online RL process always suffers from instability issues. Therefore, in this paper, we separate belief learning from the online RL process, which allows us to investigate the belief component solely, eliminate potential influences from the RL side.

\subsubsection{Stochastic MuJoCo.}
We conducted additional experiments on the stochastic MuJoCo tasks with a probability of 0.001 for the unaware noise and deterministic 128 delays. As shown in \Cref{table:stochastic_mujoco}, the results demonstrate that our DFBT-SAC achieves superior performance in these stochastic MuJoCo tasks.

\subsubsection{Inference Speed Comparison.}
We conducted additional experiments on computational efficiency. As shown in \Cref{table:inference_speed}, the results demonstrate that directly forecasting belief maintains a consistent and stable inference speed (around 4 ms) across different delays. In contrast, the recursively forecasting belief experiences inference speed issues as delays increase. In HalfCheetah-v2 with 128 delays, the training times of DATS and D-Dreamer are around 10 hours and 15 hours, respectively, while those of D-SAC and DFBT-SAC are both around 6 hours.
\begin{table}[h]
\centering
\caption{Inference speed (ms) comparison in HalfCheetah-v2.}
\scalebox{0.8}{
\begin{tabular}{c|cccc}
\hline
Delays & DATS            & D-Dreamer       & D-SAC          & DFBT-SAC       \\ \hline
8      & $1.10\pm 0.02$  & $1.85\pm 0.01$  & $4.03\pm 0.03$ & $4.18\pm 0.04$ \\
32     & $3.85\pm 0.06$  & $6.80\pm 0.04$  & $4.03\pm 0.04$ & $4.18\pm 0.04$ \\
128    & $14.97\pm 0.22$ & $26.51\pm 0.19$ & $4.03\pm 0.03$ & $4.15\pm 0.05$ \\ \hline
\end{tabular}
}
\label{table:inference_speed}
\vskip -0.1in
\end{table}
\subsection{Limitations and Challenges}
In this work, we empirically validate that our approach can address the compounding errors of recursively forecasting belief with significant performance improvement. 
However, there remain some limitations and challenges as follows.

\paragraph{Online Belief Learning.}
In this paper, we mainly consider learning DFBT from the offline dataset, which means the performance of the belief is subject to the quality and quantity of the dataset.
However, direct learning belief in the online environment will introduce an auxiliary representation task, destabilizing the learning process.

\paragraph{Sample Complexity of Directly Forecasting Belief.}
Although we have empirically demonstrated that directly forecasting belief can effectively reduce the compounding errors of recursively forecasting belief.
However, it is worth theoretically analysing the sample complexity of belief learning, especially in online scenarios.

\section{Conclusion}
This work investigates the challenges of RL in environments where inherent delays exist between actions and their corresponding outcomes.
Existing belief-based approaches usually suffer from the compounding errors issue of the recursively forecasting belief as the delays are increased, seriously hindering the performance.
To resolve this issue, we present DFBT, a new directly forecasting belief representation method.
Furthermore, we present DFBT-SAC, which facilitates multi-step bootstrapping in learning the value function via the state prediction from the DFBT, effectively improving the sample efficiency.
We demonstrate that our DFBT greatly reduces compounding errors, yielding stronger performance guarantees.
Our empirical results validate that DFBT has remarkable prediction accuracy in the D4RL datasets.
We also empirically show that our DFBT-SAC not only effectively enhance the learning efficiency with superior performance in the MuJoCo benchmark.

\clearpage
\newpage
\section*{Acknowledgement}
We sincerely appreciate the insightful suggestions provided by ICML 2025 reviewers, which help us improve the quality of this paper.
We sincerely acknowledge the support by the grant EP/Y002644/1 under the EPSRC ECR International Collaboration Grants program, funded by the International Science Partnerships Fund (ISPF) and the UK Research and Innovation. 
We also sincerely acknowledge the support by Taiwan NSTC under Grant Number NSTC-112-2221-E-002-168-MY3, and from King Abdullah University of Science and Technology (KAUST) - Center of Excellence for Generative AI, under award number 5940.
Simon Sinong Zhan and Qi Zhu are partially supported by the US National Science Foundation grants 2324936 and 2328973. 

\section*{Impact Statement}
This paper aims to advance the field of Machine Learning. While acknowledging there are many potential societal consequences of our work, we believe that none need to be specially highlighted here.


\bibliography{example_paper}

\begin{thebibliography}{53}
\providecommand{\natexlab}[1]{#1}
\providecommand{\url}[1]{\texttt{#1}}
\expandafter\ifx\csname urlstyle\endcsname\relax
  \providecommand{\doi}[1]{doi: #1}\else
  \providecommand{\doi}{doi: \begingroup \urlstyle{rm}\Url}\fi

\bibitem[Altman \& Nain(1992)Altman and Nain]{altman_delay}
Altman, E. and Nain, P.
\newblock Closed-loop control with delayed information.
\newblock \emph{ACM sigmetrics performance evaluation review}, 20\penalty0 (1):\penalty0 193--204, 1992.

\bibitem[Arjona-Medina et~al.(2019)Arjona-Medina, Gillhofer, Widrich, Unterthiner, Brandstetter, and Hochreiter]{arjona2019rudder}
Arjona-Medina, J.~A., Gillhofer, M., Widrich, M., Unterthiner, T., Brandstetter, J., and Hochreiter, S.
\newblock Rudder: Return decomposition for delayed rewards.
\newblock \emph{Advances in Neural Information Processing Systems}, 32, 2019.

\bibitem[Asadi et~al.(2018)Asadi, Misra, and Littman]{asadi2018lipschitz}
Asadi, K., Misra, D., and Littman, M.
\newblock Lipschitz continuity in model-based reinforcement learning.
\newblock In \emph{International Conference on Machine Learning}, pp.\  264--273. PMLR, 2018.

\bibitem[Berner et~al.(2019)Berner, Brockman, Chan, Cheung, Debiak, Dennison, Farhi, Fischer, Hashme, Hesse, et~al.]{berner2019dota}
Berner, C., Brockman, G., Chan, B., Cheung, V., Debiak, P., Dennison, C., Farhi, D., Fischer, Q., Hashme, S., Hesse, C., et~al.
\newblock Dota 2 with large scale deep reinforcement learning.
\newblock \emph{arXiv preprint arXiv:1912.06680}, 2019.

\bibitem[Bertsekas(2012)]{bertsekas2012dynamic}
Bertsekas, D.
\newblock \emph{Dynamic programming and optimal control: Volume I}, volume~4.
\newblock Athena scientific, 2012.

\bibitem[Bouteiller et~al.(2020)Bouteiller, Ramstedt, Beltrame, Pal, and Binas]{dcac}
Bouteiller, Y., Ramstedt, S., Beltrame, G., Pal, C., and Binas, J.
\newblock Reinforcement learning with random delays.
\newblock In \emph{International conference on learning representations}, 2020.

\bibitem[Cao et~al.(2020)Cao, Guo, Song, Gao, Chen, Zhang, and Zhang]{cao2020using}
Cao, Z., Guo, H., Song, W., Gao, K., Chen, Z., Zhang, L., and Zhang, X.
\newblock Using reinforcement learning to minimize the probability of delay occurrence in transportation.
\newblock \emph{IEEE transactions on vehicular technology}, 69\penalty0 (3):\penalty0 2424--2436, 2020.

\bibitem[Chen et~al.(2021{\natexlab{a}})Chen, Xu, Li, and Zhao]{chen2021delay}
Chen, B., Xu, M., Li, L., and Zhao, D.
\newblock Delay-aware model-based reinforcement learning for continuous control.
\newblock \emph{Neurocomputing}, 450:\penalty0 119--128, 2021{\natexlab{a}}.

\bibitem[Chen et~al.(2021{\natexlab{b}})Chen, Lu, Rajeswaran, Lee, Grover, Laskin, Abbeel, Srinivas, and Mordatch]{chen2021decision}
Chen, L., Lu, K., Rajeswaran, A., Lee, K., Grover, A., Laskin, M., Abbeel, P., Srinivas, A., and Mordatch, I.
\newblock Decision transformer: Reinforcement learning via sequence modeling.
\newblock \emph{Advances in neural information processing systems}, 34:\penalty0 15084--15097, 2021{\natexlab{b}}.

\bibitem[Chevillon(2007)]{chevillon2007direct}
Chevillon, G.
\newblock Direct multi-step estimation and forecasting.
\newblock \emph{Journal of Economic Surveys}, 21\penalty0 (4):\penalty0 746--785, 2007.

\bibitem[Clements \& Hendry(1996)Clements and Hendry]{clements1996multi}
Clements, M.~P. and Hendry, D.~F.
\newblock Multi-step estimation for forecasting.
\newblock \emph{Oxford Bulletin of Economics and Statistics}, 58\penalty0 (4):\penalty0 657--684, 1996.

\bibitem[Cooke(1963)]{cooke1963differential}
Cooke, K.~L.
\newblock Differential—difference equations.
\newblock In \emph{International symposium on nonlinear differential equations and nonlinear mechanics}, pp.\  155--171. Elsevier, 1963.

\bibitem[Feng et~al.(2019)Feng, Chen, Zhan, Fr{\"a}nzle, and Xue]{feng2019taming}
Feng, S., Chen, M., Zhan, N., Fr{\"a}nzle, M., and Xue, B.
\newblock Taming delays in dynamical systems: Unbounded verification of delay differential equations.
\newblock In \emph{International Conference on Computer Aided Verification}, pp.\  650--669. Springer, 2019.

\bibitem[Fridman \& Shaked(2003)Fridman and Shaked]{fridman2003reachable}
Fridman, E. and Shaked, U.
\newblock On reachable sets for linear systems with delay and bounded peak inputs.
\newblock \emph{Automatica}, 39\penalty0 (11):\penalty0 2005--2010, 2003.

\bibitem[Fu et~al.(2020)Fu, Kumar, Nachum, Tucker, and Levine]{fu2020d4rl}
Fu, J., Kumar, A., Nachum, O., Tucker, G., and Levine, S.
\newblock D4rl: Datasets for deep data-driven reinforcement learning, 2020.

\bibitem[Ha \& Schmidhuber(2018)Ha and Schmidhuber]{ha2018world}
Ha, D. and Schmidhuber, J.
\newblock World models.
\newblock \emph{arXiv preprint arXiv:1803.10122}, 2018.

\bibitem[Haarnoja et~al.(2018)Haarnoja, Zhou, Hartikainen, Tucker, Ha, Tan, Kumar, Zhu, Gupta, Abbeel, et~al.]{soft_actor_critic_application}
Haarnoja, T., Zhou, A., Hartikainen, K., Tucker, G., Ha, S., Tan, J., Kumar, V., Zhu, H., Gupta, A., Abbeel, P., et~al.
\newblock Soft actor-critic algorithms and applications.
\newblock \emph{arXiv preprint arXiv:1812.05905}, 2018.

\bibitem[Hafner et~al.(2019)Hafner, Lillicrap, Ba, and Norouzi]{hafner2019dream}
Hafner, D., Lillicrap, T., Ba, J., and Norouzi, M.
\newblock Dream to control: Learning behaviors by latent imagination.
\newblock \emph{arXiv preprint arXiv:1912.01603}, 2019.

\bibitem[Hasbrouck \& Saar(2013)Hasbrouck and Saar]{hasbrouck2013low}
Hasbrouck, J. and Saar, G.
\newblock Low-latency trading.
\newblock \emph{Journal of Financial Markets}, 16\penalty0 (4):\penalty0 646--679, 2013.

\bibitem[Hessel et~al.(2018)Hessel, Modayil, Van~Hasselt, Schaul, Ostrovski, Dabney, Horgan, Piot, Azar, and Silver]{hessel2018rainbow}
Hessel, M., Modayil, J., Van~Hasselt, H., Schaul, T., Ostrovski, G., Dabney, W., Horgan, D., Piot, B., Azar, M., and Silver, D.
\newblock Rainbow: Combining improvements in deep reinforcement learning.
\newblock In \emph{Proceedings of the AAAI conference on artificial intelligence}, volume~32, 2018.

\bibitem[Huang et~al.(2022)Huang, Dossa, Ye, Braga, Chakraborty, Mehta, and Ara{\~A}{\v{s}}jo]{huang2022cleanrl}
Huang, S., Dossa, R. F.~J., Ye, C., Braga, J., Chakraborty, D., Mehta, K., and Ara{\~A}{\v{s}}jo, J.~G.
\newblock Cleanrl: High-quality single-file implementations of deep reinforcement learning algorithms.
\newblock \emph{Journal of Machine Learning Research}, 23\penalty0 (274):\penalty0 1--18, 2022.

\bibitem[Hwangbo et~al.(2017)Hwangbo, Sa, Siegwart, and Hutter]{quadrotor_delay}
Hwangbo, J., Sa, I., Siegwart, R., and Hutter, M.
\newblock Control of a quadrotor with reinforcement learning.
\newblock \emph{IEEE Robotics and Automation Letters}, 2\penalty0 (4):\penalty0 2096--2103, 2017.

\bibitem[Janner et~al.(2021)Janner, Li, and Levine]{janner2021offline}
Janner, M., Li, Q., and Levine, S.
\newblock Offline reinforcement learning as one big sequence modeling problem.
\newblock \emph{Advances in neural information processing systems}, 34:\penalty0 1273--1286, 2021.

\bibitem[Karamzade et~al.(2024)Karamzade, Kim, Kalsi, and Fox]{karamzade2024reinforcement}
Karamzade, A., Kim, K., Kalsi, M., and Fox, R.
\newblock Reinforcement learning from delayed observations via world models.
\newblock \emph{arXiv preprint arXiv:2403.12309}, 2024.

\bibitem[Katsikopoulos \& Engelbrecht(2003)Katsikopoulos and Engelbrecht]{delay_mdp}
Katsikopoulos, K.~V. and Engelbrecht, S.~E.
\newblock Markov decision processes with delays and asynchronous cost collection.
\newblock \emph{IEEE transactions on automatic control}, 48\penalty0 (4):\penalty0 568--574, 2003.

\bibitem[Kim et~al.(2023)Kim, Kim, Kang, Baek, and Han]{kim2023belief}
Kim, J., Kim, H., Kang, J., Baek, J., and Han, S.
\newblock Belief projection-based reinforcement learning for environments with delayed feedback.
\newblock In \emph{Thirty-seventh Conference on Neural Information Processing Systems}, 2023.

\bibitem[Liotet et~al.(2021)Liotet, Venneri, and Restelli]{learning_belief}
Liotet, P., Venneri, E., and Restelli, M.
\newblock Learning a belief representation for delayed reinforcement learning.
\newblock In \emph{2021 International Joint Conference on Neural Networks (IJCNN)}, pp.\  1--8. IEEE, 2021.

\bibitem[Liotet et~al.(2022)Liotet, Maran, Bisi, and Restelli]{dida}
Liotet, P., Maran, D., Bisi, L., and Restelli, M.
\newblock Delayed reinforcement learning by imitation.
\newblock In \emph{International Conference on Machine Learning}, pp.\  13528--13556. PMLR, 2022.

\bibitem[Mahmood et~al.(2018)Mahmood, Korenkevych, Komer, and Bergstra]{arm_delay}
Mahmood, A.~R., Korenkevych, D., Komer, B.~J., and Bergstra, J.
\newblock Setting up a reinforcement learning task with a real-world robot.
\newblock In \emph{2018 IEEE/RSJ International Conference on Intelligent Robots and Systems (IROS)}, pp.\  4635--4640. IEEE, 2018.

\bibitem[Mnih et~al.(2013)Mnih, Kavukcuoglu, Silver, Graves, Antonoglou, Wierstra, and Riedmiller]{mnih2013playing}
Mnih, V., Kavukcuoglu, K., Silver, D., Graves, A., Antonoglou, I., Wierstra, D., and Riedmiller, M.
\newblock Playing atari with deep reinforcement learning.
\newblock \emph{arXiv preprint arXiv:1312.5602}, 2013.

\bibitem[Nath et~al.(2021)Nath, Baranwal, and Khadilkar]{revisiting_augment}
Nath, S., Baranwal, M., and Khadilkar, H.
\newblock Revisiting state augmentation methods for reinforcement learning with stochastic delays.
\newblock In \emph{Proceedings of the 30th ACM International Conference on Information \& Knowledge Management}, pp.\  1346--1355, 2021.

\bibitem[Rachelson \& Lagoudakis(2010)Rachelson and Lagoudakis]{rl_lipschitz_continuous}
Rachelson, E. and Lagoudakis, M.~G.
\newblock On the locality of action domination in sequential decision making.
\newblock 2010.

\bibitem[Schmidhuber(1990{\natexlab{a}})]{schmidhuber1990line}
Schmidhuber, J.
\newblock An on-line algorithm for dynamic reinforcement learning and planning in reactive environments.
\newblock In \emph{1990 IJCNN international joint conference on neural networks}, pp.\  253--258. IEEE, 1990{\natexlab{a}}.

\bibitem[Schmidhuber(1990{\natexlab{b}})]{schmidhuber1990making}
Schmidhuber, J.
\newblock \emph{Making the world differentiable: on using self supervised fully recurrent neural networks for dynamic reinforcement learning and planning in non-stationary environments}, volume 126.
\newblock Inst. f{\"u}r Informatik, 1990{\natexlab{b}}.

\bibitem[Schmidhuber(2019)]{schmidhuber2019reinforcement}
Schmidhuber, J.
\newblock Reinforcement learning upside down: Don't predict rewards--just map them to actions.
\newblock \emph{arXiv preprint arXiv:1912.02875}, 2019.

\bibitem[Schrittwieser et~al.(2020)Schrittwieser, Antonoglou, Hubert, Simonyan, Sifre, Schmitt, Guez, Lockhart, Hassabis, Graepel, et~al.]{schrittwieser2020mastering}
Schrittwieser, J., Antonoglou, I., Hubert, T., Simonyan, K., Sifre, L., Schmitt, S., Guez, A., Lockhart, E., Hassabis, D., Graepel, T., et~al.
\newblock Mastering atari, go, chess and shogi by planning with a learned model.
\newblock \emph{Nature}, 588\penalty0 (7839):\penalty0 604--609, 2020.

\bibitem[Silver et~al.(2016)Silver, Huang, Maddison, Guez, Sifre, Van Den~Driessche, Schrittwieser, Antonoglou, Panneershelvam, Lanctot, et~al.]{silver2016mastering}
Silver, D., Huang, A., Maddison, C.~J., Guez, A., Sifre, L., Van Den~Driessche, G., Schrittwieser, J., Antonoglou, I., Panneershelvam, V., Lanctot, M., et~al.
\newblock Mastering the game of go with deep neural networks and tree search.
\newblock \emph{nature}, 529\penalty0 (7587):\penalty0 484--489, 2016.

\bibitem[Sutton \& Barto(2018)Sutton and Barto]{rlai}
Sutton, R.~S. and Barto, A.~G.
\newblock \emph{Reinforcement learning: An introduction}.
\newblock MIT press, 2018.

\bibitem[Taieb \& Atiya(2015)Taieb and Atiya]{taieb2015bias}
Taieb, S.~B. and Atiya, A.~F.
\newblock A bias and variance analysis for multistep-ahead time series forecasting.
\newblock \emph{IEEE transactions on neural networks and learning systems}, 27\penalty0 (1):\penalty0 62--76, 2015.

\bibitem[Tarasov et~al.(2022)Tarasov, Nikulin, Akimov, Kurenkov, and Kolesnikov]{tarasov2022corl}
Tarasov, D., Nikulin, A., Akimov, D., Kurenkov, V., and Kolesnikov, S.
\newblock {CORL}: Research-oriented deep offline reinforcement learning library.
\newblock In \emph{3rd Offline RL Workshop: Offline RL as a ''Launchpad''}, 2022.
\newblock URL \url{https://openreview.net/forum?id=SyAS49bBcv}.

\bibitem[Todorov et~al.(2012)Todorov, Erez, and Tassa]{mujoco}
Todorov, E., Erez, T., and Tassa, Y.
\newblock Mujoco: A physics engine for model-based control.
\newblock In \emph{2012 IEEE/RSJ international conference on intelligent robots and systems}, pp.\  5026--5033. IEEE, 2012.

\bibitem[Vaswani et~al.(2017)Vaswani, Shazeer, Parmar, Uszkoreit, Jones, Gomez, Kaiser, and Polosukhin]{transformer}
Vaswani, A., Shazeer, N., Parmar, N., Uszkoreit, J., Jones, L., Gomez, A.~N., Kaiser, {\L}., and Polosukhin, I.
\newblock Attention is all you need.
\newblock \emph{Advances in neural information processing systems}, 30, 2017.

\bibitem[Walsh et~al.(2009)Walsh, Nouri, Li, and Littman]{learning_planning_delayed_feedback}
Walsh, T.~J., Nouri, A., Li, L., and Littman, M.~L.
\newblock Learning and planning in environments with delayed feedback.
\newblock \emph{Autonomous Agents and Multi-Agent Systems}, 18:\penalty0 83--105, 2009.

\bibitem[Wang et~al.(2023{\natexlab{a}})Wang, Han, Luo, and Li]{wang2023addressing}
Wang, W., Han, D., Luo, X., and Li, D.
\newblock Addressing signal delay in deep reinforcement learning.
\newblock In \emph{The Twelfth International Conference on Learning Representations}, 2023{\natexlab{a}}.

\bibitem[Wang et~al.(2023{\natexlab{b}})Wang, Zhan, Wang, Huang, Wang, Yang, and Zhu]{wang2023joint}
Wang, Y., Zhan, S., Wang, Z., Huang, C., Wang, Z., Yang, Z., and Zhu, Q.
\newblock Joint differentiable optimization and verification for certified reinforcement learning.
\newblock In \emph{Proceedings of the ACM/IEEE 14th International Conference on Cyber-Physical Systems (with CPS-IoT Week 2023)}, pp.\  132--141, 2023{\natexlab{b}}.

\bibitem[Wang et~al.(2023{\natexlab{c}})Wang, Zhan, Jiao, Wang, Jin, Yang, Wang, Huang, and Zhu]{wang2023enforcing}
Wang, Y., Zhan, S.~S., Jiao, R., Wang, Z., Jin, W., Yang, Z., Wang, Z., Huang, C., and Zhu, Q.
\newblock Enforcing hard constraints with soft barriers: Safe reinforcement learning in unknown stochastic environments.
\newblock In \emph{International Conference on Machine Learning}, pp.\  36593--36604. PMLR, 2023{\natexlab{c}}.

\bibitem[Wang et~al.(2024)Wang, Strupl, Faccio, Wu, Liu, Grudzie{\'n}, Tan, and Schmidhuber]{wang2024highway}
Wang, Y., Strupl, M., Faccio, F., Wu, Q., Liu, H., Grudzie{\'n}, M., Tan, X., and Schmidhuber, J.
\newblock Highway reinforcement learning.
\newblock \emph{arXiv preprint arXiv:2405.18289}, 2024.

\bibitem[Wei et~al.(2017)Wei, Wang, and Zhu]{wei2017deep}
Wei, T., Wang, Y., and Zhu, Q.
\newblock Deep reinforcement learning for building hvac control.
\newblock In \emph{2017 54th ACM/EDAC/IEEE Design Automation Conference (DAC)}, pp.\  1--6, June 2017.
\newblock \doi{10.1145/3061639.3062224}.

\bibitem[Wu et~al.(2024{\natexlab{a}})Wu, Zhan, Wang, Wang, Lin, Lv, Zhu, and Huang]{wu2024variational}
Wu, Q., Zhan, S.~S., Wang, Y., Wang, Y., Lin, C.-W., Lv, C., Zhu, Q., and Huang, C.
\newblock Variational delayed policy optimization.
\newblock \emph{Advances in neural information processing systems}, 2024{\natexlab{a}}.

\bibitem[Wu et~al.(2024{\natexlab{b}})Wu, Zhan, Wang, Wang, Lin, Lv, Zhu, Schmidhuber, and Huang]{wu2024boosting}
Wu, Q., Zhan, S.~S., Wang, Y., Wang, Y., Lin, C.-W., Lv, C., Zhu, Q., Schmidhuber, J., and Huang, C.
\newblock Boosting reinforcement learning with strongly delayed feedback through auxiliary short delays.
\newblock In \emph{Forty-first International Conference on Machine Learning (ICML 2024)}, 2024{\natexlab{b}}.

\bibitem[Xue et~al.(2020)Xue, Wang, Feng, and Zhan]{xue2020over}
Xue, B., Wang, Q., Feng, S., and Zhan, N.
\newblock Over-and underapproximating reach sets for perturbed delay differential equations.
\newblock \emph{IEEE Transactions on Automatic Control}, 66\penalty0 (1):\penalty0 283--290, 2020.

\bibitem[Xue et~al.(2021)Xue, Bai, Zhan, Liu, and Jiao]{xue2021reach}
Xue, B., Bai, Y., Zhan, N., Liu, W., and Jiao, L.
\newblock Reach-avoid analysis for delay differential equations.
\newblock In \emph{2021 60th IEEE Conference on Decision and Control (CDC)}, pp.\  1301--1307. IEEE, 2021.

\bibitem[Zhan et~al.(2024)Zhan, Wang, Wu, Jiao, Huang, and Zhu]{zhan2024state}
Zhan, S., Wang, Y., Wu, Q., Jiao, R., Huang, C., and Zhu, Q.
\newblock State-wise safe reinforcement learning with pixel observations.
\newblock In \emph{6th Annual Learning for Dynamics \& Control Conference}, pp.\  1187--1201. PMLR, 2024.

\end{thebibliography}
\bibliographystyle{icml2025}

\newpage
\appendix
\onecolumn

\section{Implementation Details}
\label{appendix:sec:implementation_details}
The implementation of DFBT and DFBT-SAC is based on CORL~\citep{tarasov2022corl} and CleanRL~\citep{huang2022cleanrl}.
The codebase for reproducing our experimental results is also provided in the Supplementary Material.
We detail the hyperparameter settings of DFBT and DFBT-SAC in \Cref{appendix:table:parameter_dbt} and \Cref{appendix:table:parameter_dbt_sac}, respectively.

\begin{table}[h]
\centering
\caption{Hyper-parameters table of DFBT. \label{appendix:table:parameter_dbt}}
\begin{tabular}{|c|c|}
\hline
Hyper-parameter        & Value                                     \\ \hline
Epoch                  & 1e3                                       \\
Batch Size             & 256                                       \\
Attention Heads Num    & 4                                         \\
Layers Num             & 10                                        \\
Hidden Dim             & 256                                       \\
Attention Dropout Rate & 0.1                                       \\
Residual Dropout Rate  & 0.1                                       \\
Hidden Dropout Rate    & 0.1                                       \\
Learning Rate          & 1e-4                                      \\
Optimizer              & AdamW                                     \\ 
Weight Decay           & 1e-4                                     \\ 
Betas                  & (0.9, 0.999)                                     \\ 
\hline
\end{tabular}
\end{table}

\begin{table}[h]
\centering
\caption{Hyper-parameters table of DFBT-SAC. \label{appendix:table:parameter_dbt_sac}}
\begin{tabular}{|c|c|}
\hline
Hyper-parameter             & Value          \\ \hline
Bootstrapping Steps $N$       & 8           \\
Learning Rate (Actor)       & 3e-4           \\
Learning Rate (Critic)      & 1e-3           \\
Learning Rate (Entropy)     & 1e-3           \\
Train Frequency (Actor)     & 2              \\
Train Frequency (Critic)    & 1              \\
Soft Update Factor (Critic) & 5e-3           \\
Batch Size                  & 256            \\
Neurons                     & {[}256, 256{]} \\
Layers                      & 3              \\
Hidden Dim                  & 256            \\
Activation                  & ReLU           \\
Optimizer                   & Adam           \\ \hline
\end{tabular}
\end{table}

\clearpage
\section{Theoretical Analysis}
\begin{theorem}[Performance Difference of Recursively Forecasting Belief]
    \label{appendix:theo_delayed_performance_diff_single}
    For the delay-free policy $\pi$ and the delayed policy $\pi_\Delta$. Given any $x_t \in \mathcal{X}$, the performance difference $I^\text{recursive}(x_t)$ of the recursively forecasting belief $b_\theta$ can be bounded as follows, respectively.

    For deterministic delays $\Delta$, we have
    $$
    \left|I^\text{recursive}(x_t)\right| 
    \leq 
    \left|I^\text{true}_{\Delta}(x_t)\right| 
    + L_V \underbrace{\frac{1 - {L_\mathcal{P}}^{\Delta}}{1 - L_\mathcal{P}} \epsilon_\mathcal{P}}_{\text{compounding errors}}.
    $$
    And for stochastic delays $\delta \sim d_\Delta(\cdot)$, we have
    $$
    \begin{aligned}
        \left|I^\text{recursive}(x_t)\right| 
        &\leq \mathop{\mathbb{E}}_{\delta \sim d_\Delta(\cdot)}\left[
        \left|I^\text{true}_{\delta}(x_t)\right| + L_V \underbrace{\frac{1 - {L_\mathcal{P}}^{\delta}}{1 - L_\mathcal{P}} \epsilon_\mathcal{P}}_{\text{compounding errors}}
        \right]. \\
    \end{aligned}
    $$
\end{theorem}
\begin{proof}

    For deterministic delays $\Delta$, the performance difference $I^\text{recursive}$ can be written as:
    $$
    \begin{aligned}
        I^\text{recursive}(x_t) &= \mathop{\mathbb{E}}_{{s_t}\sim b_\theta(\cdot|{x_t})}\left[{V^{\pi}(s_t)}\right] - {V^{\pi_\Delta}(x_t)}\\
        &= I^\text{true}_{\Delta}({x_t}) + \mathop{\mathbb{E}}_{{s_t}\sim b_\theta(\cdot|{x_t})}\left[{V^{\pi}(s_t)}\right] - \mathop{\mathbb{E}}_{{s_t}\sim b(\cdot|{x_t})}\left[{V^{\pi}(s_t)}\right].\\
    \end{aligned}
    $$
    And recall that we have the Lipschitz continuity of the value function $V^\pi$:
    $$
    \left|\mathop{\mathbb{E}}_{{s_t}\sim b_\theta(\cdot|{x_t})}\left[{V^{\pi}(s_t)}\right] - \mathop{\mathbb{E}}_{{s_t}\sim b(\cdot|{x_t})}\left[{V^{\pi}(s_t)}\right]\right|
    \leq 
    L_V \mathcal{W}(b_\theta(\cdot|{x_t}) || b(\cdot|{x_t})).
    $$

    For $\mathcal{W}(b_\theta(\cdot|{x_t}) || b(\cdot|{x_t}))$, we follow the proof sketch of~\citep{asadi2018lipschitz}.
    
    We use $b^{(i)}(\cdot|{x_t}) (i=1, \ldots, \Delta)$ to note that the belief function with the specific delays $i$. For instance, $b^{(\Delta)}(\cdot|{x_t}) = b(\cdot|{x_t})$ and $b^{(1)}(\cdot|{x_t}) = \mathcal{P}(\cdot|{s_t, a_t})$.

    Then, we have
    $$
    \begin{aligned}
    &\mathcal{W}(b^{(\Delta)}_\theta(\cdot|{x_t}) || b^{(\Delta)}(\cdot|{x_t})) \\
    &= \mathcal{W}(\mathcal{P}_\theta(\cdot|{b^{(\Delta-1)}_\theta(\cdot|{x_t}), a_{t-1}}) || \mathcal{P}(\cdot|{b^{(\Delta-1)}(\cdot|{x_t}), a_{t-1}}))\\
    & \leq \mathcal{W}(\mathcal{P}_\theta(\cdot|{b^{(\Delta-1)}_\theta(\cdot|{x_t}), a_{t-1}}) || \mathcal{P}(\cdot|{b^{(\Delta-1)}_\theta(\cdot|{x_t}), a_{t-1}})) + \mathcal{W}(\mathcal{P}(\cdot|{b^{(\Delta-1)}_\theta(\cdot|{x_t}), a_{t-1}}) || \mathcal{P}(\cdot|{b^{(\Delta-1)}(\cdot|{x_t}), a_{t-1}}))\\
    & \leq \epsilon_\mathcal{P} + L_{\mathcal{P}} \mathcal{W}(b^{\Delta-1}_\theta(\cdot|{x_t}) || b^{\Delta-1}(\cdot|{x_t}))\\
    & \leq (1 + L_{\mathcal{P}}) \epsilon_\mathcal{P} + {L_{\mathcal{P}}}^2 \mathcal{W}(b^{(\Delta-2)}_\theta(\cdot|{x_t}) || b^{(\Delta-2)}(\cdot|{x_t}))\\
    & \leq \cdots\\
    & \leq (1 + \cdots + {L_{\mathcal{P}}}^{\Delta - 2}) \epsilon_\mathcal{P} + {L_{\mathcal{P}}}^{\Delta - 1} \mathcal{W}(b^{(1)}_\theta(\cdot|{x_t}) || b^{(1)}(\cdot|{x_t}))\\
    & \leq ({1 + \cdots + L_\mathcal{P}}^{\Delta - 1}) \epsilon_\mathcal{P} \\
    & = \frac{1 - {L_\mathcal{P}}^{\Delta}}{1 - L_\mathcal{P}} \epsilon_\mathcal{P}.\\
    \end{aligned}
    $$

    Therefore, we have
    $$
    \left|I^\text{recursive}(x_t)\right| \leq \left|I^\text{true}_{\Delta}(x_t)\right| + L_V \frac{1 - {L_\mathcal{P}}^{\Delta}}{1 - L_\mathcal{P}} \epsilon_\mathcal{P}.
    $$

    The above theoretical results can be extended to the stochastic delays $\delta \sim d_\Delta(\cdot)$ easily. 
    The performance difference of the ground-truth belief $I^\text{true}_{\delta}$ is defined as:
    $$
        \begin{aligned}
            I^\text{true}_{\delta}(x_t)
            &= \frac{1}{1-\gamma} \mathop{\mathbb{E}}_{\substack{
            {\hat{s} \sim b_{\delta}(\cdot|\hat{x})}\\
            {\hat{a} \sim\pi_\delta(\cdot|\hat{x})}\\
            {\hat{x} \sim d_{\delta}^{\pi}(\cdot|x_t)}\\
            }}
            \left[{V^{\pi}(\hat{s})} - {Q^{\pi}}({\hat{s}}, \hat{a})\right].\\
        \end{aligned}
    $$
    Finally, we have
    $$
        \begin{aligned}
        \left|I^\text{recursive}(x_t)\right| 
        &\leq \sum_{\delta=1}^{\Delta}d_\Delta(\delta)\left[
        \left|I^\text{true}_{\delta}(x_t)\right| + L_V \frac{1 - {L_\mathcal{P}}^{\delta}}{1 - L_\mathcal{P}} \epsilon_\mathcal{P}
        \right]. \\
        \end{aligned}
    $$
\end{proof}

\begin{proposition}[Performance Degeneration Bound of Directly Forecasting Belief]
\label{appendix:theo_delayed_performance_diff_seq}
For the delay-free policy $\pi$ and the delayed policy $\pi_\Delta$. Given any $x_t \in \mathcal{X}$, the performance degeneration $I^\text{direct}$ of the directly forecasting belief $b_\theta$ can bounded as follows respectively.

For deterministic delays $\Delta$, we have
$$
    \begin{aligned}
    \left|I^\text{direct}(x_t)\right| 
    &\leq \left|I^\text{true}_{\Delta}(x_t)\right| + L_V \epsilon_{direct}. \\
    \end{aligned}
$$
For stochastic delays $\delta \sim d_\Delta(\cdot)$, we have
$$
    \begin{aligned}
    \left|I^\text{direct}(x_t)\right| 
    &\leq \mathop{\mathbb{E}}_{\delta \sim d_\Delta(\cdot)}\left[
    \left|I^\text{true}_{\delta}(x_t)\right|\right] + L_V \epsilon_{direct}. \\
    \end{aligned}
$$    
\end{proposition}
\begin{proof}
Applying \Cref{definition:sequence_modeling_belief_error} and the proof of \Cref{appendix:theo_delayed_performance_diff_single}.
\end{proof}

\begin{proposition}[Performance Degeneration Comparison]
\label{appendix:proposition_performance_degeneration_comparison}
    Directly forecasting belief could achieve a better performance guarantee 
    $\left|I^\text{direct}(x_t)\right| \leq \left|I^\text{recursive}(x_t)\right|$, 
    if we have
    $$
    \epsilon_{direct} \leq \frac{1 - {L_\mathcal{P}}^{\Delta}}{1 - L_\mathcal{P}} \epsilon_\mathcal{P}
    $$
    for deterministic delays $\Delta$, and
    $$
    \epsilon_{direct}
    \leq 
    \mathop{\mathbb{E}}_{\delta \sim d_\Delta(\cdot)} \left[\frac{1 - {L_\mathcal{P}}^{\delta}}{1 - L_\mathcal{P}}\right] \epsilon_\mathcal{P}
    $$
    for stochastic delays $\delta \sim d_\Delta(\cdot)$.
\end{proposition}
\begin{proof}
For deterministic delays $\Delta$, if we have 
$$
\epsilon_{direct} 
\leq 
\frac{1 - {L_\mathcal{P}}^{\Delta}}{1 - L_\mathcal{P}} \epsilon_\mathcal{P},
$$
then it is obvious that we have
$$
\underbrace{\left|I^\text{true}_{\Delta}(x_t)\right| + L_V \epsilon_{direct} }_{\left|I^\text{direct}(x_t)\right|}
\leq 
\underbrace{\left|I^\text{true}_{\Delta}(x_t)\right| + L_V \frac{1 - {L_\mathcal{P}}^{\Delta}}{1 - L_\mathcal{P}} \epsilon_\mathcal{P}}_{\left|I^\text{recursive}(x_t)\right|}.
$$

For stochastic delays $\delta \sim d_\Delta(\cdot)$, if we have 
$$
    \epsilon_{direct} 
    \leq 
    \mathop{\mathbb{E}}_{\delta \sim d_\Delta(\cdot)} \left[\frac{1 - {L_\mathcal{P}}^{\delta}}{1 - L_\mathcal{P}} \epsilon_\mathcal{P}\right],
$$
then it is obvious that we have
$$
\underbrace{\left|I^\text{true}_{\Delta}(x_t)\right| + L_V \epsilon_{direct}}_{\left|I^\text{direct}(x_t)\right|}
\leq 
\underbrace{\left|I^\text{true}_{\Delta}(x_t)\right| + L_V \mathop{\mathbb{E}}_{\delta \sim d_\Delta(\cdot)} \left[\frac{1 - {L_\mathcal{P}}^{\delta}}{1 - L_\mathcal{P}} \epsilon_\mathcal{P}\right]}_{\left|I^\text{recursive}(x_t)\right|}.
$$
\end{proof}

\clearpage
\section{Additional Results on MuJoCo.}
\label{appendix:sec:additional_results}
We report additional experimental results on MuJoCo, including more different patterns of deterministic and stochastic delays in \Cref{appendix:table:deterministic_delays} and \Cref{appendix:table:stochastic_delays}, respectively.

\begin{table*}[h]
\centering
\caption{Performance on MuJoCo with Deterministic Delays. The best performance is underlined, the best belief-based method is in \textcolor{red}{red}.\label{appendix:table:deterministic_delays}}
\scalebox{0.95}{
\begin{tabular}{c|c|ccc|cccc}
\hline
                                 &                                                                             & \multicolumn{3}{c|}{Augmentation-based}                                                                    & \multicolumn{4}{c}{Belief-based}                                                                           \\
\multirow{-2}{*}{Task}           & \multirow{-2}{*}{\begin{tabular}[c]{@{}c@{}}Delays\end{tabular}} & A-SAC             & BPQL                                     & ADRL                                     & DATS              & D-Dreamer              & D-SAC             & DFBT-SAC (ours)                                    \\ \hline
                                 & $8$                                                                           & $0.10_{\pm 0.01}$ & { $0.40_{\pm 0.04}$} & \underline{$0.44_{\pm 0.03}$}                        & $0.08_{\pm 0.01}$ & $0.08_{\pm 0.01}$ & $0.12_{\pm 0.06}$ & \textcolor{red}{$0.35_{\pm 0.12}$}                        \\
                                 & $16$                                                                          & $0.06_{\pm 0.03}$ & \underline{$0.42_{\pm 0.02}$} & $0.29_{\pm 0.10}$                        & $0.08_{\pm 0.01}$ & $0.08_{\pm 0.01}$ & $0.11_{\pm 0.10}$ & \textcolor{red}{$0.40_{\pm 0.05}$}                        \\
                                 & $32$                                                                          & $0.02_{\pm 0.02}$ & $0.40_{\pm 0.03}$                        & $0.26_{\pm 0.04}$                        & $0.11_{\pm 0.04}$ & $0.08_{\pm 0.00}$ & $0.08_{\pm 0.02}$ & \underline{\textcolor{red}{$0.42_{\pm 0.03}$}} \\
                                 & $64$                                                                          & $0.01_{\pm 0.01}$ & $0.15_{\pm 0.12}$                        & $0.16_{\pm 0.02}$                        & $0.12_{\pm 0.05}$ & $0.11_{\pm 0.05}$ & $0.12_{\pm 0.07}$ & \underline{\textcolor{red}{$0.39_{\pm 0.06}$}} \\
\multirow{-5}{*}{HalfCheetah-v2} & $128$                                                                         & $0.04_{\pm 0.06}$ & $0.08_{\pm 0.13}$                        & $0.14_{\pm 0.02}$                        & $0.10_{\pm 0.08}$ & $0.15_{\pm 0.05}$ & $0.09_{\pm 0.04}$ & \underline{\textcolor{red}{$0.41_{\pm 0.03}$}} \\ \hline
                                 & $8$                                                                           & $0.61_{\pm 0.31}$ & $0.87_{\pm 0.09}$                        & \underline{$0.95_{\pm 0.16}$} & $0.41_{\pm 0.31}$ & $0.11_{\pm 0.01}$ & $0.16_{\pm 0.05}$ & \textcolor{red}{$0.77_{\pm 0.18}$}                        \\
                                 & $16$                                                                          & $0.17_{\pm 0.06}$ & $0.92_{\pm 0.16}$                        & \underline{$0.94_{\pm 0.17}$} & $0.24_{\pm 0.31}$ & $0.19_{\pm 0.13}$ & $0.11_{\pm 0.01}$ & \textcolor{red}{$0.89_{\pm 0.13}$}                        \\
                                 & $32$                                                                          & $0.11_{\pm 0.02}$ & \underline{$0.89_{\pm 0.14}$} & $0.73_{\pm 0.20}$                        & $0.07_{\pm 0.04}$ & $0.11_{\pm 0.05}$ & $0.11_{\pm 0.01}$ & \textcolor{red}{$0.68_{\pm 0.20}$}                        \\
                                 & $64$                                                                          & $0.05_{\pm 0.00}$ & \underline{$0.23_{\pm 0.30}$} & $0.11_{\pm 0.03}$                        & $0.13_{\pm 0.00}$ & $0.09_{\pm 0.05}$ & $0.08_{\pm 0.02}$ & \textcolor{red}{$0.19_{\pm 0.02}$}                        \\
\multirow{-5}{*}{Hopper-v2}      & $128$                                                                         & $0.04_{\pm 0.01}$ & $0.08_{\pm 0.02}$                        & $0.07_{\pm 0.01}$                        & $0.08_{\pm 0.01}$ & $0.09_{\pm 0.03}$ & $0.06_{\pm 0.01}$ & \underline{\textcolor{red}{$0.20_{\pm 0.03}$}} \\ \hline
                                 & $8$                                                                           & $0.44_{\pm 0.26}$ & \underline{$1.07_{\pm 0.02}$} & $0.97_{\pm 0.10}$                        & $0.13_{\pm 0.05}$ & $0.11_{\pm 0.06}$ & $0.09_{\pm 0.05}$ & \textcolor{red}{$0.99_{\pm 0.03}$}                        \\
                                 & $16$                                                                          & $0.13_{\pm 0.02}$ & \underline{$0.96_{\pm 0.05}$} & $0.67_{\pm 0.21}$                        & $0.06_{\pm 0.10}$ & $0.12_{\pm 0.03}$ & $0.08_{\pm 0.04}$ & \textcolor{red}{$0.95_{\pm 0.11}$}                        \\
                                 & $32$                                                                          & $0.10_{\pm 0.02}$ & $0.37_{\pm 0.25}$                        & $0.16_{\pm 0.08}$                        & $0.02_{\pm 0.03}$ & $0.08_{\pm 0.05}$ & $0.08_{\pm 0.02}$ & \underline{\textcolor{red}{$0.64_{\pm 0.10}$}} \\
                                 & $64$                                                                          & $0.07_{\pm 0.01}$ & $0.14_{\pm 0.03}$                        & $0.10_{\pm 0.01}$                        & $0.01_{\pm 0.02}$ & $0.08_{\pm 0.03}$ & $0.08_{\pm 0.04}$ & \underline{\textcolor{red}{$0.41_{\pm 0.14}$}} \\
\multirow{-5}{*}{Walker2d-v2}    & $128$                                                                         & $0.06_{\pm 0.00}$ & $0.07_{\pm 0.03}$                        & $0.08_{\pm 0.01}$                        & $0.02_{\pm 0.02}$ & $0.08_{\pm 0.05}$ & $0.11_{\pm 0.06}$ & \underline{\textcolor{red}{$0.40_{\pm 0.08}$}} \\ \hline
\end{tabular}
}
\end{table*}

\begin{table*}[h]
\centering
\caption{Performance on MuJoCo with Stochastic Delays. The best performance is underlined, and the best belief-based method is in \textcolor{red}{red}.\label{appendix:table:stochastic_delays}}
\scalebox{0.95}{
\begin{tabular}{c|c|ccc|cccc}
\hline
                                 &                                                                             & \multicolumn{3}{c|}{Augmentation-based}                                                                    & \multicolumn{4}{c}{Belief-based}                                                                           \\
\multirow{-2}{*}{Task}           & \multirow{-2}{*}{\begin{tabular}[c]{@{}c@{}}Delays\end{tabular}} & A-SAC             & BPQL                                     & ADRL                                     & DATS              & D-Dreamer              & D-SAC             & DFBT-SAC (ours)                                    \\ \hline
                                 & $U(1,8)$                                                                       & $0.09_{\pm 0.01}$ & $0.21_{\pm 0.07}$                        & \multicolumn{1}{c|}{$0.17_{\pm 0.07}$} & $0.09_{\pm 0.03}$ & $0.02_{\pm 0.01}$ & $0.03_{\pm 0.01}$ & \underline{ \textcolor{red}{$0.37_{\pm 0.12}$}} \\
                                 & $U(1,16)$                                                                      & $0.04_{\pm 0.04}$ & $0.31_{\pm 0.08}$                        & \multicolumn{1}{c|}{$0.24_{\pm 0.04}$} & $0.13_{\pm 0.03}$ & $0.03_{\pm 0.02}$ & $0.01_{\pm 0.01}$ & \underline{ \textcolor{red}{$0.37_{\pm 0.06}$}} \\
                                 & $U(1,32)$                                                                      & $0.01_{\pm 0.00}$ & \underline{ $0.33_{\pm 0.07}$} & \multicolumn{1}{c|}{$0.23_{\pm 0.02}$} & $0.11_{\pm 0.04}$ & $0.02_{\pm 0.00}$ & $0.01_{\pm 0.01}$ & \textcolor{red}{$0.31_{\pm 0.16}$}                        \\
                                 & $U(1,64)$                                                                      & $0.06_{\pm 0.11}$ & $0.23_{\pm 0.06}$                        & \multicolumn{1}{c|}{$0.17_{\pm 0.02}$} & $0.16_{\pm 0.03}$ & $0.04_{\pm 0.03}$ & $0.01_{\pm 0.00}$ & \underline{ \textcolor{red}{$0.40_{\pm 0.06}$}} \\
\multirow{-5}{*}{HalfCheetah-v2} & $U(1,128)$                                                                     & $0.01_{\pm 0.01}$ & $0.03_{\pm 0.03}$                        & \multicolumn{1}{c|}{$0.15_{\pm 0.02}$} & $0.16_{\pm 0.03}$ & $0.16_{\pm 0.00}$ & $0.02_{\pm 0.00}$ & \underline{ \textcolor{red}{$0.39_{\pm 0.04}$}} \\ \hline
                                 & $U(1,8)$                                                                       & $0.17_{\pm 0.05}$ & $0.20_{\pm 0.04}$                        & \multicolumn{1}{c|}{$0.18_{\pm 0.04}$} & $0.04_{\pm 0.01}$ & $0.07_{\pm 0.05}$ & $0.14_{\pm 0.04}$ & \underline{ \textcolor{red}{$0.86_{\pm 0.18}$}} \\
                                 & $U(1,16)$                                                                      & $0.08_{\pm 0.02}$ & $0.11_{\pm 0.11}$                        & \multicolumn{1}{c|}{$0.07_{\pm 0.04}$} & $0.04_{\pm 0.01}$ & $0.03_{\pm 0.01}$ & $0.04_{\pm 0.02}$ & \underline{ \textcolor{red}{$0.89_{\pm 0.17}$}} \\
                                 & $U(1,32)$                                                                      & $0.05_{\pm 0.01}$ & $0.07_{\pm 0.09}$                        & \multicolumn{1}{c|}{$0.05_{\pm 0.01}$} & $0.05_{\pm 0.01}$ & $0.04_{\pm 0.01}$ & $0.03_{\pm 0.01}$ & \underline{ \textcolor{red}{$0.43_{\pm 0.21}$}} \\
                                 & $U(1,64)$                                                                      & $0.03_{\pm 0.01}$ & $0.03_{\pm 0.01}$                        & \multicolumn{1}{c|}{$0.03_{\pm 0.01}$} & $0.05_{\pm 0.01}$ & $0.03_{\pm 0.01}$ & $0.03_{\pm 0.01}$ & \underline{ \textcolor{red}{$0.17_{\pm 0.05}$}} \\
\multirow{-5}{*}{Hopper-v2}      & $U(1,128)$                                                                     & $0.03_{\pm 0.01}$ & $0.04_{\pm 0.01}$                        & \multicolumn{1}{c|}{$0.04_{\pm 0.02}$} & $0.05_{\pm 0.00}$ & $0.03_{\pm 0.01}$ & $0.03_{\pm 0.00}$ & \underline{ \textcolor{red}{$0.14_{\pm 0.01}$}} \\ \hline
                                 & $U(1,8)$                                                                       & $0.36_{\pm 0.24}$ & $0.40_{\pm 0.32}$                        & \multicolumn{1}{c|}{$0.41_{\pm 0.15}$} & $0.07_{\pm 0.01}$ & $0.07_{\pm 0.05}$ & $0.12_{\pm 0.04}$ & \underline{ \textcolor{red}{$1.11_{\pm 0.10}$}} \\
                                 & $U(1,16)$                                                                      & $0.19_{\pm 0.10}$ & $0.27_{\pm 0.17}$                        & \multicolumn{1}{c|}{$0.24_{\pm 0.10}$} & $0.08_{\pm 0.02}$ & $0.13_{\pm 0.08}$ & $0.07_{\pm 0.02}$ & \underline{ \textcolor{red}{$0.99_{\pm 0.06}$}} \\
                                 & $U(1,32)$                                                                      & $0.12_{\pm 0.03}$ & $0.16_{\pm 0.04}$                        & \multicolumn{1}{c|}{$0.11_{\pm 0.05}$} & $0.09_{\pm 0.04}$ & $0.12_{\pm 0.04}$ & $0.05_{\pm 0.02}$ & \underline{ \textcolor{red}{$0.67_{\pm 0.15}$}} \\
                                 & $U(1,64)$                                                                      & $0.08_{\pm 0.02}$ & $0.09_{\pm 0.08}$                        & \multicolumn{1}{c|}{$0.06_{\pm 0.01}$} & $0.08_{\pm 0.04}$ & $0.15_{\pm 0.05}$ & $0.06_{\pm 0.03}$ & \underline{ \textcolor{red}{$0.41_{\pm 0.10}$}} \\
\multirow{-5}{*}{Walker2d-v2}    & $U(1,128)$                                                                     & $0.06_{\pm 0.01}$ & $0.06_{\pm 0.06}$                        & $0.04_{\pm 0.02}$                      & $0.10_{\pm 0.04}$ & $0.15_{\pm 0.07}$ & $0.03_{\pm 0.04}$ & \underline{ \textcolor{red}{$0.30_{\pm 0.13}$}} \\ \hline
\end{tabular}
}
\end{table*}

\clearpage
\section{Learning Curves on MuJoCo.}
\label{appendix:sec:learning_curves}
We report learning curves on MuJoCo with different patterns of deterministic and stochastic delays in \Cref{appendix:fig_deterministic} and \Cref{appendix:fig_stochastic}, respectively.

\begin{figure}[h]
    \centering
    \centerline{
        \subfigure[HalfCheetah-v2 (8 Delays)]{\includegraphics[width=0.33\linewidth]{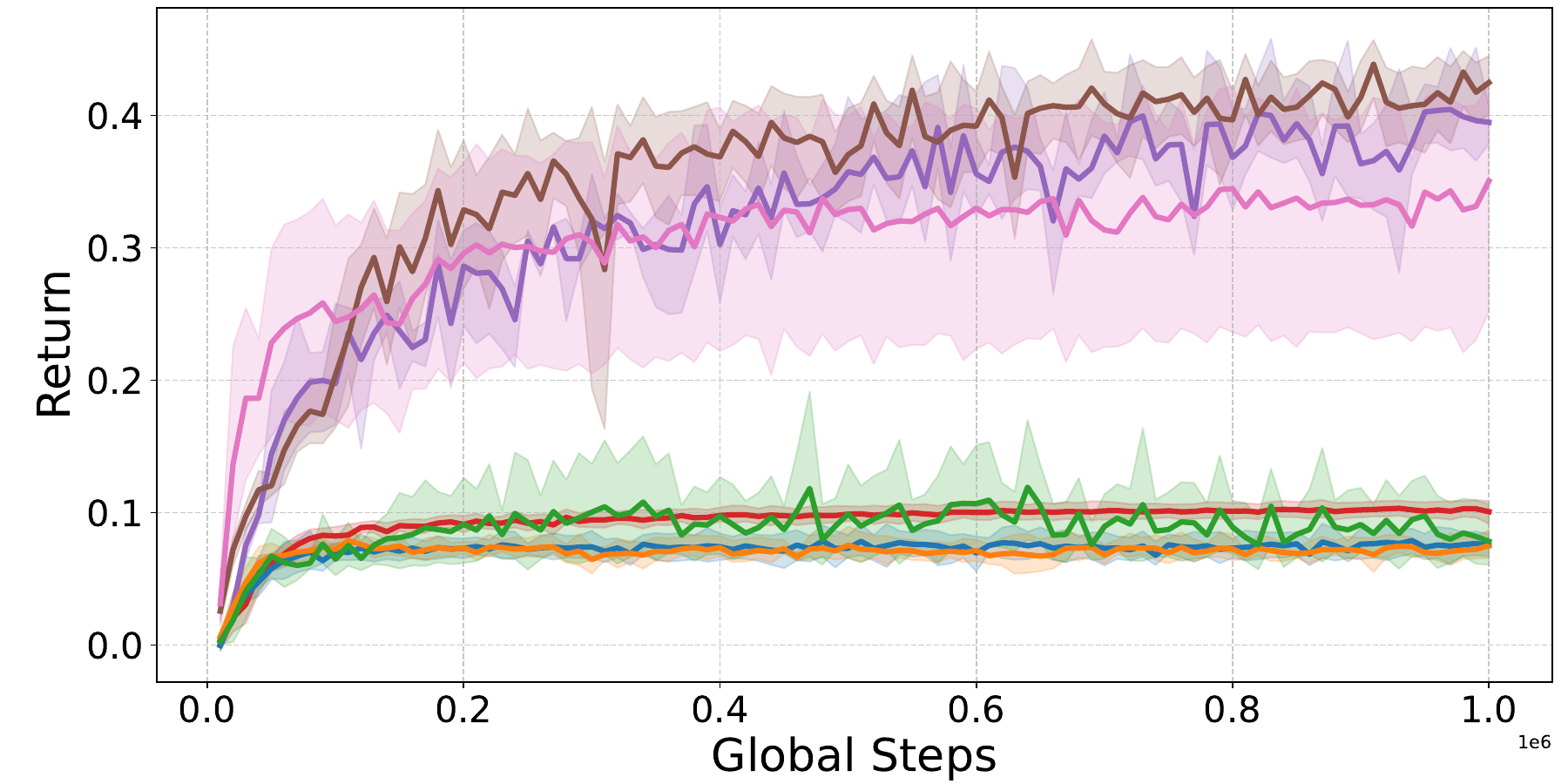}}
        \subfigure[Hopper-v2 (8 Delays)]{\includegraphics[width=0.33\linewidth]{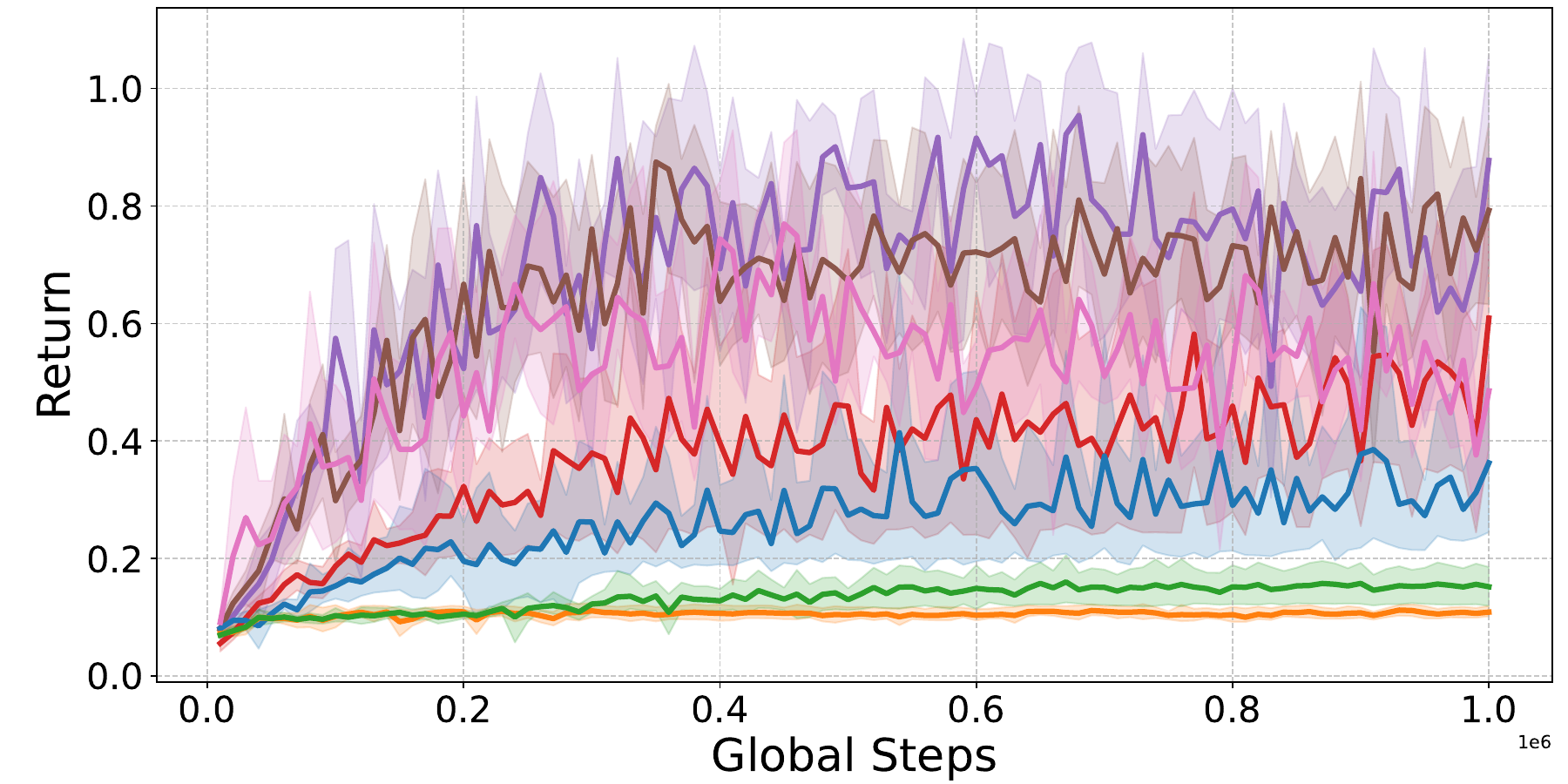}}
        \subfigure[Walker2d-v2 (8 Delays)]{\includegraphics[width=0.33\linewidth]{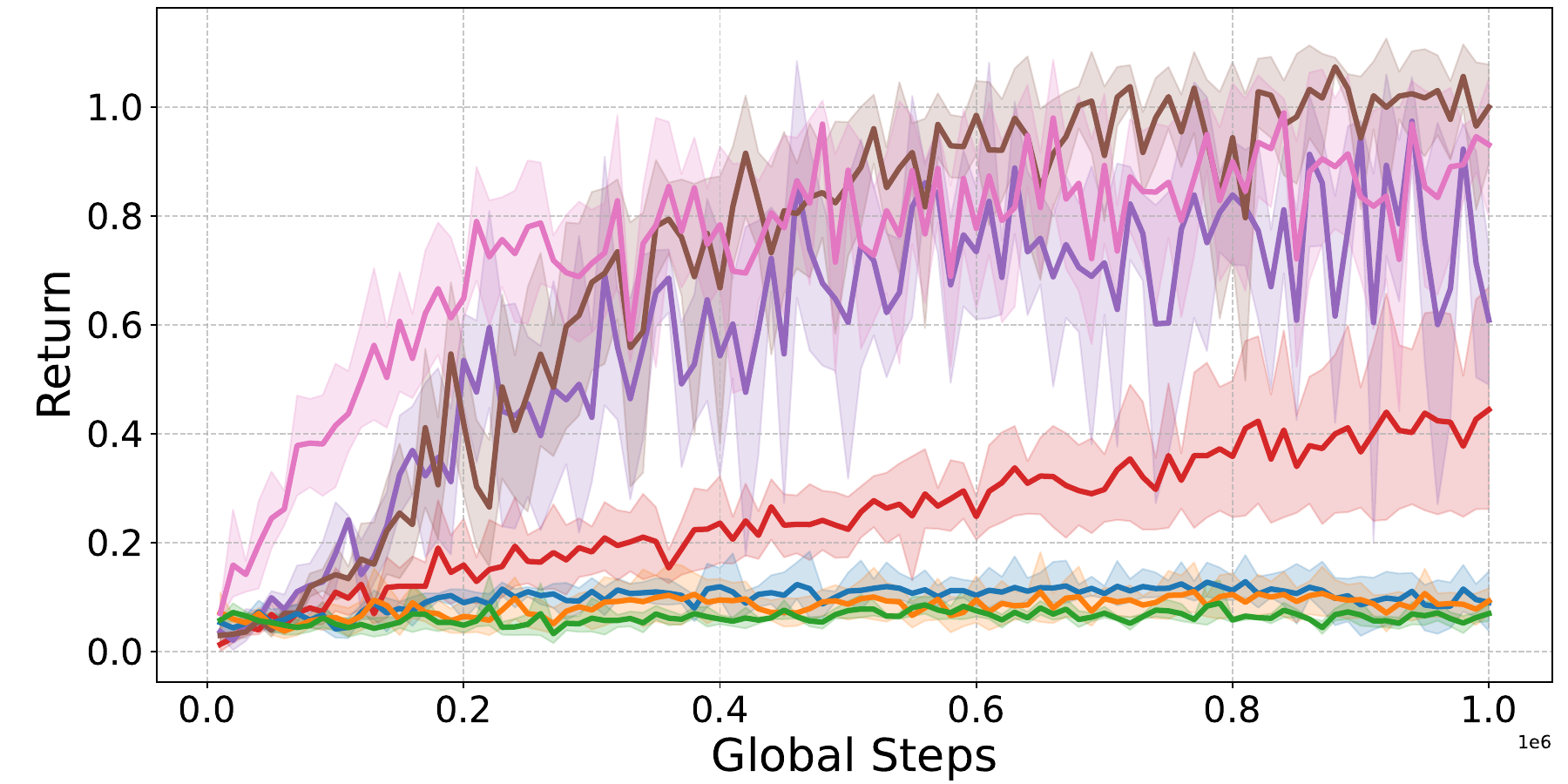}}
    }
    \centerline{
        \subfigure[HalfCheetah-v2 (16 Delays)]{\includegraphics[width=0.33\linewidth]{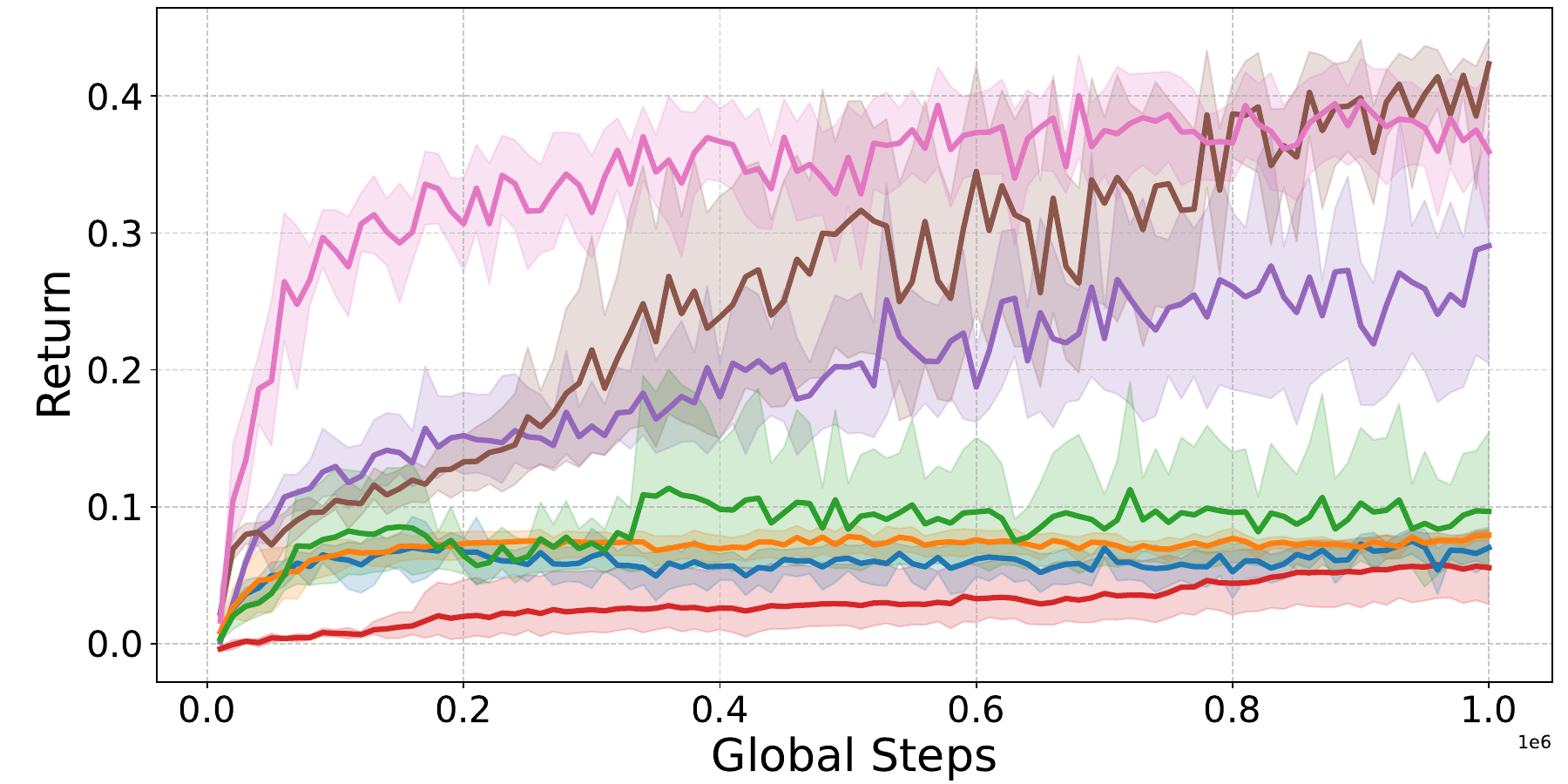}}
        \subfigure[Hopper-v2 (16 Delays)]{\includegraphics[width=0.33\linewidth]{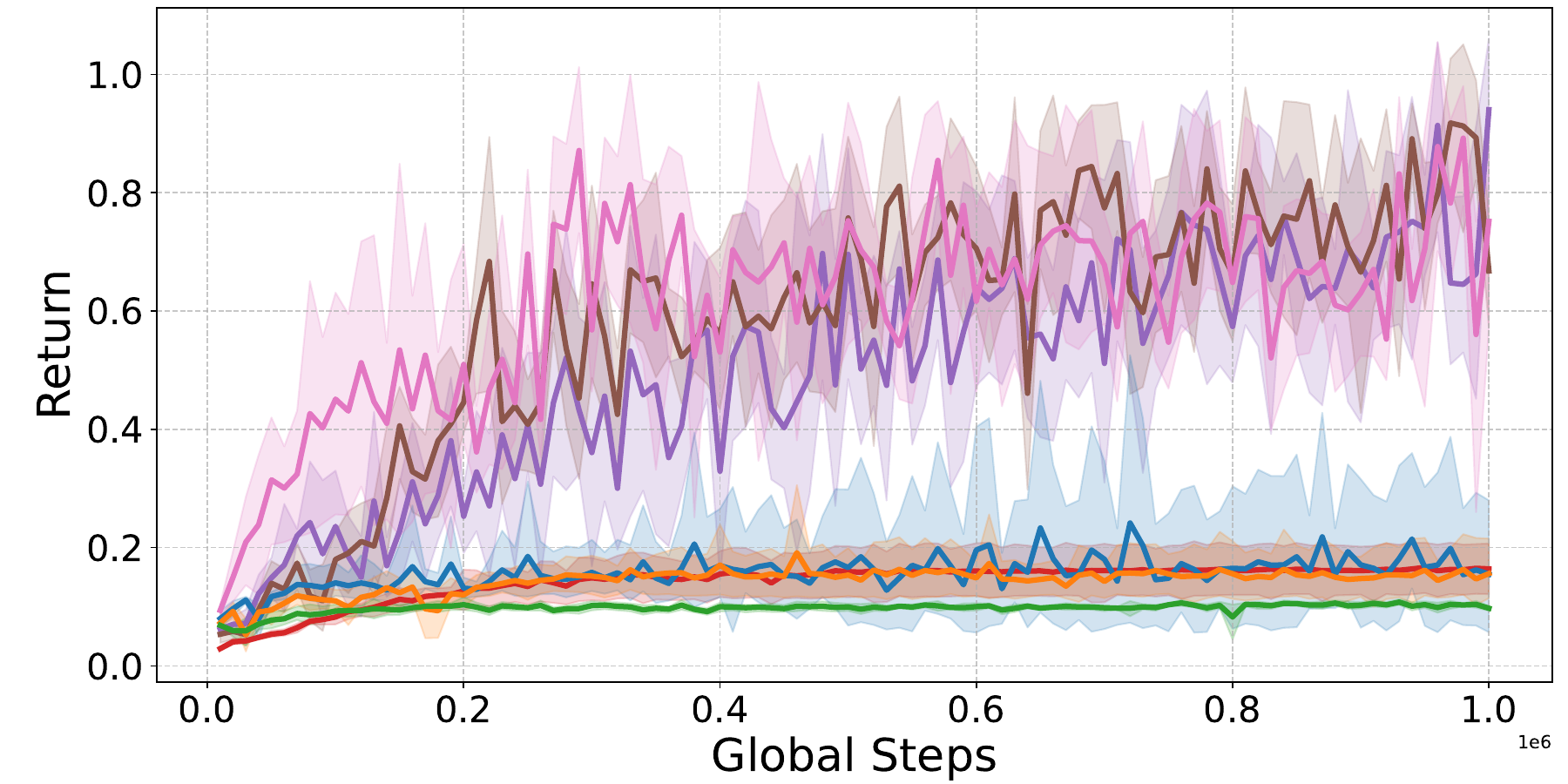}}
        \subfigure[Walker2d-v2 (16 Delays)]{\includegraphics[width=0.33\linewidth]{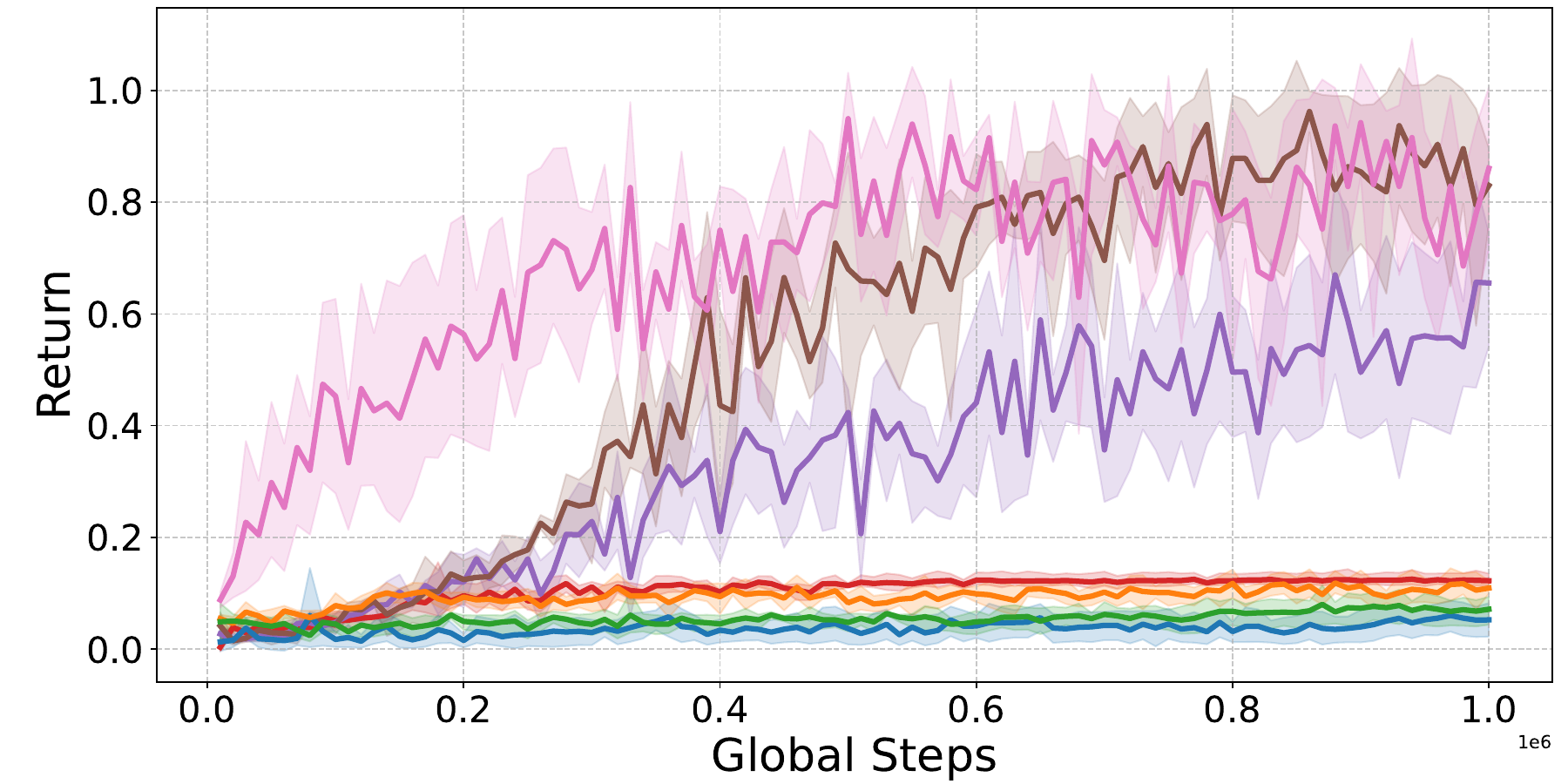}}
    }
    \centerline{
        \subfigure[HalfCheetah-v2 (32 Delays)]{\includegraphics[width=0.33\linewidth]{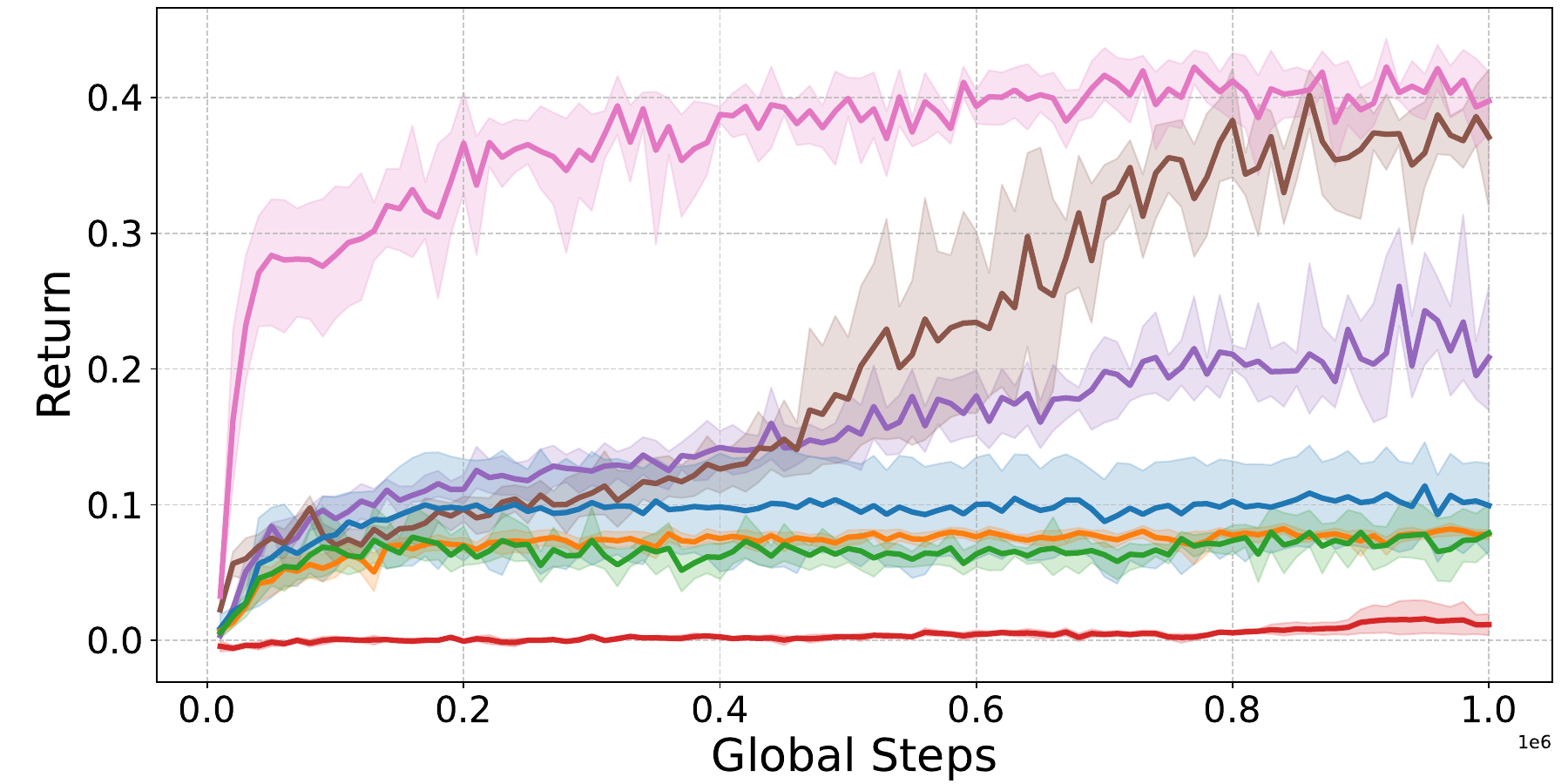}}
        \subfigure[Hopper-v2 (32 Delays)]{\includegraphics[width=0.33\linewidth]{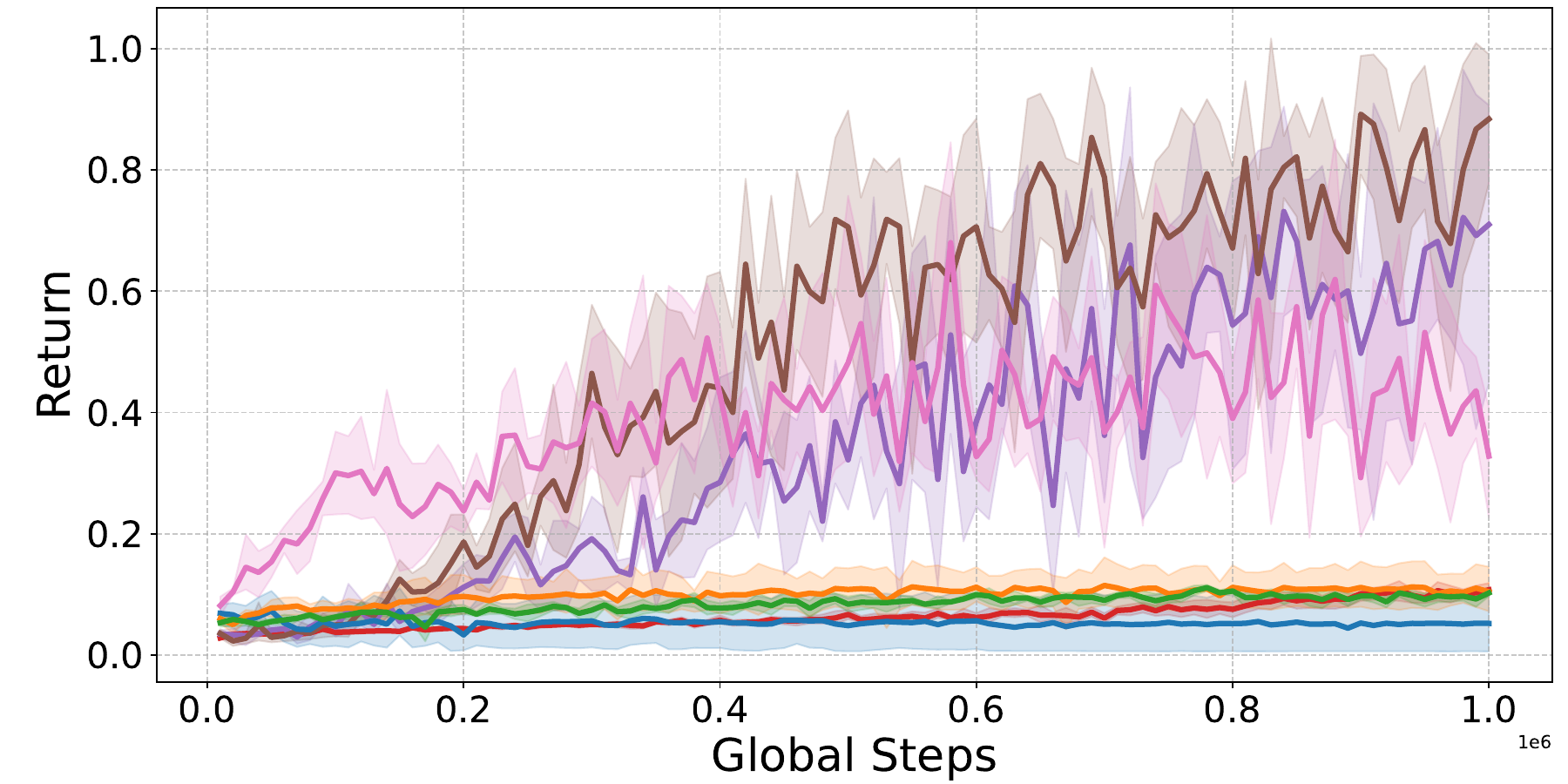}}
        \subfigure[Walker2d-v2 (32 Delays)]{\includegraphics[width=0.33\linewidth]{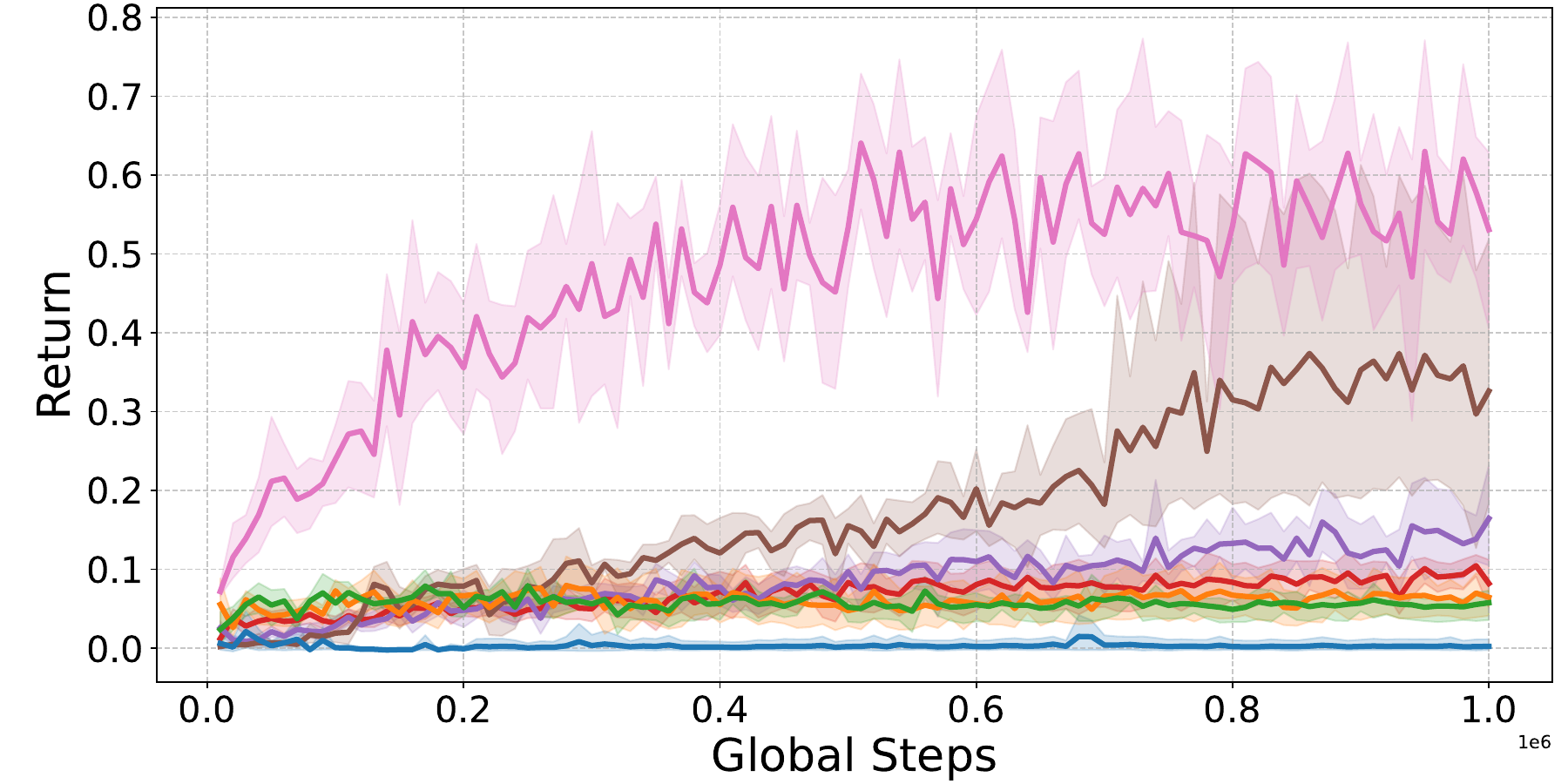}}
    }
    \centerline{
        \subfigure[HalfCheetah-v2 (64 Delays)]{\includegraphics[width=0.33\linewidth]{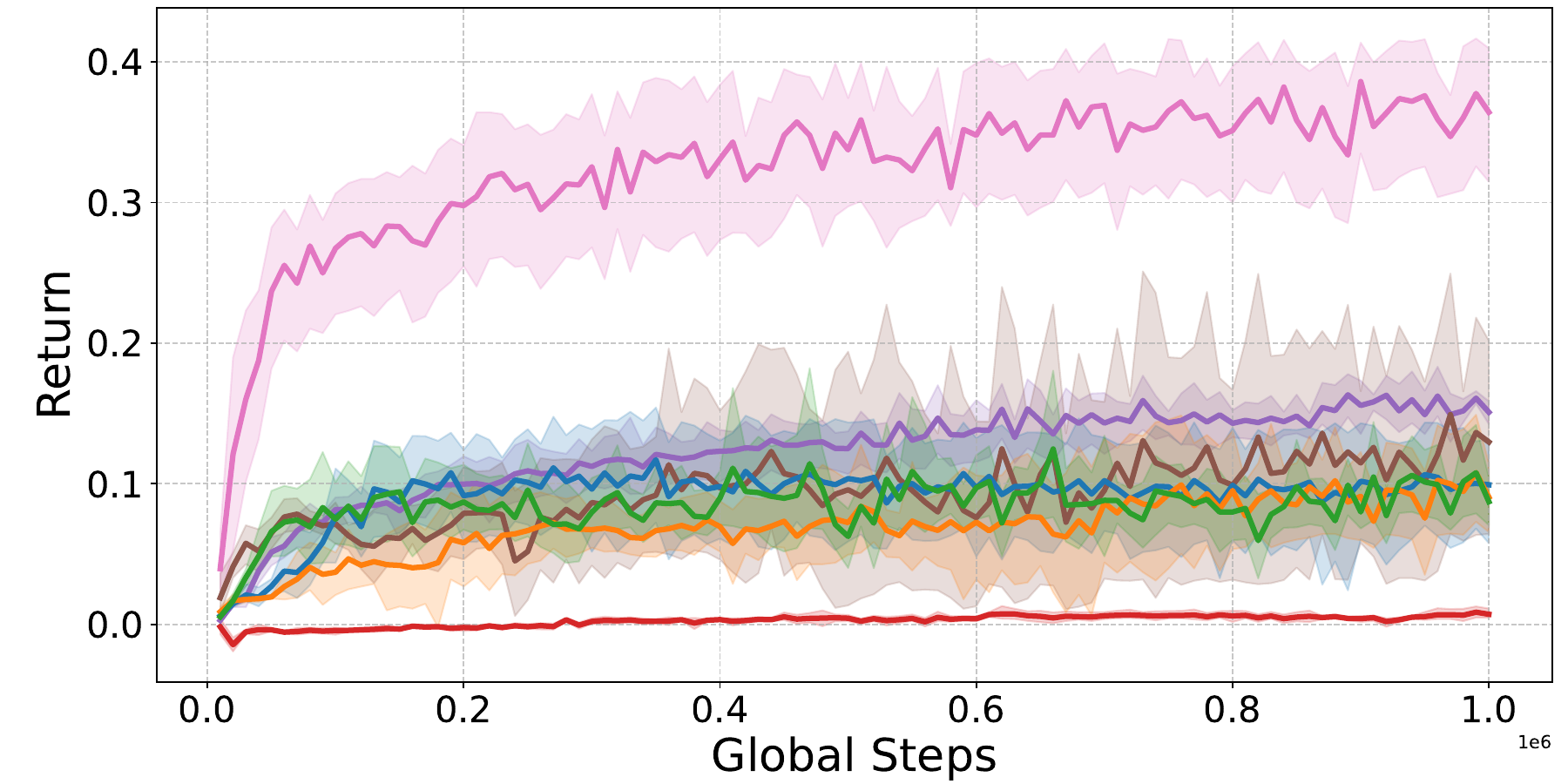}}
        \subfigure[Hopper-v2 (64 Delays)]{\includegraphics[width=0.33\linewidth]{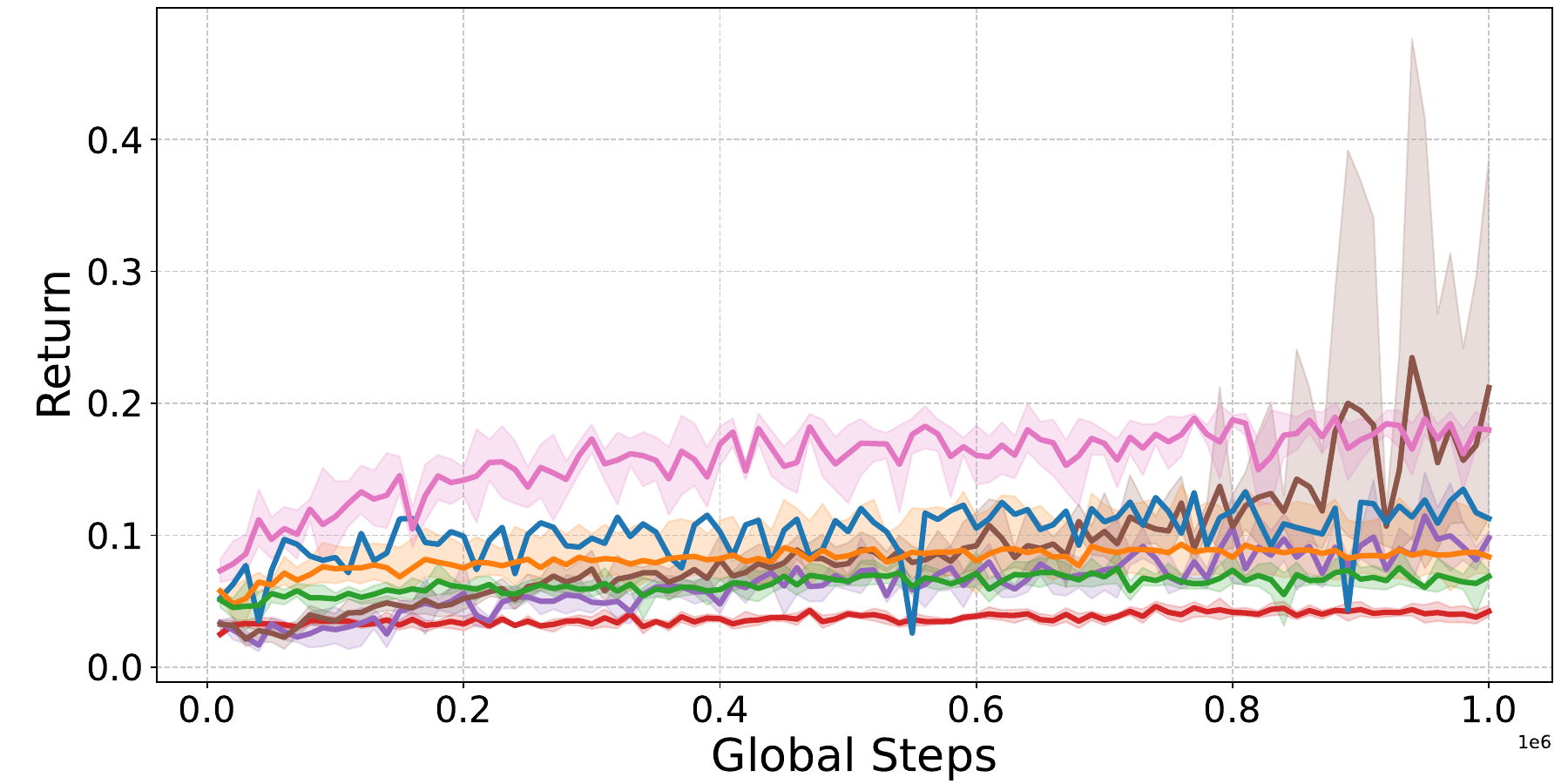}}
        \subfigure[Walker2d-v2 (64 Delays)]{\includegraphics[width=0.33\linewidth]{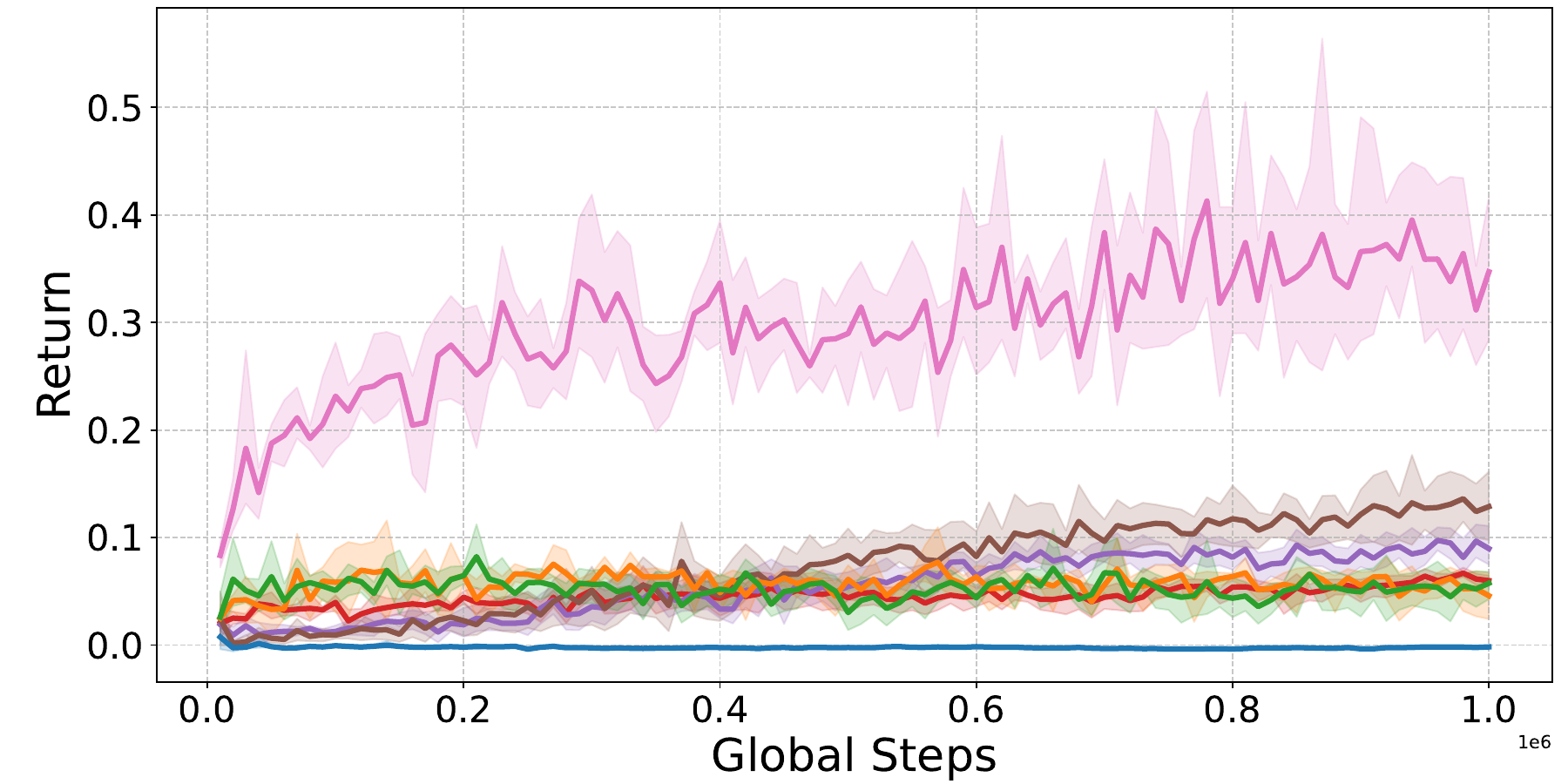}}
    }
    \centerline{
        \subfigure[HalfCheetah-v2 (128 Delays)]{\includegraphics[width=0.33\linewidth]{figs/deterministic_results/halfcheetah_D128.pdf}}
        \subfigure[Hopper-v2 (128 Delays)]{\includegraphics[width=0.33\linewidth]{figs/deterministic_results/hopper_D128.pdf}}
        \subfigure[Walker2d-v2 (128 Delays)]{\includegraphics[width=0.33\linewidth]{figs/deterministic_results/walker2d_D128.pdf}}
    }
    \includegraphics[width=1.\linewidth]{figs/deterministic_results/legend.pdf}
    \caption{Learning Curves on MuJoCo with Deterministic Delays.\label{appendix:fig_deterministic}}
\end{figure}

\begin{figure}[h]
    \centering
    \centerline{
        \subfigure[HalfCheetah-v2 ($U(1, 8)$ Delays)]{\includegraphics[width=0.33\linewidth]{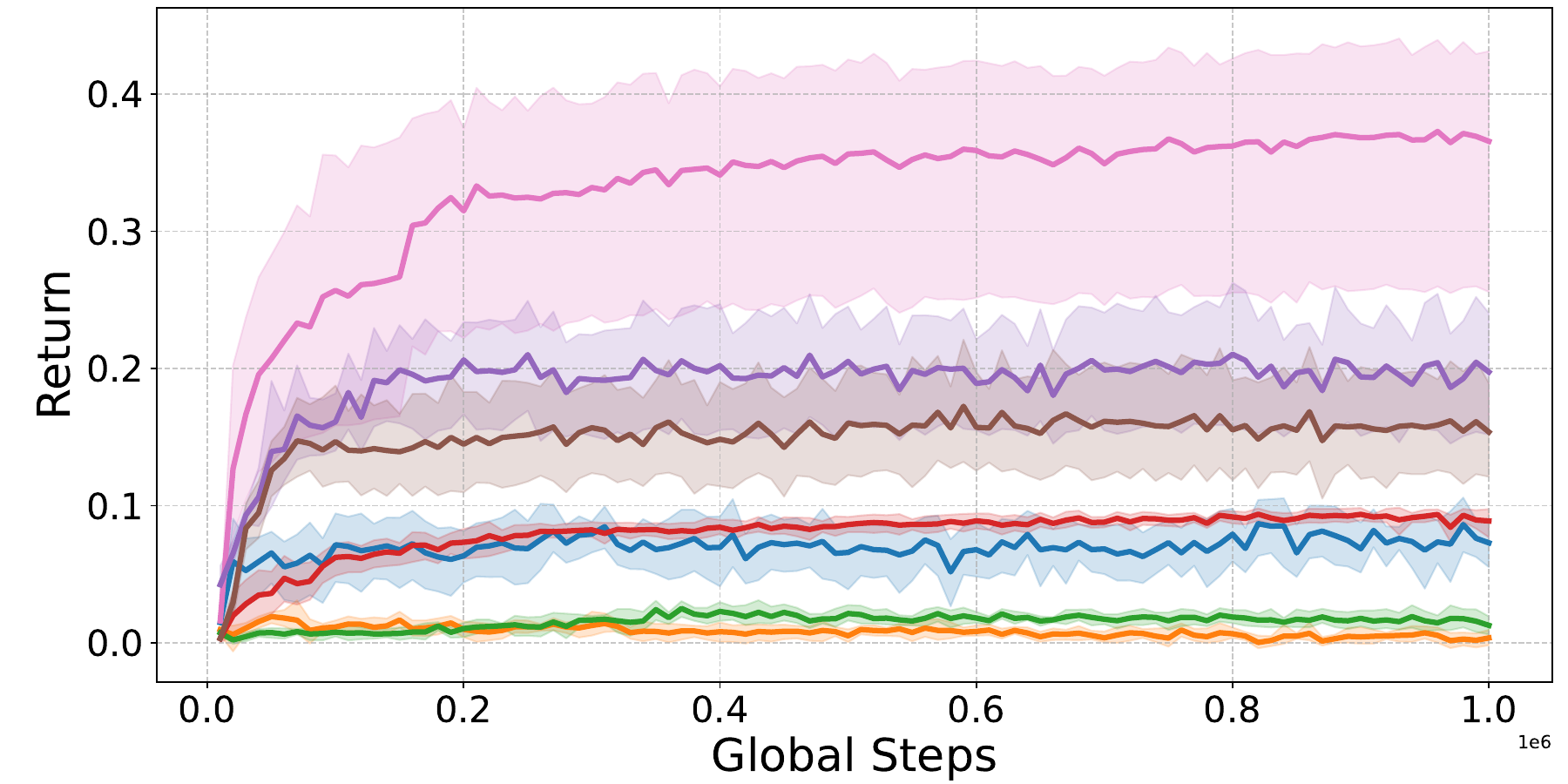}}
        \subfigure[Hopper-v2 ($U(1, 8)$ Delays)]{\includegraphics[width=0.33\linewidth]{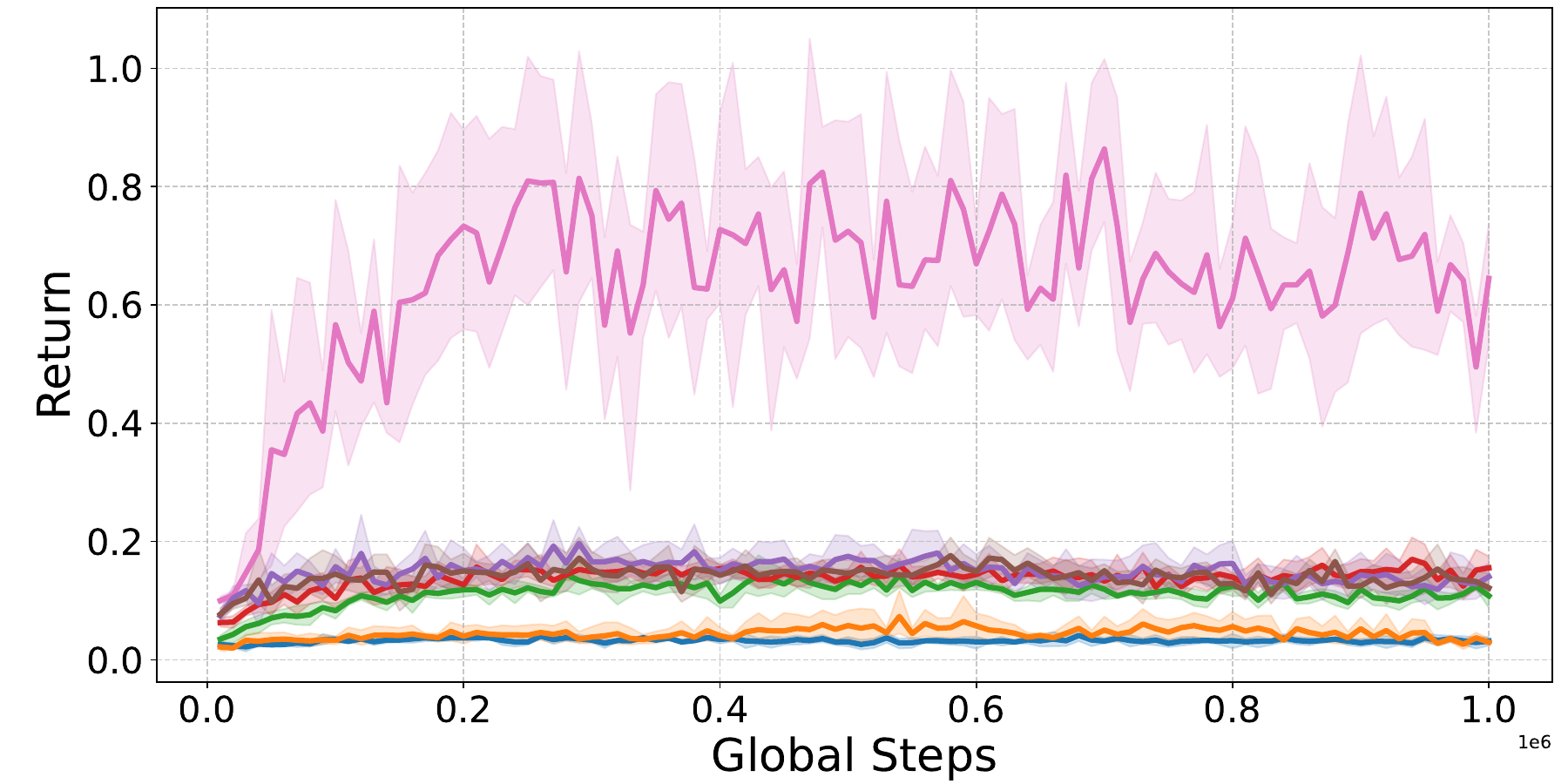}}
        \subfigure[Walker2d-v2 ($U(1, 8)$ Delays)]{\includegraphics[width=0.33\linewidth]{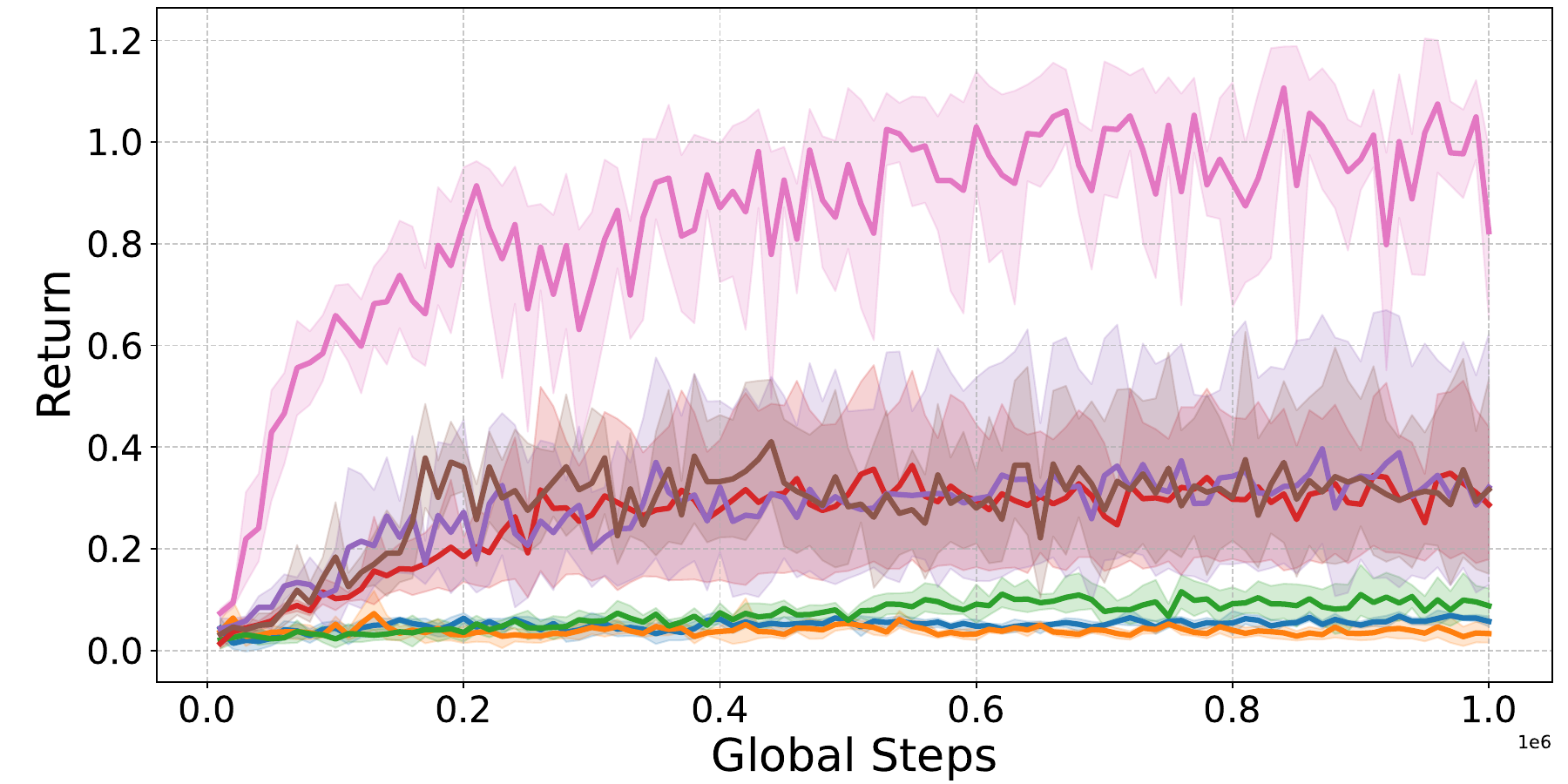}}
    }
    \centerline{
        \subfigure[HalfCheetah-v2 ($U(1, 16)$ Delays)]{\includegraphics[width=0.33\linewidth]{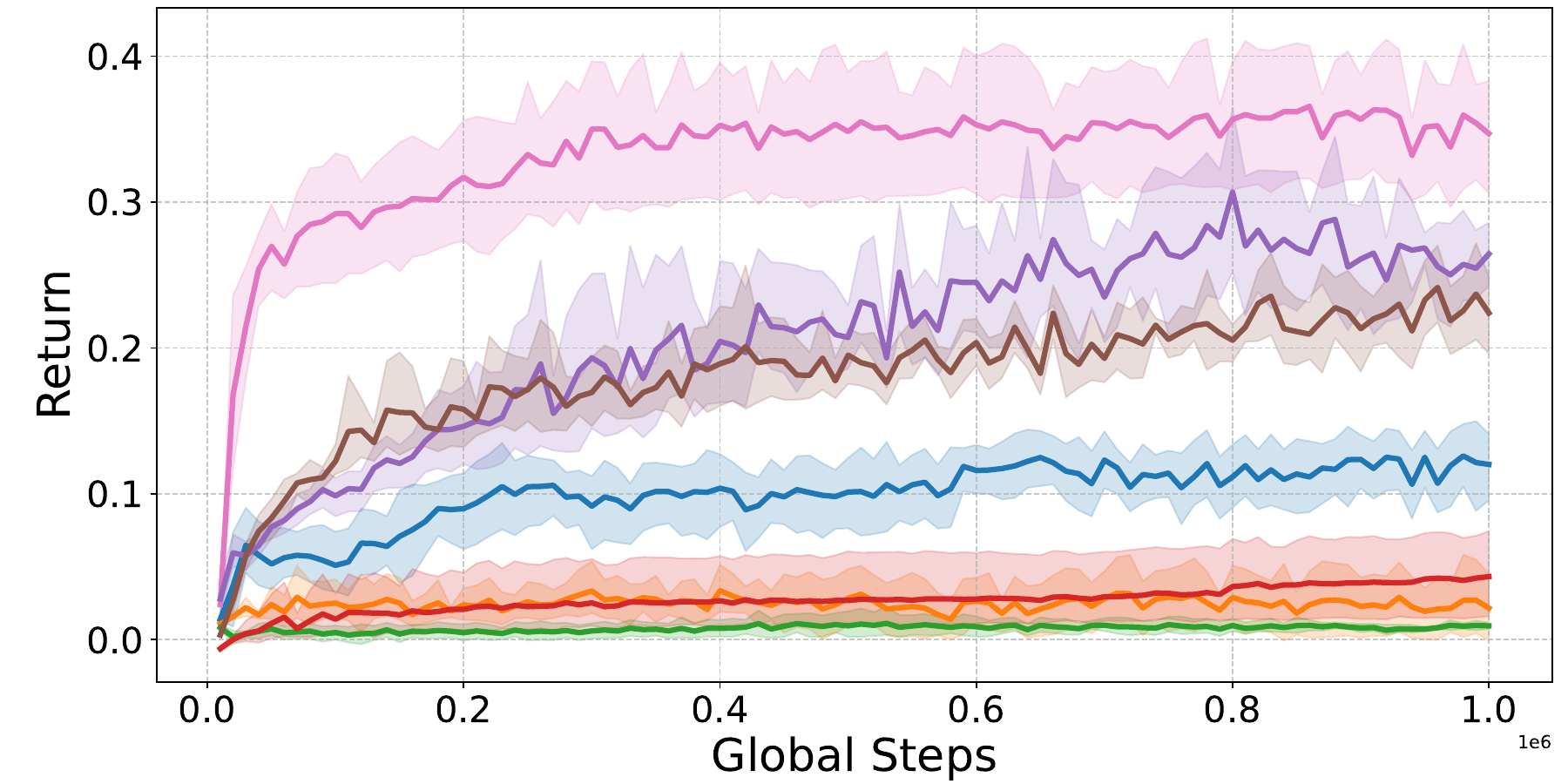}}
        \subfigure[Hopper-v2 ($U(1, 16)$ Delays)]{\includegraphics[width=0.33\linewidth]{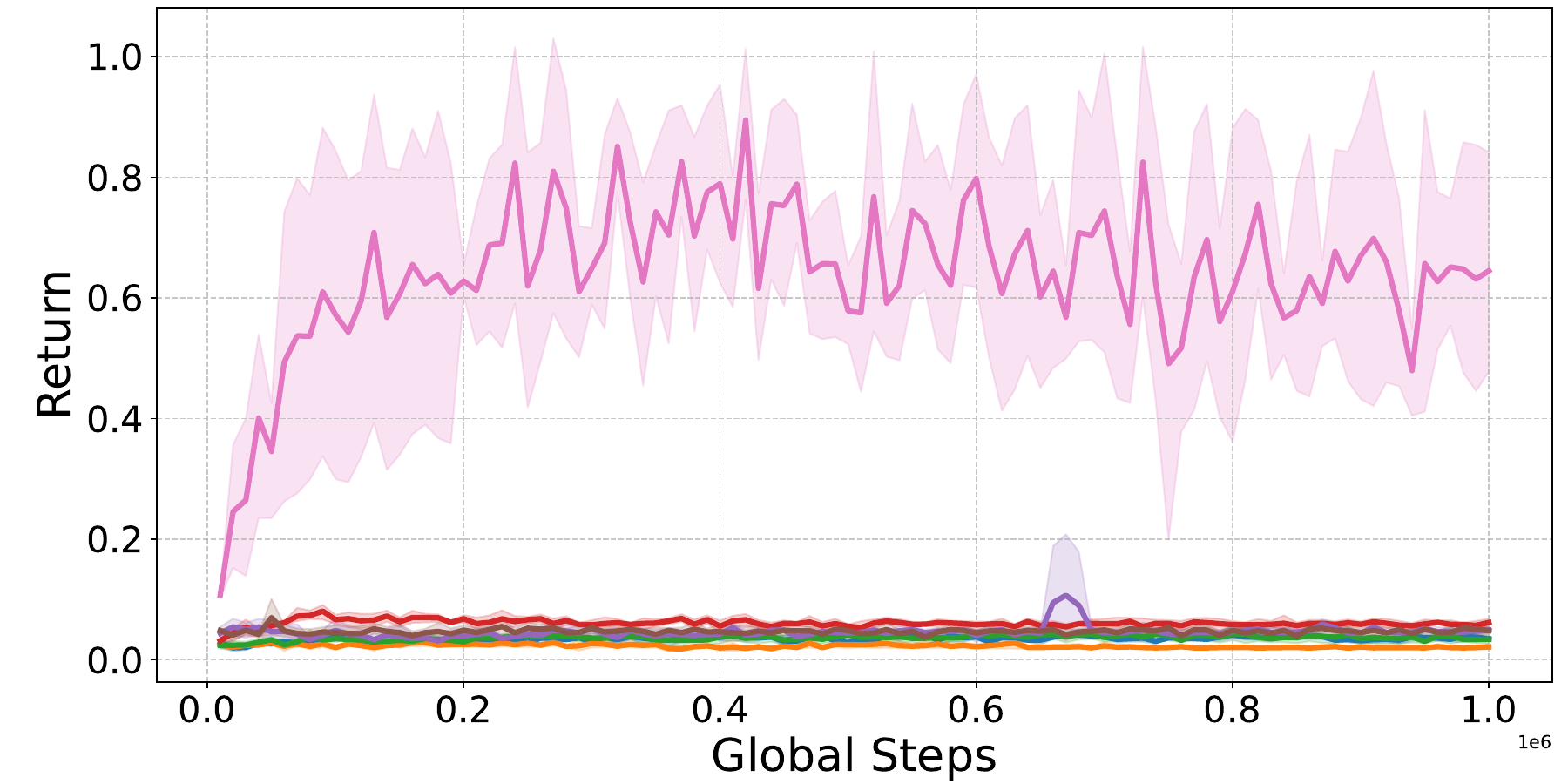}}
        \subfigure[Walker2d-v2 ($U(1, 16)$ Delays)]{\includegraphics[width=0.33\linewidth]{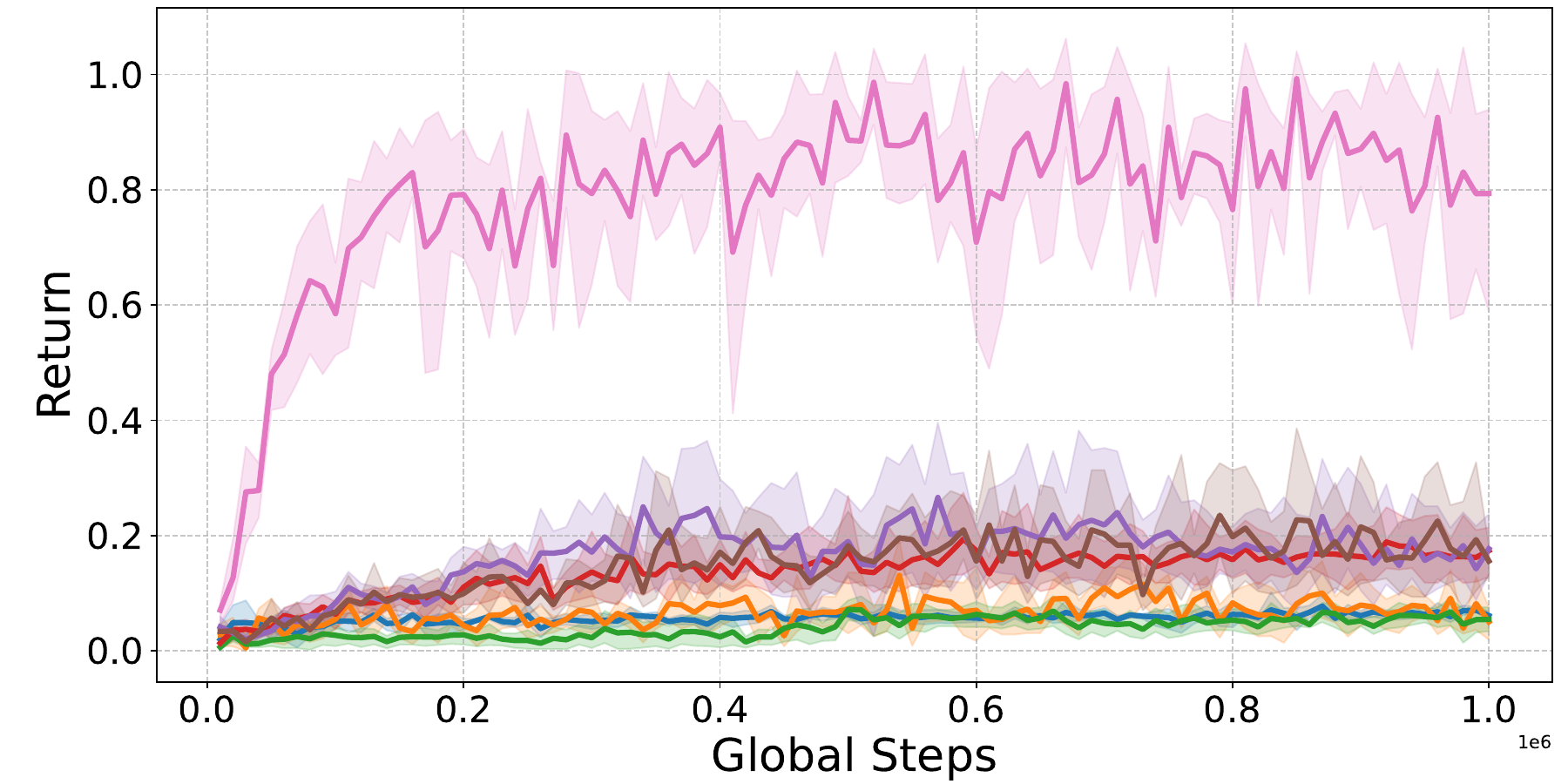}}
    }
    \centerline{
        \subfigure[HalfCheetah-v2 ($U(1, 32)$ Delays)]{\includegraphics[width=0.33\linewidth]{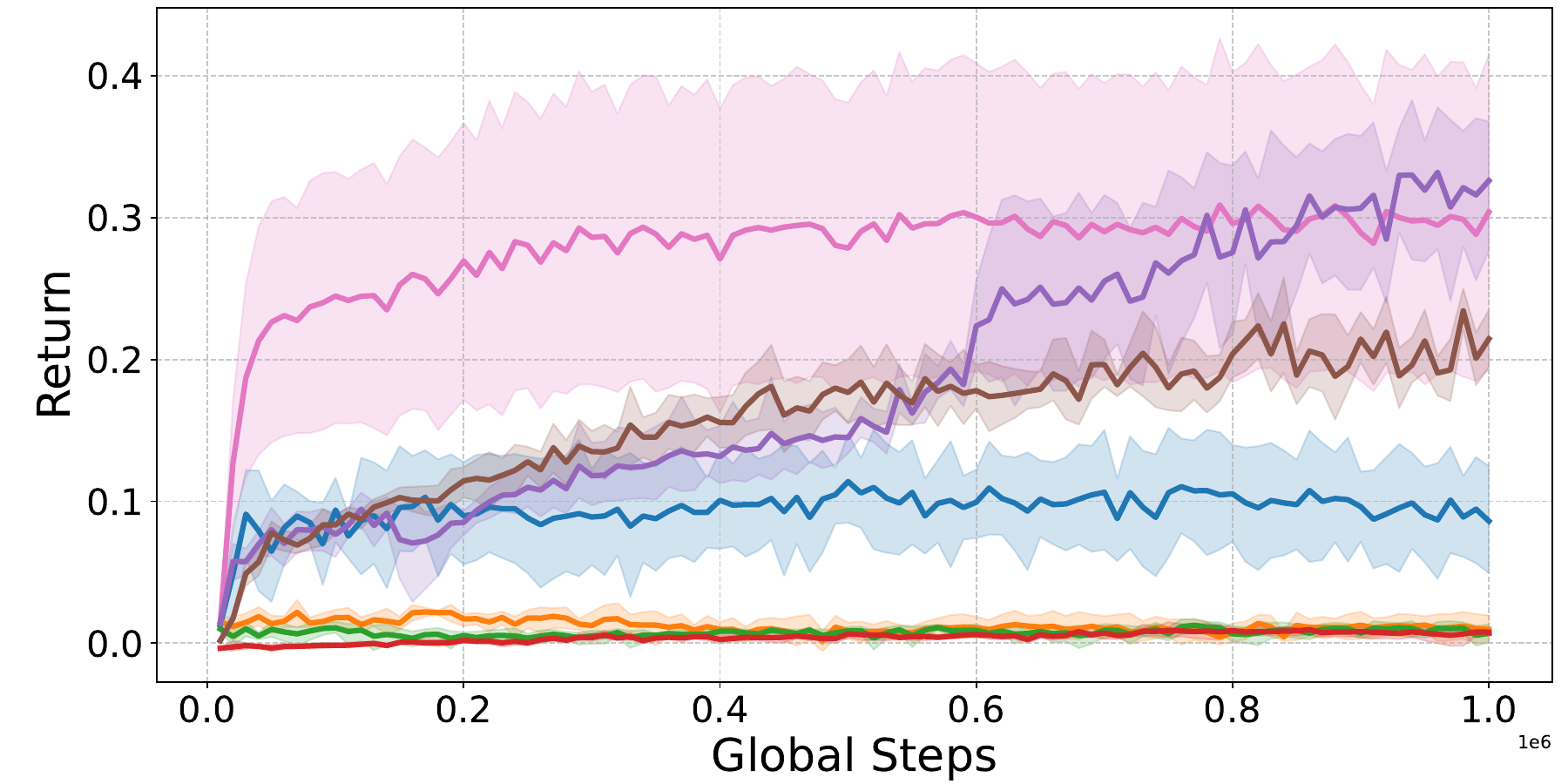}}
        \subfigure[Hopper-v2 ($U(1, 32)$ Delays)]{\includegraphics[width=0.33\linewidth]{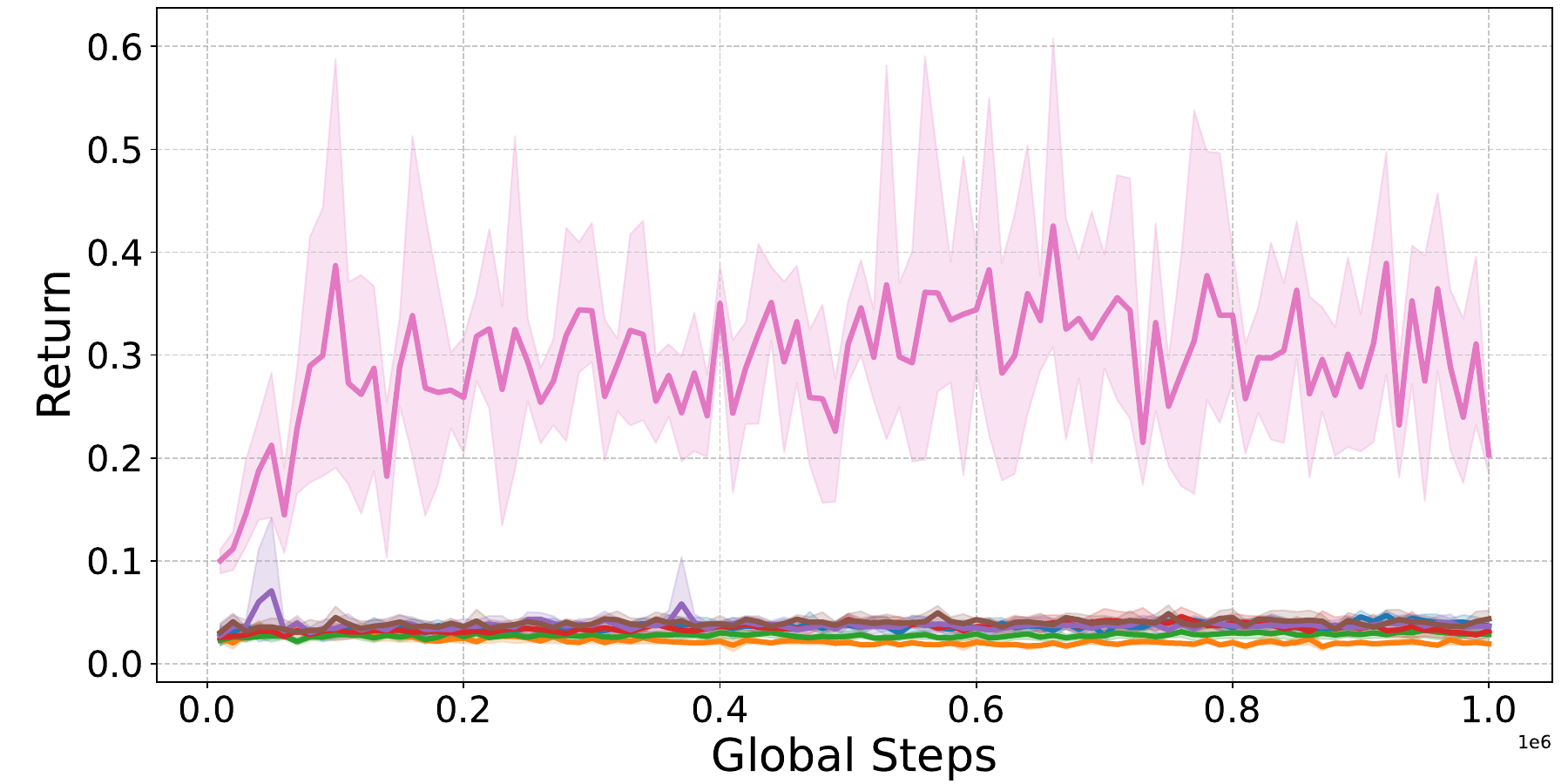}}
        \subfigure[Walker2d-v2 ($U(1, 32)$ Delays)]{\includegraphics[width=0.33\linewidth]{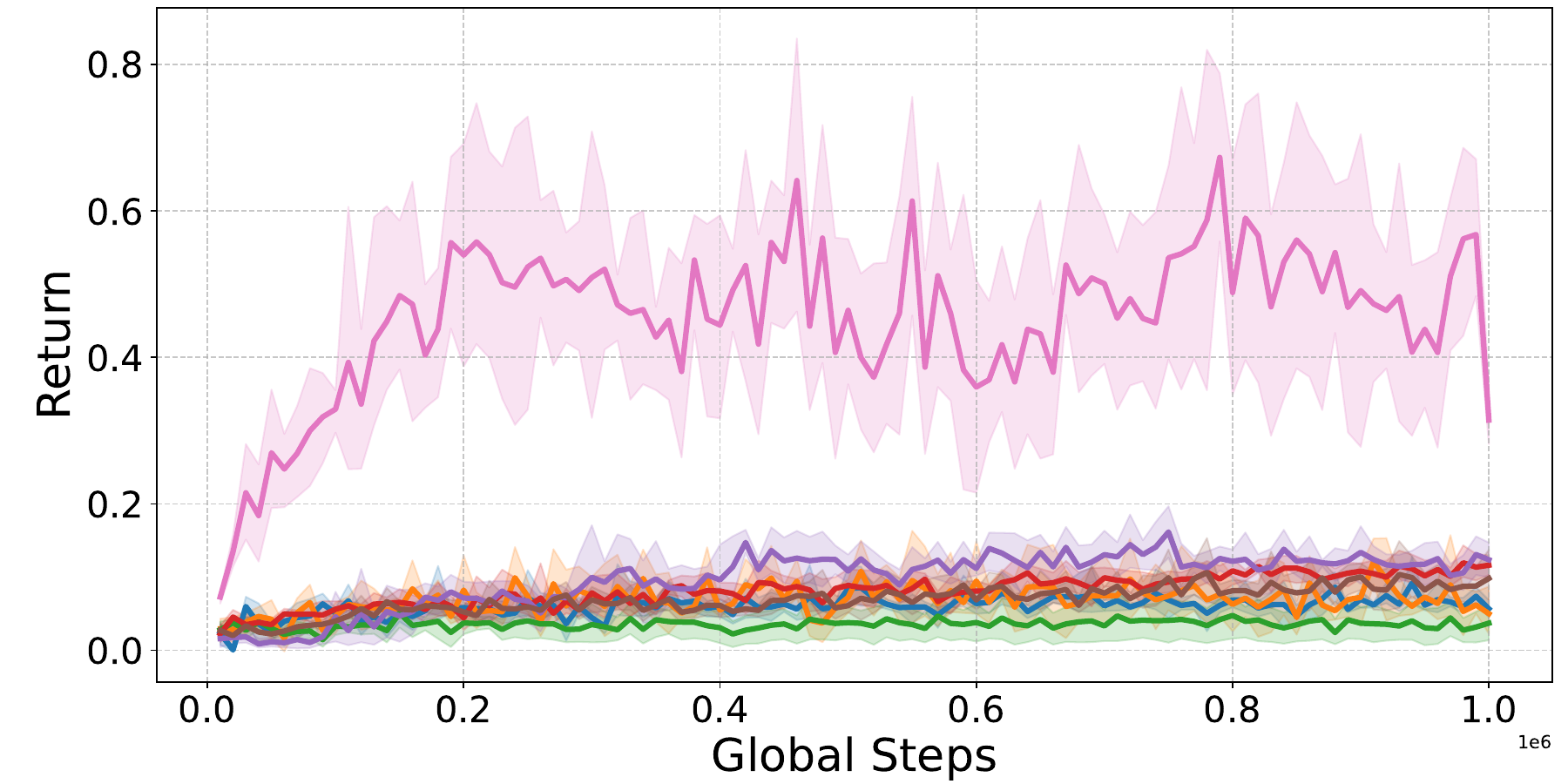}}
    }
    \centerline{
        \subfigure[HalfCheetah-v2 ($U(64)$ Delays)]{\includegraphics[width=0.33\linewidth]{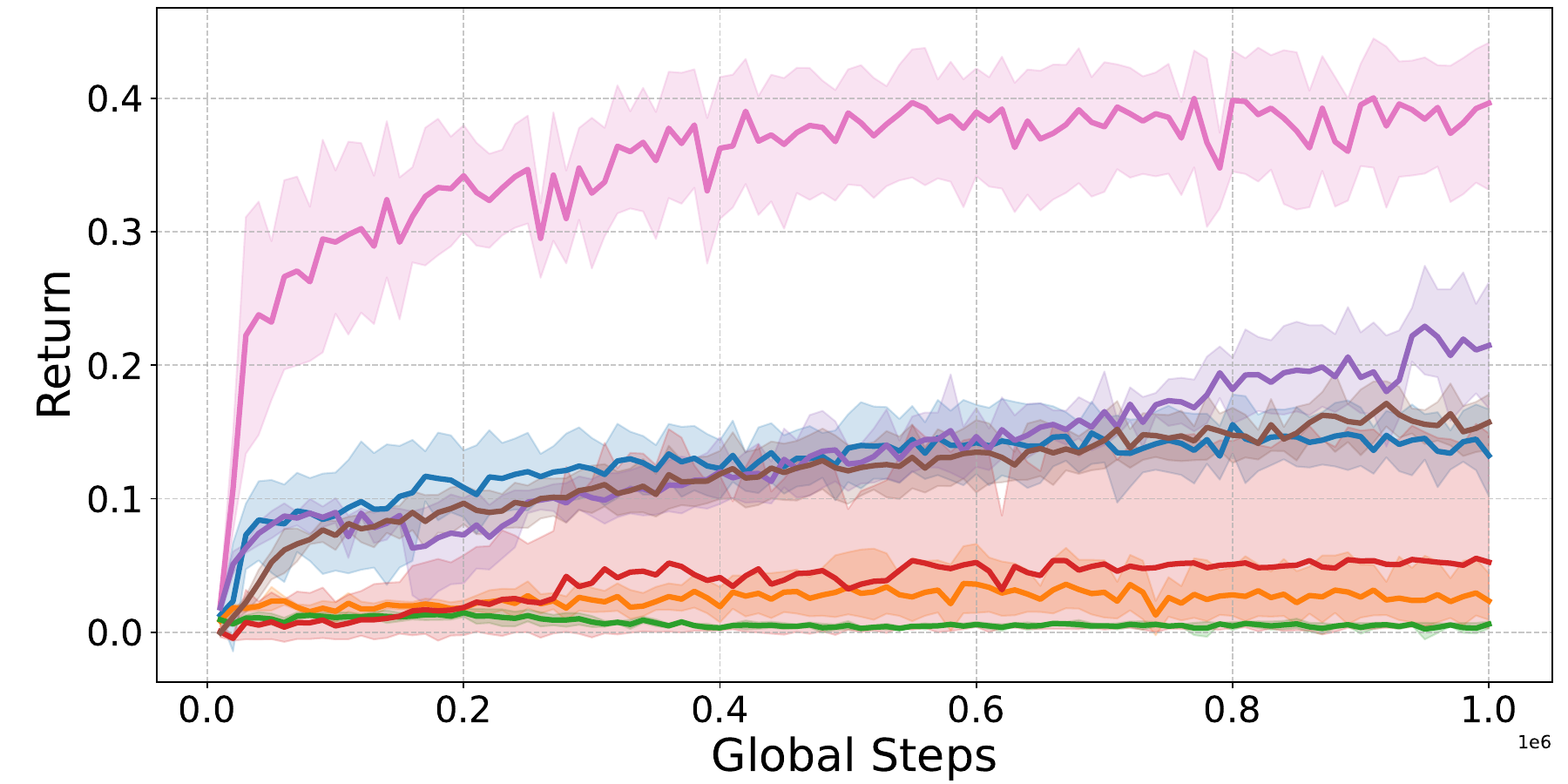}}
        \subfigure[Hopper-v2 ($U(1, 64)$ Delays)]{\includegraphics[width=0.33\linewidth]{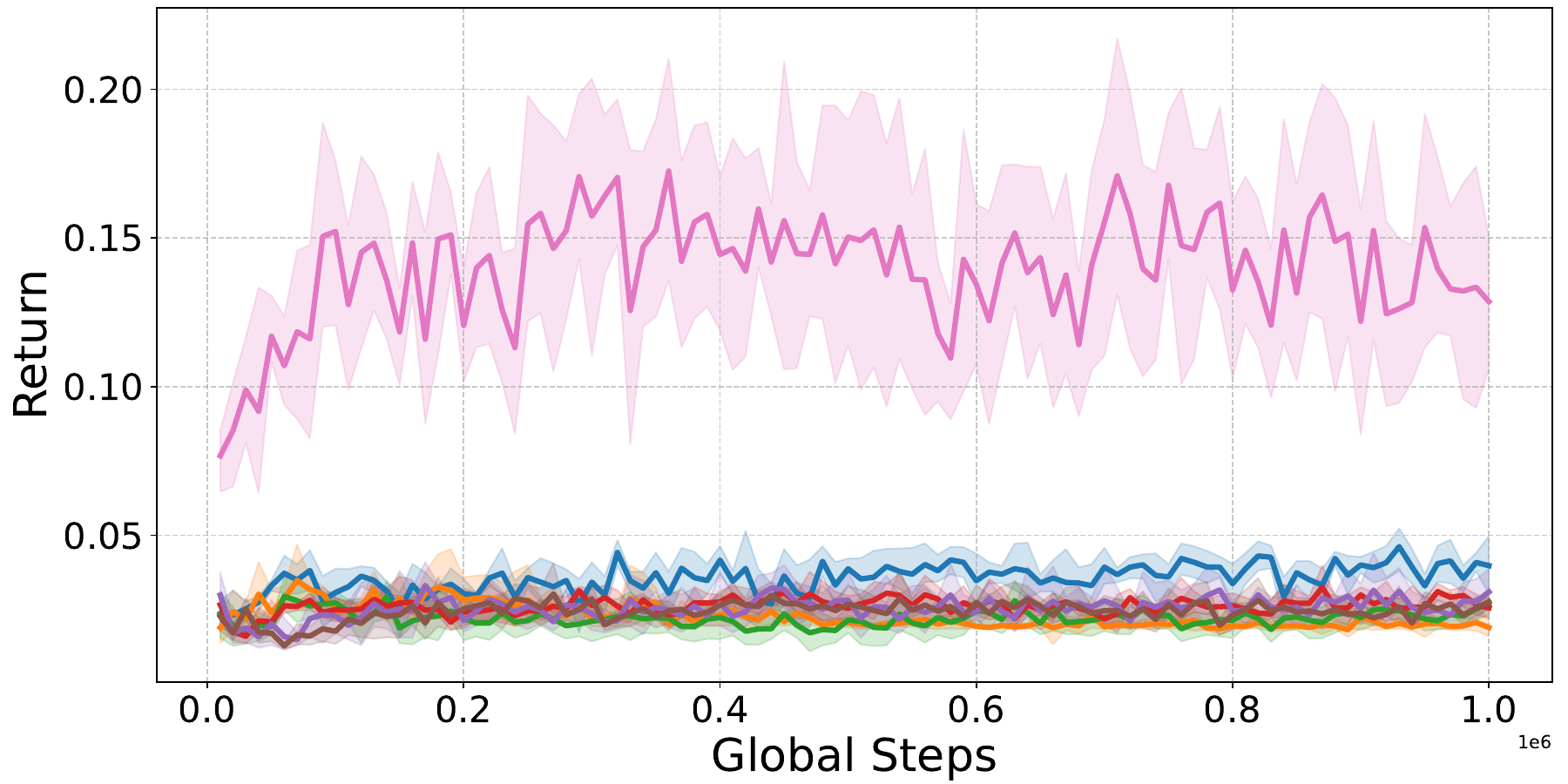}}
        \subfigure[Walker2d-v2 ($U(1, 64)$ Delays)]{\includegraphics[width=0.33\linewidth]{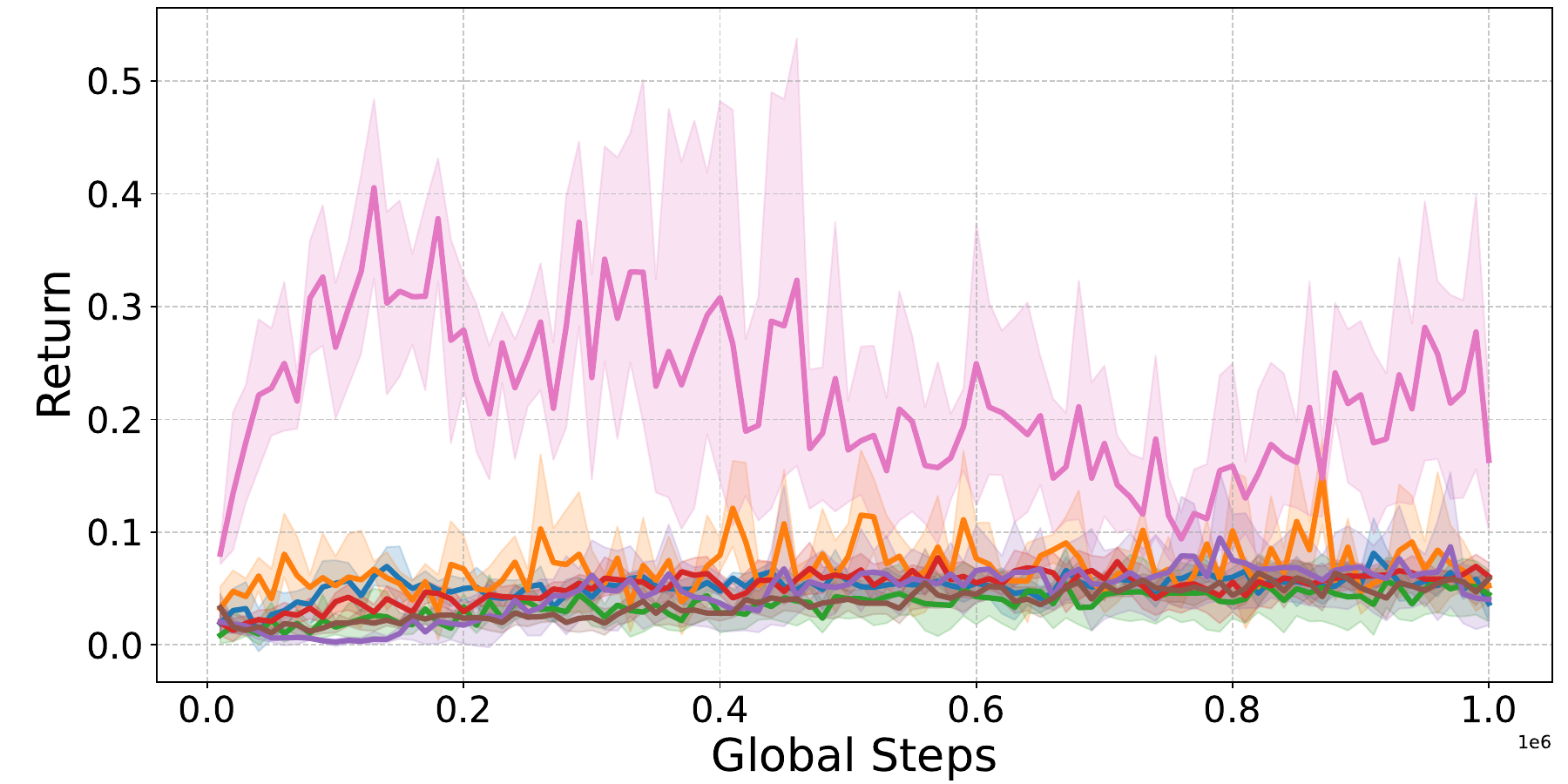}}
    }
    \centerline{
        \subfigure[HalfCheetah-v2 ($U(1, 128)$ Delays)]{\includegraphics[width=0.33\linewidth]{figs/stochastic_results/halfcheetah_D128.pdf}}
        \subfigure[Hopper-v2 ($U(1, 128)$ Delays)]{\includegraphics[width=0.33\linewidth]{figs/stochastic_results/hopper_D128.pdf}}
        \subfigure[Walker2d-v2 ($U(1, 128)$ Delays)]{\includegraphics[width=0.33\linewidth]{figs/stochastic_results/walker2d_D128.pdf}}
    }
    \includegraphics[width=1.\linewidth]{figs/deterministic_results/legend.pdf}
    \caption{Learning Curves on MuJoCo with Stochastic Delays.\label{appendix:fig_stochastic}}
\end{figure}

\clearpage
\section{Belief Qualitative Comparison}
\label{appendix:sec:belief_qualitative_comparison}
We report the qualitative comparison of the beliefs on HalfCheetah-v2, Hopper-v2, and Walker2d-v2 in \Cref{appendix:fig:qualitative_halfcheetah}, \Cref{appendix:fig:qualitative_hopper}, and \Cref{appendix:fig:qualitative_walker2d}, respectively.

\begin{figure}[h]
    \centering
    \centerline{
        \subfigure[\scriptsize Truth ($8$ Delays)]{\includegraphics[width=0.18\linewidth]{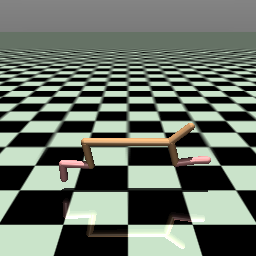}}
        \subfigure[\scriptsize DATS ($8$ Delays)]{\includegraphics[width=0.18\linewidth]{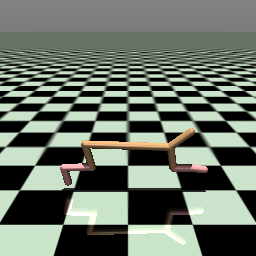}}
        \subfigure[\scriptsize D-Dreamer ($8$ Delays)]{\includegraphics[width=0.18\linewidth]{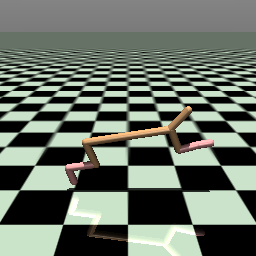}}
        \subfigure[\scriptsize D-SAC ($8$ Delays)]{\includegraphics[width=0.18\linewidth]{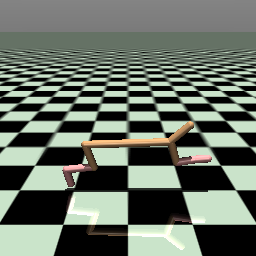}}
        \subfigure[\scriptsize DFBT ($8$ Delays)]{\includegraphics[width=0.18\linewidth]{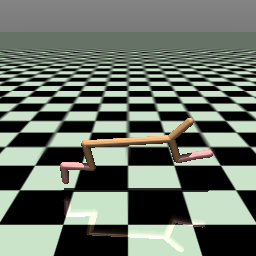}}
    }
    \centerline{
        \subfigure[\scriptsize Truth ($16$ Delays)]{\includegraphics[width=0.18\linewidth]{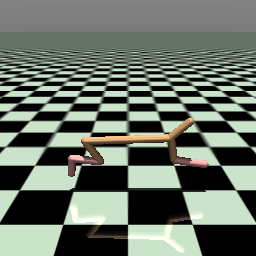}}
        \subfigure[\scriptsize DATS ($16$ Delays)]{\includegraphics[width=0.18\linewidth]{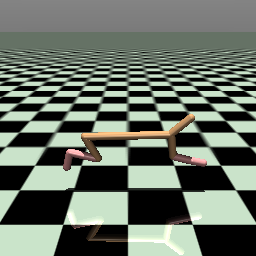}}
        \subfigure[\scriptsize D-Dreamer ($16$ Delays)]{\includegraphics[width=0.18\linewidth]{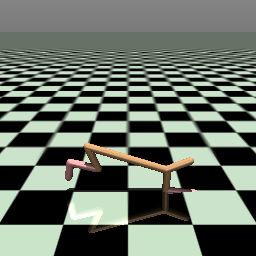}}
        \subfigure[\scriptsize D-SAC ($16$ Delays)]{\includegraphics[width=0.18\linewidth]{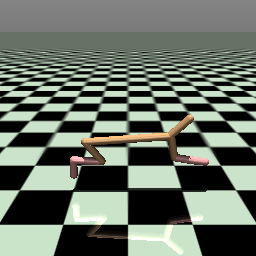}}
        \subfigure[\scriptsize DFBT ($16$ Delays)]{\includegraphics[width=0.18\linewidth]{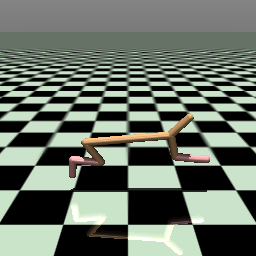}}
    }
    \centerline{
        \subfigure[\scriptsize Truth ($32$ Delays)]{\includegraphics[width=0.18\linewidth]{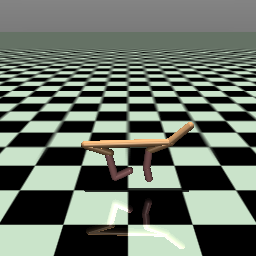}}
        \subfigure[\scriptsize DATS ($32$ Delays)]{\includegraphics[width=0.18\linewidth]{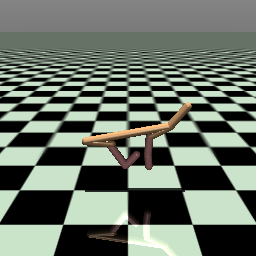}}
        \subfigure[\scriptsize D-Dreamer ($32$ Delays)]{\includegraphics[width=0.18\linewidth]{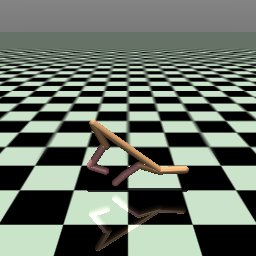}}
        \subfigure[\scriptsize D-SAC ($32$ Delays)]{\includegraphics[width=0.18\linewidth]{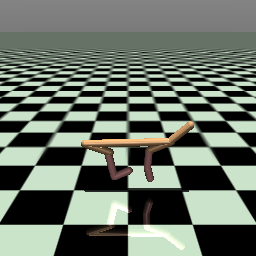}}
        \subfigure[\scriptsize DFBT ($32$ Delays)]{\includegraphics[width=0.18\linewidth]{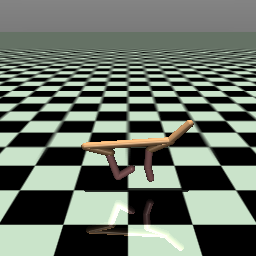}}
    }
    \centerline{
        \subfigure[\scriptsize Truth ($64$ Delays)]{\includegraphics[width=0.18\linewidth]{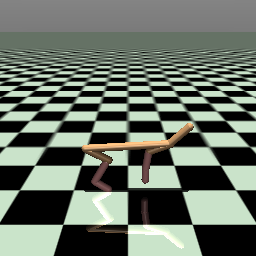}}
        \subfigure[\scriptsize DATS ($64$ Delays)]{\includegraphics[width=0.18\linewidth]{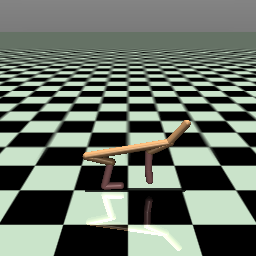}}
        \subfigure[\scriptsize D-Dreamer ($64$ Delays)]{\includegraphics[width=0.18\linewidth]{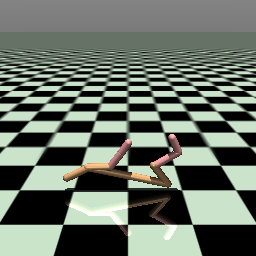}}
        \subfigure[\scriptsize D-SAC ($64$ Delays)]{\includegraphics[width=0.18\linewidth]{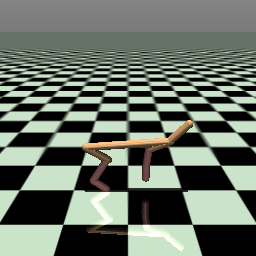}}
        \subfigure[\scriptsize DFBT ($64$ Delays)]{\includegraphics[width=0.18\linewidth]{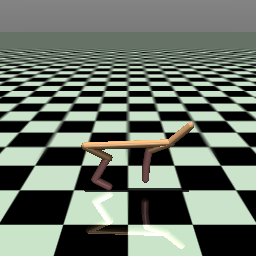}}
    }
    \centerline{
        \subfigure[\scriptsize Truth ($128$ Delays)]{\includegraphics[width=0.18\linewidth]{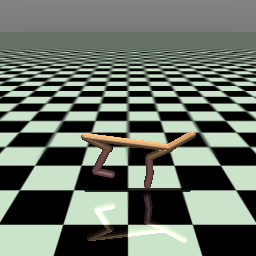}}
        \subfigure[\scriptsize DATS ($128$ Delays)]{\includegraphics[width=0.18\linewidth]{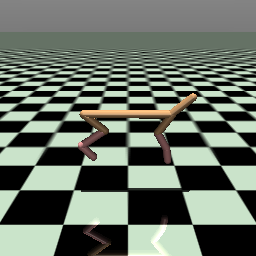}}
        \subfigure[\scriptsize D-Dreamer ($128$ Delays)]{\includegraphics[width=0.18\linewidth]{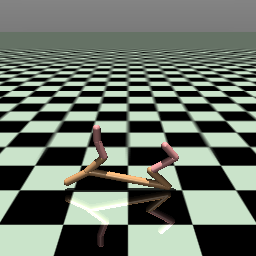}}
        \subfigure[\scriptsize D-SAC ($128$ Delays)]{\includegraphics[width=0.18\linewidth]{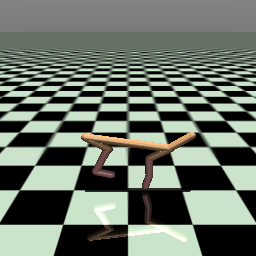}}
        \subfigure[\scriptsize DFBT ($128$ Delays)]{\includegraphics[width=0.18\linewidth]{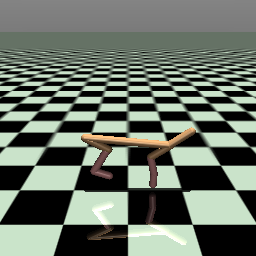}}
    }
    \caption{Belief qualitative comparison on HalfCheetah-v2 with different delays.\label{appendix:fig:qualitative_halfcheetah}}
\end{figure}

\begin{figure}[h]
    \centering
    \centerline{
        \subfigure[\scriptsize Truth ($8$ Delays)]{\includegraphics[width=0.18\linewidth]{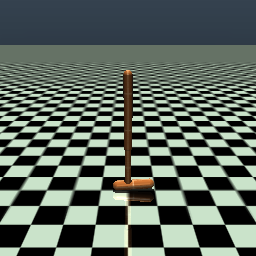}}
        \subfigure[\scriptsize DATS ($8$ Delays)]{\includegraphics[width=0.18\linewidth]{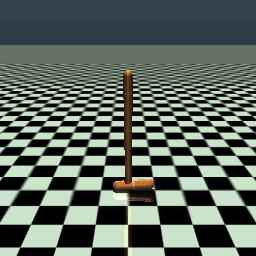}}
        \subfigure[\scriptsize D-Dreamer ($8$ Delays)]{\includegraphics[width=0.18\linewidth]{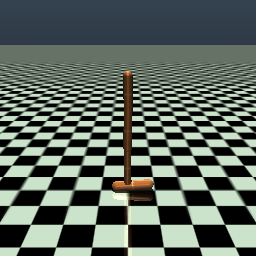}}
        \subfigure[\scriptsize D-SAC ($8$ Delays)]{\includegraphics[width=0.18\linewidth]{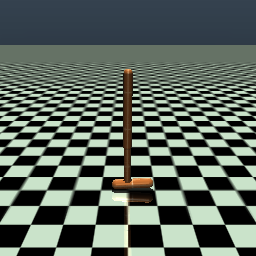}}
        \subfigure[\scriptsize DFBT ($8$ Delays)]{\includegraphics[width=0.18\linewidth]{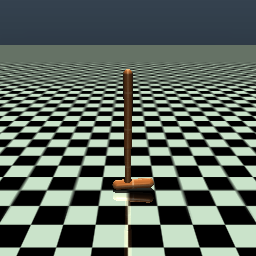}}
    }
    \centerline{
        \subfigure[\scriptsize Truth ($16$ Delays)]{\includegraphics[width=0.18\linewidth]{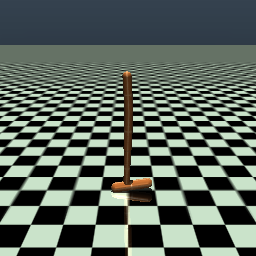}}
        \subfigure[\scriptsize DATS ($16$ Delays)]{\includegraphics[width=0.18\linewidth]{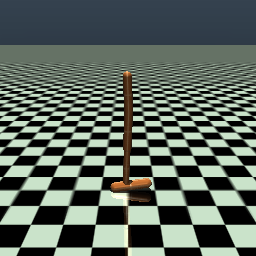}}
        \subfigure[\scriptsize D-Dreamer ($16$ Delays)]{\includegraphics[width=0.18\linewidth]{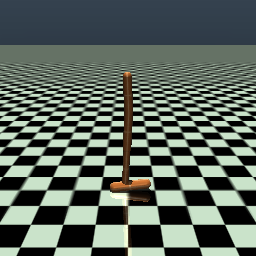}}
        \subfigure[\scriptsize D-SAC ($16$ Delays)]{\includegraphics[width=0.18\linewidth]{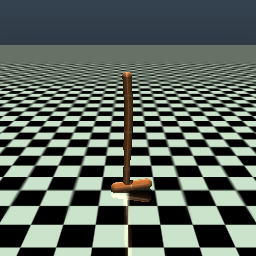}}
        \subfigure[\scriptsize DFBT ($16$ Delays)]{\includegraphics[width=0.18\linewidth]{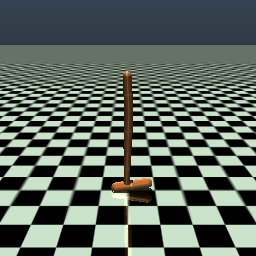}}
    }
    \centerline{
        \subfigure[\scriptsize Truth ($32$ Delays)]{\includegraphics[width=0.18\linewidth]{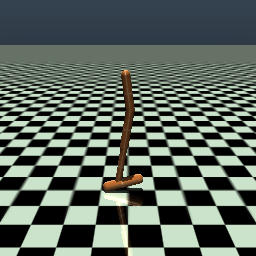}}
        \subfigure[\scriptsize DATS ($32$ Delays)]{\includegraphics[width=0.18\linewidth]{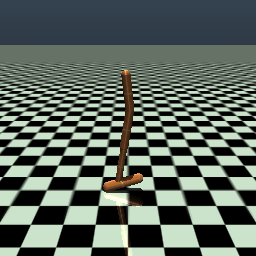}}
        \subfigure[\scriptsize D-Dreamer ($32$ Delays)]{\includegraphics[width=0.18\linewidth]{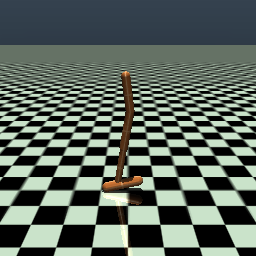}}
        \subfigure[\scriptsize D-SAC ($32$ Delays)]{\includegraphics[width=0.18\linewidth]{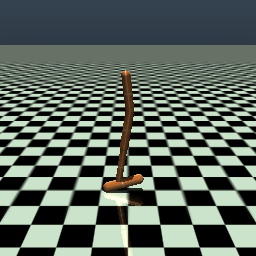}}
        \subfigure[\scriptsize DFBT ($32$ Delays)]{\includegraphics[width=0.18\linewidth]{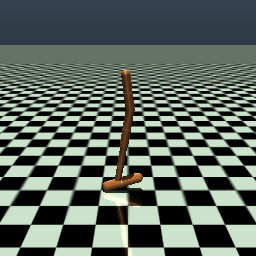}}
    }
    \centerline{
        \subfigure[\scriptsize Truth ($64$ Delays)]{\includegraphics[width=0.18\linewidth]{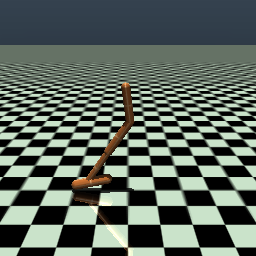}}
        \subfigure[\scriptsize DATS ($64$ Delays)]{\includegraphics[width=0.18\linewidth]{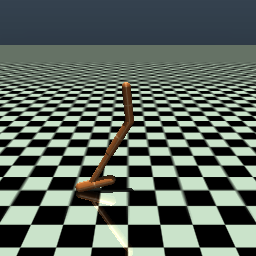}}
        \subfigure[\scriptsize D-Dreamer ($64$ Delays)]{\includegraphics[width=0.18\linewidth]{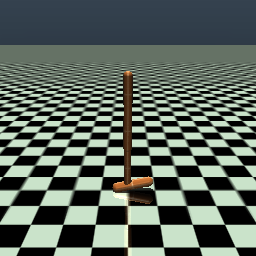}}
        \subfigure[\scriptsize D-SAC ($64$ Delays)]{\includegraphics[width=0.18\linewidth]{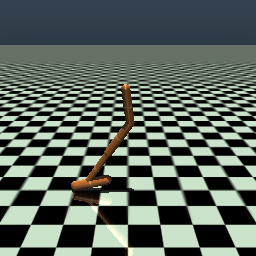}}
        \subfigure[\scriptsize DFBT ($64$ Delays)]{\includegraphics[width=0.18\linewidth]{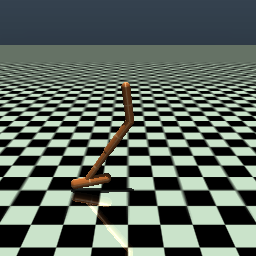}}
    }
    \centerline{
        \subfigure[\scriptsize Truth ($128$ Delays)]{\includegraphics[width=0.18\linewidth]{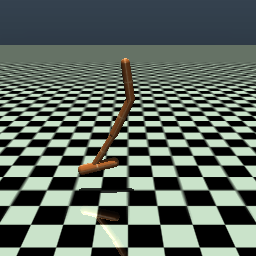}}
        \subfigure[\scriptsize DATS ($128$ Delays)]{\includegraphics[width=0.18\linewidth]{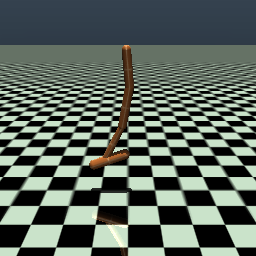}}
        \subfigure[\scriptsize D-Dreamer ($128$ Delays)]{\includegraphics[width=0.18\linewidth]{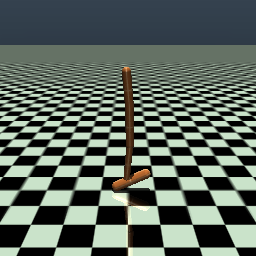}}
        \subfigure[\scriptsize D-SAC ($128$ Delays)]{\includegraphics[width=0.18\linewidth]{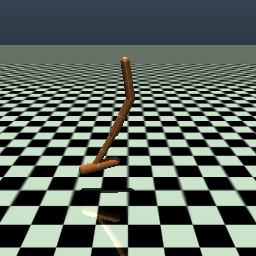}}
        \subfigure[\scriptsize DFBT ($128$ Delays)]{\includegraphics[width=0.18\linewidth]{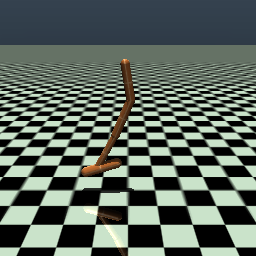}}
    }
    \caption{Belief qualitative comparison on Hopper-v2 with different delays.\label{appendix:fig:qualitative_hopper}}
\end{figure}

\begin{figure}[h]
    \centering
    \centerline{
        \subfigure[\scriptsize Truth ($8$ Delays)]{\includegraphics[width=0.18\linewidth]{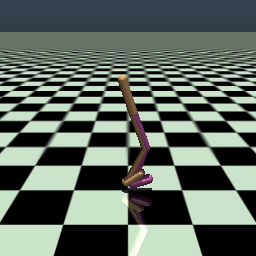}}
        \subfigure[\scriptsize DATS ($8$ Delays)]{\includegraphics[width=0.18\linewidth]{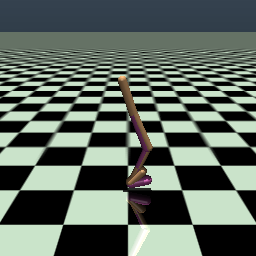}}
        \subfigure[\scriptsize D-Dreamer ($8$ Delays)]{\includegraphics[width=0.18\linewidth]{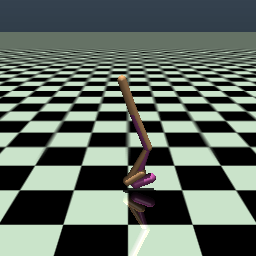}}
        \subfigure[\scriptsize D-SAC ($8$ Delays)]{\includegraphics[width=0.18\linewidth]{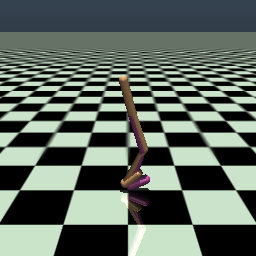}}
        \subfigure[\scriptsize DFBT ($8$ Delays)]{\includegraphics[width=0.18\linewidth]{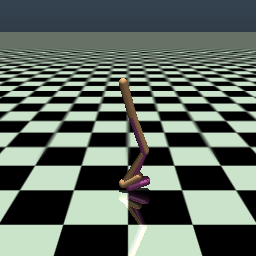}}
    }
    \centerline{
        \subfigure[\scriptsize Truth ($16$ Delays)]{\includegraphics[width=0.18\linewidth]{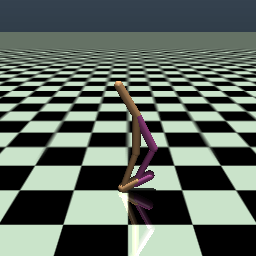}}
        \subfigure[\scriptsize DATS ($16$ Delays)]{\includegraphics[width=0.18\linewidth]{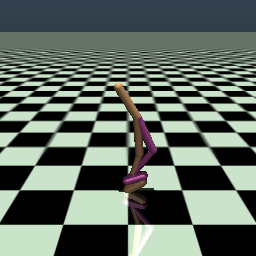}}
        \subfigure[\scriptsize D-Dreamer ($16$ Delays)]{\includegraphics[width=0.18\linewidth]{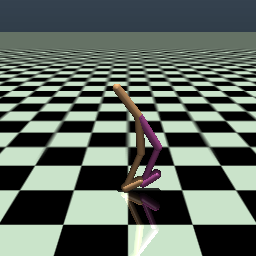}}
        \subfigure[\scriptsize D-SAC ($16$ Delays)]{\includegraphics[width=0.18\linewidth]{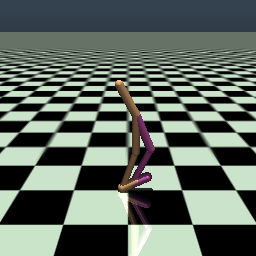}}
        \subfigure[\scriptsize DFBT ($16$ Delays)]{\includegraphics[width=0.18\linewidth]{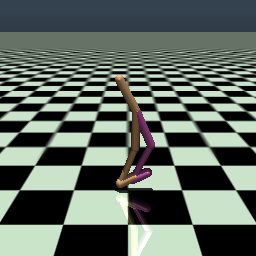}}
    }
    \centerline{
        \subfigure[\scriptsize Truth ($32$ Delays)]{\includegraphics[width=0.18\linewidth]{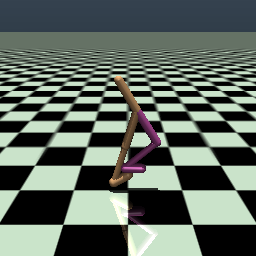}}
        \subfigure[\scriptsize DATS ($32$ Delays)]{\includegraphics[width=0.18\linewidth]{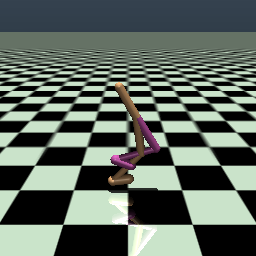}}
        \subfigure[\scriptsize D-Dreamer ($32$ Delays)]{\includegraphics[width=0.18\linewidth]{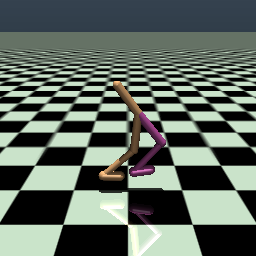}}
        \subfigure[\scriptsize D-SAC ($32$ Delays)]{\includegraphics[width=0.18\linewidth]{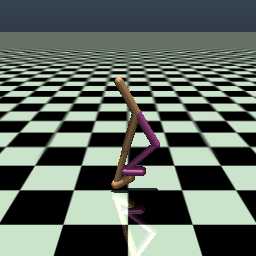}}
        \subfigure[\scriptsize DFBT ($32$ Delays)]{\includegraphics[width=0.18\linewidth]{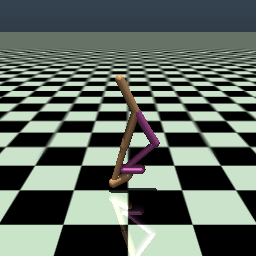}}
    }
    \centerline{
        \subfigure[\scriptsize Truth ($64$ Delays)]{\includegraphics[width=0.18\linewidth]{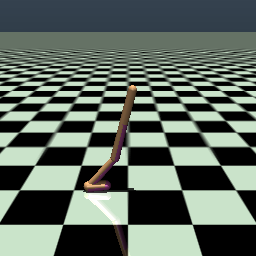}}
        \subfigure[\scriptsize DATS ($64$ Delays)]{\includegraphics[width=0.18\linewidth]{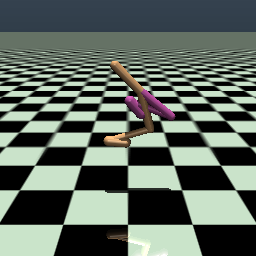}}
        \subfigure[\scriptsize D-Dreamer ($64$ Delays)]{\includegraphics[width=0.18\linewidth]{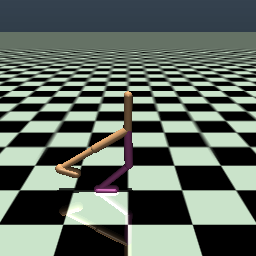}}
        \subfigure[\scriptsize D-SAC ($64$ Delays)]{\includegraphics[width=0.18\linewidth]{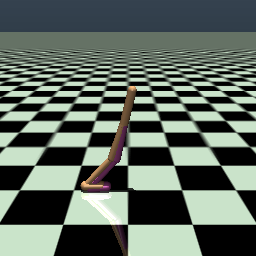}}
        \subfigure[\scriptsize DFBT ($64$ Delays)]{\includegraphics[width=0.18\linewidth]{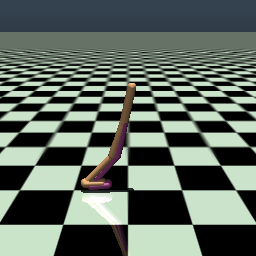}}
    }
    \centerline{
        \subfigure[\scriptsize Truth ($128$ Delays)]{\includegraphics[width=0.18\linewidth]{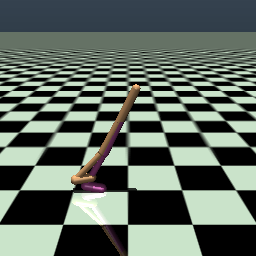}}
        \subfigure[\scriptsize DATS ($128$ Delays)]{\includegraphics[width=0.18\linewidth]{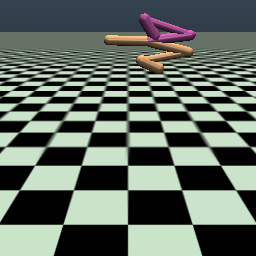}}
        \subfigure[\scriptsize D-Dreamer ($128$ Delays)]{\includegraphics[width=0.18\linewidth]{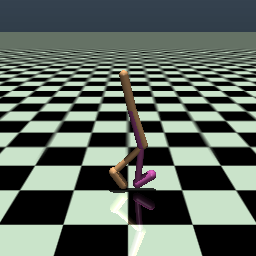}}
        \subfigure[\scriptsize D-SAC ($128$ Delays)]{\includegraphics[width=0.18\linewidth]{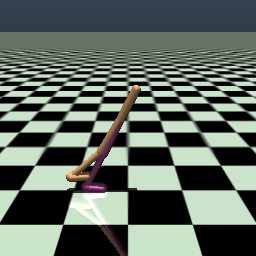}}
        \subfigure[\scriptsize DFBT ($128$ Delays)]{\includegraphics[width=0.18\linewidth]{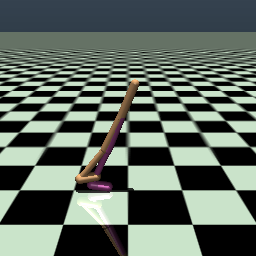}}
    }
    \caption{Belief qualitative comparison on Walker2d-v2 with different delays.\label{appendix:fig:qualitative_walker2d}}
\end{figure}


\end{document}